%% file: main.tex
\documentclass{article}

\usepackage{microtype}
\usepackage{graphicx}
\usepackage{subcaption}
\usepackage{booktabs} 

\usepackage[accepted]{icml2025}

\usepackage[latin1]{inputenc}
\usepackage{pgfplots}
\usepackage{float}
\usepackage{pifont}
\usepackage{makecell}
\usepackage{nicefrac}
\usepackage{wrapfig}
\usepackage[inline]{enumitem}
\usepackage{stfloats} 

\usepackage{etoc}

\usepackage{tikz}
\usepgfplotslibrary{external} 
\usepackage{pgf}
\usepackage{pgfplots}
\usepackage{pgfplotstable}
\usepgfplotslibrary{groupplots} 
\usetikzlibrary{arrows.meta}
\usetikzlibrary{decorations.text}    
\usetikzlibrary{math}
\tikzexternaldisable

\definecolor{cherry1}{rgb}{0.215686, 0.215686, 0.215686}
\definecolor{cherry2}{rgb}{0.563899, 0.155919, 0.156577}
\definecolor{cherry3}{rgb}{0.747389, 0.178584, 0.180272}
\definecolor{cherry4}{rgb}{0.836168, 0.264453, 0.26819}
\definecolor{cherry5}{rgb}{0.880144, 0.397868, 0.404399}
\definecolor{cherry6}{rgb}{0.911942, 0.567676, 0.576412}

\usetikzlibrary{fit,calc}
\newcommand*{\tikzmk}[1]{\tikz[remember picture,overlay,] \node (#1) {};\ignorespaces}
\newcommand{\boxit}[1]{\tikz[remember picture,overlay]{\node[yshift=3pt,fill=#1,opacity=.25,fit={(A)($(B)+(.90\linewidth,.8\baselineskip)$)}] {};}\ignorespaces}

\input{preamble}
\usepackage{accents}

\usepackage[colorinlistoftodos]{todonotes}

\allowdisplaybreaks

\begin{document}

\icmltitlerunning{Tuning SMC Samplers via Greedy Incremental Divergence Minimization}

\twocolumn[
\icmltitle{
Tuning Sequential Monte Carlo Samplers via\\Greedy Incremental Divergence Minimization
}


\icmlsetsymbol{equal}{*}

\begin{icmlauthorlist}
\icmlauthor{Kyurae Kim}{equal,penn}
\icmlauthor{Zuheng Xu}{equal,ubc}
\icmlauthor{Jacob R. Gardner}{penn}
\icmlauthor{Trevor Campbell}{ubc}
\end{icmlauthorlist}

\icmlaffiliation{penn}{Dept. Computer and Information Science, University of Pennsylvania, Philadelphia, U.S.}
\icmlaffiliation{ubc}{Dept. Statistics, University of British Columbia, Vancouver, Canada}

\icmlcorrespondingauthor{Kyurae Kim}{\href{mailto:kyrim@seas.upenn.edu}{kyrkim@seas.upenn.edu}}
\icmlcorrespondingauthor{Zuheng Xu}{\href{mailto:zuheng.xu@stat.ubc.ca}{zuheng.xu@stat.ubc.ca}}
\icmlcorrespondingauthor{Trevor Campbell}{\href{mailto:trevor@stat.ubc.ca}{trevor@stat.ubc.ca}}
\icmlcorrespondingauthor{Jacob R. Gardner}{\href{mailto:jrgardner@seas.upenn.edu}{jacobrg@seas.upenn.edu}}

\icmlkeywords{sequential Monte Carlo, Bayesian inference, Markov chain Monte Carlo, annealed importance sampling}

\vskip 0.3in
]



\printAffiliationsAndNotice{\icmlEqualContribution} 

\begin{abstract}
The performance of sequential Monte Carlo (SMC) samplers heavily depends on the tuning of the Markov kernels used in the path proposal.
For SMC samplers with unadjusted Markov kernels, standard tuning objectives, such as the Metropolis-Hastings acceptance rate or the expected-squared jump distance, are no longer applicable.
While stochastic gradient-based end-to-end optimization has been explored for tuning SMC samplers, they often incur excessive training costs, even for tuning just the kernel step sizes.
In this work, we propose a general adaptation framework for tuning the Markov kernels in SMC samplers by minimizing the incremental Kullback-Leibler (KL) divergence between the proposal and target paths.
For step size tuning, we provide a gradient- and tuning-free algorithm that is generally applicable for kernels such as Langevin Monte Carlo (LMC).
We further demonstrate the utility of our approach by providing a tailored scheme for tuning \textit{kinetic} LMC used in SMC samplers.
Our implementations are able to obtain a full \textit{schedule} of tuned parameters at the cost of a few vanilla SMC runs, which is a fraction of gradient-based approaches.
\end{abstract}


\input{section_introduction}

\input{section_background}

\input{section_method}

\input{section_experiments}


\input{section_discussions}

\clearpage

\section*{Acknowledgements}
The authors sincerely thank Nicolas Chopin for pointing out relevant theoretical results, Alexandre Bouchard-C\^{o}t\'e for discussions throughout the project, and the anonymous reviewers for helpful suggestions.

K. Kim and J. R. Gardner were supported through the NSF award [IIS2145644]; T. Campbell and D. Xu were supported by the NSERC Discovery Grant RGPIN-2025-04208. 
We also gratefully acknowledge the use of the ARC Sockeye computing platform at the University of British Columbia.

\section*{Impact Statement}
This paper presents a method for automatically tuning sequential Monte Carlo (SMC) samplers.
Due to the technical nature of the work, we do not expect to have direct societal consequences.
Downstream applications of SMC samplers, however, include uncertainty quantification~\citep{dai_invitation_2022}, statistical model comparison~\citep{zhou_automatic_2016}, and more recently, conditional generation from diffusion models~\citep{wu_practical_2023,trippe_diffusion_2023} and steering large language models to ensure their output conforms to ethical constraints~\citep{zhao_probabilistic_2024}. 
Our work may improve the efficiency and efficacy of such tasks, indirectly affecting their societal consequences.


\bibliographystyle{icml2025}
\bibliography{references}

\clearpage
\appendix


\onecolumn
\renewcommand{\baselinestretch}{0.75}\normalsize
{\hypersetup{linkbordercolor=black,linkcolor=black}
\tableofcontents
}
\renewcommand{\baselinestretch}{1.0}\normalsize
\twocolumn

\clearpage
\section{Benchmark Problems}\label{section:problems}
\input{section_problems}

\clearpage
\section{Details on the Experimental  Setup}\label{section:config}

\input{section_config}

\clearpage
\section{Algorithms}\label{section:algorithms}
\input{section_algorithms}

\clearpage
\section{Theoretical Analysis}\label{section:analysis}
\input{section_analysis}

\clearpage
\section{Backward Kernels}\label{section:backward}
\input{section_backward_kernels}

\clearpage
\onecolumn
\section{Additional Experimental Results}
\input{section_additional_results}

\end{document}

%% file: preamble.tex
\usepackage{amssymb}

\usepackage{amsmath,amsthm}
\usepackage{mathtools}
\usepackage{iftex}
\usepackage{etoolbox}

\ifxetex
  \usepackage{microtype}
  \usepackage{unicode-math}
  
  \newcommand{\mathmat}[1]{\mathbfit{#1}}
  \newcommand{\mathvec}[1]{\mathbfit{#1}}
  \newcommand{\mathvecgreek}[1]{\mathbfit{#1}}
  \newcommand{\mathmatgreek}[1]{\mathbfit{#1}}

  \newcommand{\mathrv}[1]{\mathsfit{#1}}
  \newcommand{\mathrvvec}[1]{\mathbfsfit{#1}}
  \newcommand{\mathrvgreek}[1]{\mathsfit{#1}}
  \newcommand{\mathrvvecgreek}[1]{\mathbfsfit{#1}}

  \usepackage{fontspec}
\else

  \usepackage[OMLmathsfit]{isomath}
  \usepackage{fixmath}

  \newcommand{\mathmat}[1]{\mathbold{#1}}
  \newcommand{\mathvec}[1]{\mathbold{#1}}
  \newcommand{\mathvecgreek}[1]{\mathbold{#1}}
  \newcommand{\mathmatgreek}[1]{\mathbold{#1}}

  \newcommand{\mathrv}[1]{\mathsfit{#1}}
  \newcommand{\mathrvvec}[1]{\mathsfit{#1}}
  \newcommand{\mathrvgreek}[1]{\mathsfit{#1}}
  \newcommand{\mathrvvecgreek}[1]{\mathsfit{#1}}
\fi

\usepackage{booktabs,threeparttable,arydshln}
\usepackage{multirow}
\usepackage{nicematrix}

\usepackage[algo2e, ruled]{algorithm2e}
\SetKwComment{Comment}{$\triangleright$ }{}
\makeatletter
\newcommand{\removelatexerror}{\let\@latex@error\@gobble}
\makeatother

\usepackage{pgfplots}
\pgfkeys{/pgfplots/tuftelike/.style={
  semithick,
  tick style={major tick length=4pt,semithick,black},
  separate axis lines,
  axis x line*=bottom,
  axis x line shift=2pt,
  xlabel shift=-2pt,
  axis y line*=left,
  axis y line shift=2pt,
  ylabel shift=-2pt}}
  
\usepackage{proof-at-the-end}

\usepackage{mdframed}
\mdfdefinestyle {mdleftbarcoloredstyle} {%
    leftline          = true,%
    rightline         = false,%
    topline           = false,%
    bottomline        = false,%
    linewidth         = 2.5pt,%
    backgroundcolor   = gray!10,%
    linecolor         = gray,%
    innertopmargin    = 0.0ex,%
    innerbottommargin = 0.7ex,%
    innerleftmargin   = 1.ex,%
    ntheorem          = false,%
}

\mdfdefinestyle {coloredstyle} {%
    leftline          = false,%
    rightline         = false,%
    topline           = false,%
    bottomline        = false,%
    backgroundcolor   = gray!10,%
    innertopmargin    = 1.0ex,%
    innerbottommargin = 0.7ex,%
    innerleftmargin   = 1.ex,%
}

\declaretheoremstyle[
    spaceabove    = \parsep,
    spacebelow    = \parsep,
    bodyfont      = \normalfont\itshape,
]{theoremsty}

\declaretheorem[name=Theorem,    style=theoremsty]{theorem}

\declaretheorem[name=Corollary,  style=theoremsty]{corollary}
\declaretheorem[name=Lemma,      style=theoremsty]{lemma}
\declaretheorem[name=Theorem,    style=theoremsty, numbered=no]{theorem*}
\declaretheorem[name=Proposition,style=theoremsty, numbered=no]{proposition*}
\declaretheorem[name=Corollary,  style=theoremsty, numbered=no]{corollary*}
\declaretheorem[name=Lemma,      style=theoremsty, numbered=no]{lemma*}

\declaretheorem[name=Open Problem,style=theoremsty, mdframed={style = coloredstyle}, numbered=no]{openproblem*}
\declaretheorem[name=Conjecture,  style=theoremsty, mdframed={style = coloredstyle}, numbered=no]{conjecture*}

\declaretheoremstyle[
    spaceabove=\parsep,
    spacebelow=\parsep,
    bodyfont=\normalfont,
]{normalsty}
\declaretheorem[name=Remark,      style=normalsty]{remark}
\declaretheorem[name=Definition,  style=normalsty]{definition}
\declaretheorem[name=Assumption,  style=normalsty]{assumption}

\declaretheorem[name=Remark,      style=normalsty, numbered=no]{remark*}
\declaretheorem[name=Definition,  style=normalsty, numbered=no]{definition*}
\declaretheorem[name=Assumption,  style=normalsty, numbered=no]{assumption*}

\makeatletter
\providecommand\IfFormatAtLeastTF{\@ifl@t@r\fmtversion}
\IfFormatAtLeastTF{2020/10/02}{%
  \usepackage[bookmarksopen=true,bookmarks=true,backref=page,colorlinks=truebreaklinks]{hyperref}
  \usepackage[capitalise, nameinlink]{cleveref}
  \renewcommand*{\backref}[1]{}
  \renewcommand*{\backrefalt}[4]{({\footnotesize%
  \ifcase #1 Not cited.%
    \or page~#2%
    \else pages #2%
  \fi%
  })}

}{%
  \usepackage{hyperref}
  \usepackage[capitalise, nameinlink]{cleveref}%
}
\makeatother

\crefname{assumption}{A}{A}
\crefname{appendix}{App.}{Apps.}
\crefname{algorithm}{Alg.}{Algs.}
\crefname{equation}{Eq.}{Eqs.}
\crefname{figure}{Fig.}{Figs.}
\crefname{tabular}{Tab.}{Tabs.}
\crefname{section}{\S}{\S}

\Crefname{assumption}{Assumption}{Assumptions}
\Crefname{appendix}{Appendix}{Appendices}
\Crefname{algorithm}{Algorithm}{Algorithms}
\Crefname{equation}{Equation}{Equations}
\Crefname{figure}{Figure}{Figures}
\Crefname{tabular}{Table}{Tables}
\Crefname{section}{Section}{Sections}

\definecolor{linkcolor}{HTML}{6929C4}
\definecolor{citecolor}{HTML}{0043CE}
\hypersetup{colorlinks={true},linkcolor=linkcolor,citecolor=citecolor}

\crefname{assumption}{Assumption}{Assumption}
\crefname{condition}{Condition}{Condition}

\crefname{framedtheorem}{Theorem}{Theorems}
\crefname{framedproposition}{Proposition}{Propositions}
\crefname{framedlemma}{Lemma}{Lemmas}

\makeatletter
\def\adl@drawiv#1#2#3{%
        \hskip.5\tabcolsep
        \xleaders#3{#2.5\@tempdimb #1{1}#2.5\@tempdimb}%
                #2\z@ plus1fil minus1fil\relax
        \hskip.5\tabcolsep}
\newcommand{\cdashlinelr}[1]{%
  \noalign{\vskip\aboverulesep
           \global\let\@dashdrawstore\adl@draw
           \global\let\adl@draw\adl@drawiv}
  \cdashline{#1}
  \noalign{\global\let\adl@draw\@dashdrawstore
           \vskip\belowrulesep}}
\makeatother

\pgfkeys{/prAtEnd/global custom defaults/.style={
    end,
    restate,
    text proof={Proof},
    text link={\textit{Proof.} See the \hyperref[proof:prAtEnd\pratendcountercurrent]{\textit{full proof}} in page~\pageref{proof:prAtEnd\pratendcountercurrent}.}
  }
}

\DeclareMathOperator*{\minimize}{minimize}

\DeclareMathOperator*{\argmin}{arg\,min} 

\newcommand*\xbar[1]{%
  \hbox{%
    \vbox{%
      \hrule height 0.6pt 
      \kern0.33ex
      \hbox{%
        \kern-0.1em
        \ensuremath{#1}%
        \kern-0.1em
      }%
    }%
  }%
}

\newcommand{\DKL}[2]{\mathrm{D}_{\mathrm{KL}}(#1,#2)}

\newcommand{\norm}[1]{{\left\lVert #1 \right\rVert}}
\newcommand{\abs}[1]{{\left| #1 \right|}}

\def\do#1{\csdef{v#1}{\mathvec{#1}}}
\docsvlist{a,b,c,d,e,f,g,h,i,j,k,l,m,n,o,p,q,r,s,t,u,v,w,x,y,z}
\docsvlist{A,B,C,D,E,F,G,H,I,J,K,L,M,N,O,P,Q,R,S,T,U,V,W,X,Y,Z}

\def\do#1{\csdef{v#1}{\mathvecgreek{\csname#1\endcsname}}}
\docsvlist{alpha,beta,gamma,delta,epsilon,zeta,eta,theta,iota,kappa,lambda,mu,nu,xi,omikron,pi,rho,sigma,tau,upsilon,phi,chi,psi,omega}

\def\do#1{\csdef{rv#1}{\mathrv{#1}}}
\docsvlist{a,b,c,d,e,f,g,h,i,j,k,l,m,n,o,p,q,r,s,t,u,v,w,x,y,z}
\docsvlist{A,B,C,D,E,F,G,H,I,J,K,L,M,N,O,P,Q,R,S,T,U,V,W,X,Y,Z}

\def\do#1{\csdef{rv#1}{\mathrvgreek{\csname#1\endcsname}}}
\docsvlist{alpha,beta,gamma,delta,epsilon,zeta,eta,theta,iota,kappa,lambda,mu,nu,xi,omikron,pi,rho,sigma,tau,upsilon,phi,chi,psi,omega}

\def\do#1{\csdef{rvv#1}{\mathrvvec{#1}}}
\docsvlist{a,b,c,d,e,f,g,h,i,j,k,l,m,n,o,p,q,r,s,t,u,v,w,x,y,z}

\def\do#1{\csdef{rvv#1}{\mathrvvecgreek{\csname#1\endcsname}}}
\docsvlist{alpha,beta,gamma,delta,epsilon,zeta,eta,theta,iota,kappa,lambda,mu,nu,xi,omikron,pi,rho,sigma,tau,upsilon,phi,chi,psi,omega}

\def\do#1{\csdef{rvm#1}{\mathrvvecgreek{#1}}}
\docsvlist{A,B,C,D,E,F,G,H,I,J,K,L,M,N,O,P,Q,R,S,T,U,V,W,X,Y,Z}

\def\do#1{\csdef{m#1}{\mathmat{#1}}}
\docsvlist{A,B,C,D,E,F,G,H,I,J,K,L,M,N,O,P,Q,R,S,T,U,V,W,X,Y,Z}

\def\do#1{\csdef{m#1}{\mathmatgreek{\csname#1\endcsname}}}
\docsvlist{Sigma,Lambda,Phi,Gamma,Theta,Delta}

\definecolor{ibmred100}{rgb}{0.17647058823529413,0.027450980392156862,0.03529411764705882}
\definecolor{ibmred90}{rgb}{0.3215686274509804,0.01568627450980392,0.03137254901960784}
\definecolor{ibmred80}{rgb}{0.4588235294117647,0.054901960784313725,0.07450980392156863}
\definecolor{ibmred70}{rgb}{0.6352941176470588,0.09803921568627451,0.12156862745098039}
\definecolor{ibmred60}{rgb}{0.8549019607843137,0.11764705882352941,0.1568627450980392}
\definecolor{ibmred50}{rgb}{0.9803921568627451,0.30196078431372547,0.33725490196078434}
\definecolor{ibmred40}{rgb}{1.0,0.5137254901960784,0.5372549019607843}
\definecolor{ibmred30}{rgb}{1.0,0.7019607843137254,0.7215686274509804}
\definecolor{ibmred20}{rgb}{1.0,0.8431372549019608,0.8509803921568627}
\definecolor{ibmred10}{rgb}{1.0,0.9450980392156862,0.9450980392156862}

\definecolor{ibmmagenta100}{rgb}{0.16470588235294117,0.0392156862745098,0.09411764705882353}
\definecolor{ibmmagenta90}{rgb}{0.3176470588235294,0.00784313725490196,0.1411764705882353}
\definecolor{ibmmagenta80}{rgb}{0.4549019607843137,0.03529411764705882,0.21568627450980393}
\definecolor{ibmmagenta70}{rgb}{0.6235294117647059,0.09411764705882353,0.3254901960784314}
\definecolor{ibmmagenta60}{rgb}{0.8156862745098039,0.14901960784313725,0.4392156862745098}
\definecolor{ibmmagenta50}{rgb}{0.9333333333333333,0.3254901960784314,0.5882352941176471}
\definecolor{ibmmagenta40}{rgb}{1.0,0.49411764705882355,0.7137254901960784}
\definecolor{ibmmagenta30}{rgb}{1.0,0.6862745098039216,0.8235294117647058}
\definecolor{ibmmagenta20}{rgb}{1.0,0.8392156862745098,0.9098039215686274}
\definecolor{ibmmagenta10}{rgb}{1.0,0.9411764705882353,0.9686274509803922}

\definecolor{ibmpurple100}{rgb}{0.10980392156862745,0.058823529411764705,0.18823529411764706}
\definecolor{ibmpurple90}{rgb}{0.19215686274509805,0.07450980392156863,0.3686274509803922}
\definecolor{ibmpurple80}{rgb}{0.28627450980392155,0.11372549019607843,0.5450980392156862}
\definecolor{ibmpurple70}{rgb}{0.4117647058823529,0.1607843137254902,0.7686274509803922}
\definecolor{ibmpurple60}{rgb}{0.5411764705882353,0.24705882352941178,0.9882352941176471}
\definecolor{ibmpurple50}{rgb}{0.6470588235294118,0.43137254901960786,1.0}
\definecolor{ibmpurple40}{rgb}{0.7450980392156863,0.5843137254901961,1.0}
\definecolor{ibmpurple30}{rgb}{0.8313725490196079,0.7333333333333333,1.0}
\definecolor{ibmpurple20}{rgb}{0.9098039215686274,0.8549019607843137,1.0}
\definecolor{ibmpurple10}{rgb}{0.9647058823529412,0.9490196078431372,1.0}

\definecolor{ibmblue100}{rgb}{0.0,0.06666666666666667,0.2549019607843137}
\definecolor{ibmblue90}{rgb}{0.0,0.11372549019607843,0.4235294117647059}
\definecolor{ibmblue80}{rgb}{0.0,0.17647058823529413,0.611764705882353}
\definecolor{ibmblue70}{rgb}{0.0,0.2627450980392157,0.807843137254902}
\definecolor{ibmblue60}{rgb}{0.058823529411764705,0.3843137254901961,0.996078431372549}
\definecolor{ibmblue50}{rgb}{0.27058823529411763,0.5372549019607843,1.0}
\definecolor{ibmblue40}{rgb}{0.47058823529411764,0.6627450980392157,1.0}
\definecolor{ibmblue30}{rgb}{0.6509803921568628,0.7843137254901961,1.0}
\definecolor{ibmblue20}{rgb}{0.8156862745098039,0.8862745098039215,1.0}
\definecolor{ibmblue10}{rgb}{0.9294117647058824,0.9607843137254902,1.0}

\definecolor{ibmcyan100}{rgb}{0.023529411764705882,0.09019607843137255,0.15294117647058825}
\definecolor{ibmcyan90}{rgb}{0.00392156862745098,0.15294117647058825,0.28627450980392155}
\definecolor{ibmcyan80}{rgb}{0.0,0.22745098039215686,0.42745098039215684}
\definecolor{ibmcyan70}{rgb}{0.0,0.3254901960784314,0.6039215686274509}
\definecolor{ibmcyan60}{rgb}{0.0,0.4470588235294118,0.7647058823529411}
\definecolor{ibmcyan50}{rgb}{0.06666666666666667,0.5725490196078431,0.9098039215686274}
\definecolor{ibmcyan40}{rgb}{0.2,0.6941176470588235,1.0}
\definecolor{ibmcyan30}{rgb}{0.5098039215686274,0.8117647058823529,1.0}
\definecolor{ibmcyan20}{rgb}{0.7294117647058823,0.9019607843137255,1.0}
\definecolor{ibmcyan10}{rgb}{0.8980392156862745,0.9647058823529412,1.0}

\definecolor{ibmteal100}{rgb}{0.03137254901960784,0.10196078431372549,0.10980392156862745}
\definecolor{ibmteal90}{rgb}{0.00784313725490196,0.16862745098039217,0.18823529411764706}
\definecolor{ibmteal80}{rgb}{0.0,0.2549019607843137,0.26666666666666666}
\definecolor{ibmteal70}{rgb}{0.0,0.36470588235294116,0.36470588235294116}
\definecolor{ibmteal60}{rgb}{0.0,0.49019607843137253,0.4745098039215686}
\definecolor{ibmteal50}{rgb}{0.0,0.615686274509804,0.6039215686274509}
\definecolor{ibmteal40}{rgb}{0.03137254901960784,0.7411764705882353,0.7294117647058823}
\definecolor{ibmteal30}{rgb}{0.23921568627450981,0.8588235294117647,0.8509803921568627}
\definecolor{ibmteal20}{rgb}{0.6196078431372549,0.9411764705882353,0.9411764705882353}
\definecolor{ibmteal10}{rgb}{0.8509803921568627,0.984313725490196,0.984313725490196}

\definecolor{ibmgreen100}{rgb}{0.027450980392156862,0.09803921568627451,0.03137254901960784}
\definecolor{ibmgreen90}{rgb}{0.00784313725490196,0.17647058823529413,0.050980392156862744}
\definecolor{ibmgreen80}{rgb}{0.01568627450980392,0.2627450980392157,0.09019607843137255}
\definecolor{ibmgreen70}{rgb}{0.054901960784313725,0.3764705882352941,0.15294117647058825}
\definecolor{ibmgreen60}{rgb}{0.09803921568627451,0.5019607843137255,0.2196078431372549}
\definecolor{ibmgreen50}{rgb}{0.1411764705882353,0.6313725490196078,0.2823529411764706}
\definecolor{ibmgreen40}{rgb}{0.25882352941176473,0.7450980392156863,0.396078431372549}
\definecolor{ibmgreen30}{rgb}{0.43529411764705883,0.8627450980392157,0.5490196078431373}
\definecolor{ibmgreen20}{rgb}{0.6549019607843137,0.9411764705882353,0.7294117647058823}
\definecolor{ibmgreen10}{rgb}{0.8705882352941177,0.984313725490196,0.9019607843137255}

%% file: section_introduction.tex
\section{Introduction}
Sequential Monte Carlo (SMC; \citealp{dai_invitation_2022,delmoral_sequential_2006,chopin_introduction_2020}) is a general methodology for simulating Feynman-Kac models~\citep{delmoral_feynmankac_2004,delmoral_mean_2016}, which describe the evolution of distributions through sequential changes of measure.
When tuned well, SMC provides state-of-the-art performance in a wide range of modern problem settings, 
from inference in both state-space models and static models~\citep{dai_invitation_2022,chopin_introduction_2020,doucet_tutorial_2011,cappe_overview_2007},
to training deep generative models~\citep{arbel_annealed_2021,matthews_continual_2022,doucet_differentiable_2023,maddison_filtering_2017}, steering large language models~\citep{zhao_probabilistic_2024,lew_sequential_2023}, conditional generation from diffusion models~\citep{trippe_diffusion_2023,wu_practical_2023}, and solving inverse problems with diffusion model priors~\citep{cardoso_monte_2024,dou_diffusion_2024,achituve_inverse_2025}.

In practice, however, tuning SMC samplers is often a significant challenge. For example, for static models~\citep{chopin_sequential_2002,delmoral_sequential_2006}, one must tune the number of steps, number of particles, target distribution, and Markov kernel at each step, as well as criteria for triggering particle resampling.
Since the asymptotic variance of SMC samplers is additive over the steps~\citep{delmoral_sequential_2006,gerber_negative_2019,chopin_central_2004,webber_unifying_2019,bernton_schrodinger_2019}, all of the above must be tuned adequately \textit{at all times}; an SMC run will not be able to recover from a single mistuned step.
While multiple methods for adapting the path of intermediate targets have been proposed~\citep{zhou_automatic_2016,syed_optimised_2024}, especially in the AIS context~\citep{kiwaki_variational_2015,goshtasbpour_adaptive_2023,masrani_qpaths_2021,jasra_inference_2011}, 
methods and criteria for tuning the path proposal kernels are relatively scarce.


Markov kernels commonly used in SMC can be divided into two categories: those of the Metropolis-Hastings (\citealp{metropolis_equation_1953,hastings_monte_1970}) type, commonly referred to as \textit{adjusted} kernels, and \textit{unadjusted} kernels. 
For tuning adjusted kernels, one can leverage ideas from the adaptive Markov chain Monte Carlo (MCMC;~\citealp{robert_monte_2004}) literature, such as controlling the acceptance probability~\citep{andrieu_controlled_2001,atchade_adaptive_2005} or maximizing the expected-squared jump distance~\citep{pasarica_adaptively_2010}. 
Both have previously been incorporated into adaptive SMC methods \citep{fearnhead_adaptive_2013,buchholz_adaptive_2021}. 
On the other hand, tuning unadjusted kernels, which 
have favorable high-dimensional convergence properties compared to their adjusted counterparts~\citep{lee_lower_2021,chewi_optimal_2021,roberts_optimal_1998,wu_minimax_2022,biswas_estimating_2019} and enable fully differentiable samplers~\citep{geffner_mcmc_2021,zhang_differentiable_2021,doucet_scorebased_2022}, is not as straightforward as most techniques from adaptive MCMC cannot be used.

%

Instead, the typical approach to tuning unadjusted kernels is to minimize a variational objective via stochastic gradient descent (SGD; \citealp{robbins_stochastic_1951,bottou_optimization_2018}) in an end-to-end fashion~\citep{doucet_scorebased_2022,goshtasbpour_optimization_2023,salimans_markov_2015,caterini_hamiltonian_2018,gu_neural_2015,arbel_annealed_2021,matthews_continual_2022,maddison_filtering_2017,geffner_mcmc_2021,heng_controlled_2020,chehab_provable_2023,geffner_langevin_2023,naesseth_variational_2018,le_autoencoding_2018,zenn_resampling_2023}.
End-to-end optimization approaches are costly: SGD typically requires at least thousands of iterations to converge (\textit{e.g}, \citealt{geffner_mcmc_2021} use \(1.5 \times 10^{5}\) SGD steps for tuning AIS), where each iteration itself involves an entire run of SMC/AIS. Moreover, SGD is sensitive to several tuning parameters, such as the step size, batch size, and initialization~\citep{sivaprasad_optimizer_2020}. 
But many of the unadjusted kernels, \textit{e.g.}, random walk MH~\citep{metropolis_equation_1953,hastings_monte_1970}, Metropolis-adjusted Langevin~\citep{rossky_brownian_1978,besag_comments_1994}, Hamiltonian Monte Carlo~\citep{duane_hybrid_1987,neal_mcmc_2011}, have only a few scalar parameters (\textit{e.g.}, step size) subject to tuning. In this setting, the full generality (and cost) of SGD is not required; it is possible to design a simpler and more efficient method for tuning each transition kernel sequentially in a single SMC/AIS run.



In this work, 
we propose a novel strategy for tuning path proposal kernels of SMC samplers.
Our approach is based on greedily minimizing the incremental Kullback-Leibler (KL; \citealp{kullback_information_1951}) divergence between the target and the proposal path measures at each SMC step (\cref{section:objective}).
This is reminiscent of annealed flow transport (AFT; \citealp{arbel_annealed_2021,matthews_continual_2022}), where a normalizing flow~\citep{papamakarios_normalizing_2021} proposal is trained at each step by minimizing the incremental KL.
Instead of training a whole normalizing flow, which requires expensive gradient-based optimization, we tune the parameters of off-the-shelf kernels at each step.
This simplifies the optimization process, leading to a gradient- and tuning-free step size adaptation algorithm with quantitative convergence guarantees  (\cref{section:stepsize}). 

Using our tuning scheme, we provide complete implementations of tuning-free adaptive SMC samplers for static models:
\begin{enumerate*}
    \item[(i)] SMC-LMC, which is based on Langevin Monte Carlo (LMC;~\citealp{rossky_brownian_1978,parisi_correlation_1981,grenander_representations_1994}), also commonly known as the unadjusted Langevin algorithm, and
    \item[(ii)] SMC-KLMC, which uses kinetic Langevin Monte Carlo with the ``OBABO'' discretization~\citep{duane_hybrid_1987,horowitz_generalized_1991,monmarche_highdimensional_2021}, also known as unadjusted generalized Hamiltonian Monte Carlo~\citep{neal_mcmc_2011}.
\end{enumerate*}
Our method achieves lower variance in normalizing constant estimates compared to the best fixed step sizes obtained through grid search or SGD-based tuning methods. 
Additionally, the step size schedules found by our method achieve lower or comparable variance than those found by end-to-end optimization approaches without involving any manual tuning (\cref{section:experiments}).




%% file: section_background.tex
\vspace{-1ex}
\section{Background}
\paragraph{Notation.}
Let \(\mathcal{B}\left(\mathcal{X}\right)\) be the set of Borel-measurable subsets of some set \(\mathcal{X} \subseteq \mathbb{R}^d\).
With some abuse of notation, we use the same symbol to denote both a distribution and its density.
Also, \(\log_+\left(x\right) \triangleq \log \max\left(x, 1\right)\), \({[\cdot]}_+ \triangleq \max\left(\cdot, 0\right) \), and \([T] \triangleq \{1, \ldots, T\} \).

\vspace{-1ex}
\subsection{SMC sampler and Feynman Kac Models}
\textit{Sequential Monte Carlo} (SMC; \citealp{dai_invitation_2022,delmoral_sequential_2006,chopin_introduction_2020}) is a general framework for sampling from Feynman-Kac models~\citep{delmoral_feynmankac_2004,delmoral_mean_2016}.
Consider a space $\mathcal{X}$ with a $\sigma$-finite base measure.  
Feynman-Kac models describe a change of measure between the \emph{target path distribution}
{%
\setlength{\belowdisplayskip}{1ex} \setlength{\belowdisplayshortskip}{1ex}
\setlength{\abovedisplayskip}{1ex} \setlength{\abovedisplayshortskip}{1ex}
\begin{align*}
    P^{\vtheta}_{0\text{:}T}(\mathrm{d} x_{0:T})
    \! \triangleq \!
    \frac{1}{Z_T^{\vtheta}} \Bigg\{ G_0\left(\vx_0\right) \prod^T_{t=1} G_t^{\vtheta}\left(\vx_{t-1}, \vx_t\right)\Bigg\}
    Q_{0\text{:}T}^{\vtheta}(\mathrm{d} \vx_{0:T})
\end{align*}
}%
and the
\emph{proposal path distribution}
{%
\setlength{\belowdisplayskip}{1ex} \setlength{\belowdisplayshortskip}{1ex}
\setlength{\abovedisplayskip}{1ex} \setlength{\abovedisplayshortskip}{1ex}
\begin{align*}
    Q^{\vtheta}_{0\text{:}T}(\mathrm{d}\vx_{0:T})
    \triangleq
    q\left(\mathrm{d}\vx_0\right) {\prod^T_{t=1}} K_t^{\vtheta}\left(\vx_{t-1}, \mathrm{d}\vx_t\right) \; , \; \text{where} 
\end{align*}
}%
\begin{center}
\vspace{-3ex}
   {\begingroup
    \setlength\tabcolsep{2pt} 
  \begin{tabular}{cp{0.8\linewidth}}
    \(q\) & is the \emph{reference} or initial proposal distribution, \\
    \({(K_t^{\vtheta})}_{t \in [T]}\) & are Markov kernels parameterized with \(\vtheta\), and \\
    \({(G_t^{\vtheta})}_{t \in [T]}\) & are non-negative \(Q\)-measurable functions referred to as \textit{potentials}. \\
  \end{tabular}
  \endgroup}
\vspace{-3ex}
\end{center}
The (intermediate) normalizing constant at time $t \in [T]$ is 
{%
\setlength{\belowdisplayskip}{1ex} \setlength{\belowdisplayshortskip}{1ex}
\setlength{\abovedisplayskip}{0ex} \setlength{\abovedisplayshortskip}{0ex}
\[
    Z_t^{\vtheta} = \int_{\mathcal{X}^{t+1}} G_0\left(\vx_0\right) \prod^t_{s = 1} G_{s}^{\vtheta}\left(\vx_{s-1}, \vx_{s}\right)
    Q^{\vtheta}_{0\text{:}t}(\mathrm{d}\vx_{0:t})
    \, .
\]
}%
The goal is often to draw samples from $P_{0:T}^{\vtheta}$ or to estimate the normalizing constant $Z_T^{\vtheta}$.

At time $t=0$, SMC draws $N$ particles $\rvvx_0^{1:N}$ from the initial proposal $q$, each assigned with equal weights $\rvvw_0^n = 1$ for $n \in [N]$.  
At each subsequent time $t \in[T]$, particles $\rvvx_{t-1}^{1:N}$ are transported via the transition kernel $K_t^\theta$, reweighted using the potentials $G^\theta_t$, and optionally resampled to discard particles with low weights. See the textbook by \citet{chopin_introduction_2020} for more details.

At each time $t \in [T]$, the SMC sampler outputs a set of weighted particles \((\bar{\rvw}_t^{1\text{:}N}, \rvvx_t^{1\text{:}N})\), where \(\bar{\rvw}_t^n \triangleq \rvw_t^n / \sum_{m \in [N]} \rvw^{m}_{t}\), along with an estimate of the normalizing constant \(\widehat{\rvZ}_{t, N}\).
Under suitable conditions, SMC samplers return consistent estimates of the expectation \(P_t\left(\varphi\right)\) of a measurable function $\varphi: \mathcal{X} \to \mathbb{R}$ over the marginal $P_t$ and the normalizing constant $Z_t$~\citep{delmoral_feynmankac_2004,delmoral_mean_2016}:
{%
\setlength{\belowdisplayskip}{0ex} \setlength{\belowdisplayshortskip}{0ex}
\setlength{\abovedisplayskip}{1ex} \setlength{\abovedisplayshortskip}{1ex}
\begin{align*}
    \sum_{n \in [N]} \bar{\rvw}_t^n \, \varphi\left(\rvvx_t^n\right) 
    \xrightarrow{N \to \infty} P_t\left(\varphi\right)
    \;\;\text{and}\;\;
    \widehat{\rvZ}_{t, N} &\xrightarrow{N \to \infty} Z_t \, .
\end{align*}
}%
Different choices of $G_{0:T}$, $q$, and $K^\theta_{0:T}$ can describe the same target path distribution $P_{0:T}^\theta$ but result in vastly different SMC algorithm performance. Proper tuning is thus essential for achieving high efficiency and accuracy.

\vspace{-1ex}
\subsection{Sequential Monte Carlo for Static Models}\label{section:static_models}

In this work, we focus on SMC samplers for \emph{static models} 
 where we target a ``static'' distribution $\pi$, whose density \(\pi : \mathcal{X} \to \mathbb{R}_{> 0}\) is known up to a normalizing constant $Z$ through the unnormalized density function \(\gamma : \mathcal{X} \to \mathbb{R}_{> 0}\):
{%
\setlength{\belowdisplayskip}{1ex} \setlength{\belowdisplayshortskip}{1ex}
\setlength{\abovedisplayskip}{1ex} \setlength{\abovedisplayshortskip}{1ex}
\[
    \pi\left(\vx\right) \triangleq \frac{\gamma\left(\vx\right)}{Z}, \quad \text{where}\quad Z = {\textstyle\int}_{\mathcal{X}} \gamma\left(\vx\right) \mathrm{d}\vx \, .
\]
}%
This can be embedded into a sequential inference targeting a ``path'' of distributions 
\(\left(\pi_0, \ldots, \pi_T\right)\), 
where the endpoints are constrained as \(\pi_0 = q\) and \(\pi_T = \pi\).
It is common to choose the \emph{geometric annealing path} by setting the density of \(\pi_t\), for \(t \in \left\{ 0, \ldots, T \right\}\), as
{%
\setlength{\belowdisplayskip}{1ex} \setlength{\belowdisplayshortskip}{1ex}
\setlength{\abovedisplayskip}{1ex} \setlength{\abovedisplayshortskip}{1ex}
\begin{align}
    \pi_t\left(\vx\right) \propto \gamma_t\left(\vx\right) \triangleq {q\left(\vx\right)}^{1 - \lambda_t} {\gamma\left(\vx\right)}^{\lambda_t} ,
    \label{eq:geometric_annealing}
\end{align}
}%
where the ``temperature schedule'' \({(\lambda_t)}_{t \in \left\{0, \ldots, T\right\}}\) is monotonically increasing as \(0 = \lambda_0 < \ldots < \lambda_T = 1\). 

To implement an SMC sampler that simulates the path ${\left(\pi_t\right)}_{t \in [T]}$, we introduce a sequence of \emph{backward} Markov kernels ${(L^\theta_{t-1})}_{t \in [T]}$ (and refer to the ${(K^\theta_t)}_{t \in [T]}$ as \emph{forward} kernels). 
We can then form a Feynman-Kac model by setting the potential for \(t \geq 1\) as
{%
\setlength{\belowdisplayskip}{1ex} \setlength{\belowdisplayshortskip}{1ex}
\setlength{\abovedisplayskip}{1ex} \setlength{\abovedisplayshortskip}{1ex}
\begin{align}
    G_t^{\vtheta}\left(\vx_{t-1}, \vx_t\right) 
    =
    \frac{Z_{t-1}}{Z_t}
    \frac{\mathrm{d} \left( \pi_{t} \otimes L_{t-1}^{\vtheta} \right) }{\mathrm{d} \left( \pi_{t-1} \otimes K_{t}^{\vtheta} \right) } (\vx_{t-1}, \vx_t) \, .
    \label{eq:potential_measure_static_models}
\end{align}
}%
As long as the condition
{%
\setlength{\belowdisplayskip}{1ex} \setlength{\belowdisplayshortskip}{1ex}
\setlength{\abovedisplayskip}{1ex} \setlength{\abovedisplayshortskip}{1ex}
\begin{equation}
\pi_{t} \otimes L_{t-1}^{\vtheta} \ll \pi_{t-1} \otimes K_{t}^{\vtheta} \label{eq:backward_absolute_continuity}
\end{equation}
}%
holds for all \(t \geq 0\) and the Radon-Nikodym derivative can be evaluated pointwise, \cref{eq:potential_measure_static_models} is equivalent to 
{%
\setlength{\belowdisplayskip}{1ex} \setlength{\belowdisplayshortskip}{1ex}
\setlength{\abovedisplayskip}{1ex} \setlength{\abovedisplayshortskip}{1ex}
\begin{align}
    G_t^{\vtheta}\left(\vx_{t-1}, \vx_t\right) 
    &= \frac{ \gamma_{t}\left(\vx_t\right) L^{\theta}_{t-1}\left(\vx_{t}, \vx_{t-1}\right) }{ \gamma_{t-1}\left(\vx_{t-1}\right) K_{t}^{\theta}\left(\vx_{t-1}, \vx_{t}\right) } \, .
    \label{eq:potential_static_models}
\end{align}
}%
Other than the constraint \cref{eq:backward_absolute_continuity}, the choice of forward and backward kernels is a matter of design.
Typically, the forward kernel \(K_t^{\vtheta}\) is selected as a $\pi_t$-invariant (\textit{i.e.}, adjusted) MCMC kernel~\citep{delmoral_sequential_2006}, such that the particles following \(P_{t-1}\) are transported to approximately follow \(\pi_t\).
This Feynman-Kac model targets the path measure
{%
\setlength{\belowdisplayskip}{2ex} \setlength{\belowdisplayshortskip}{2ex}
\setlength{\abovedisplayskip}{2ex} \setlength{\abovedisplayshortskip}{2ex}
\[
    P_{0\text{:}T}^{\vtheta} (\mathrm{d}\vx_{0:T}) =  \pi\left(\mathrm{d}\vx_{T}\right) {\textstyle \prod^T_{t=1}} L_{t-1}^{\vtheta}\left(\vx_{t}, \mathrm{d}\vx_{t-1}\right) \, .
\]
}%
Then, the marginal of \(\rvvx_T\) is \(\pi\), and for \(t \in [T]\), the intermediate normalizing constant $Z^\theta_t$ is precisely \(Z_t = \int_{\mathcal{X}} \gamma_t\left(\vx\right) \mathrm{d}\vx\).





%% file: section_method.tex
\section{Adaptation Methodology}

\subsection{Adaptation Objective}\label{section:objective}
The variance of sequential Monte Carlo is minimized when the target path measure \(P\) and the proposal path measure \(Q\) are close together~\citep{delmoral_sequential_2006,gerber_negative_2019,chopin_central_2004,webber_unifying_2019,bernton_schrodinger_2019}.
A common practice has been to make them close by solving
{%
\setlength{\belowdisplayskip}{1ex} \setlength{\belowdisplayshortskip}{1ex}
\setlength{\abovedisplayskip}{1ex} \setlength{\abovedisplayshortskip}{1ex}
\[
    \minimize_{\vtheta}\; \DKL{Q^{\vtheta}_{0\text{:}T}}{P^{\vtheta}_{0\text{:}T}} \, .
\]
}%
In this work, we are interested in a scheme enabling efficient online adaptation within SMC samplers.
One could appeal to the chain rule of the KL divergence:
{%
\setlength{\belowdisplayskip}{0ex} \setlength{\belowdisplayshortskip}{0ex}
\setlength{\abovedisplayskip}{1ex} \setlength{\abovedisplayshortskip}{1ex}
\begin{align*}
    &\DKL{Q^{\vtheta}_{0\text{:}T}}{P^{\vtheta}_{0\text{:}T}} 
    \\
    &\;\;=
    \DKL{Q_0}{P_0}
    +
    \sum_{t  \in [T]} \mathbb{E}_{Q_{t-1}}  \{{ \DKL{Q_{t \mid t-1}}{P_{t \mid t-1}} } \} \, ,
\end{align*}
}%
and attempt to minimize the incremental KL terms.
Unfortunately, at each step of SMC, we have access to a particle approximation \(P_{t-1}\) but not the marginal path proposal \(Q_{t-1}\) due to resampling.
Instead, we can consider the \textit{forward/inclusive} KL divergence
{%
\setlength{\belowdisplayskip}{0ex} \setlength{\belowdisplayshortskip}{0ex}
\setlength{\abovedisplayskip}{1ex} \setlength{\abovedisplayshortskip}{1ex}
\begin{align*}
    &\DKL{P^{\vtheta}_{0\text{:}T}}{Q^{\vtheta}_{0\text{:}T}} 
    \\
    &\;\;=
    \DKL{P_0}{Q_0}
    +
    \sum_{t \in [T]} \mathbb{E}_{P_{t-1}}  \{{ \DKL{P_{t \mid t-1}}{Q_{t \mid t-1}} } \} \, .
\end{align*}
}%
Estimating the incremental forward KL divergence \(\DKL{P_{t \mid t-1}}{Q_{t \mid t-1}}\), however, is difficult due to the expectation taken over \(P_{t \mid t - 1}\), often resulting in high variance.
Therefore, we would like to have a proper divergence measure between the joint paths that
\begin{enumerate*}
    \item[(i)] decomposes into \(T\) incremental terms like the chain rule of the KL divergence, 
    \item[(ii)] is easy to estimate, just like the exclusive KL divergence.
\end{enumerate*}

Notice that naively summing the incremental exclusive KL divergences as
{%
\setlength{\belowdisplayskip}{0ex} \setlength{\belowdisplayshortskip}{0ex}
\setlength{\abovedisplayskip}{1ex} \setlength{\abovedisplayshortskip}{1ex}
\begin{alignat*}{4}
    &
    \mathrm{D}_{\mathsf{path}}\left(P_{0\text{:}T}, Q_{0\text{:}T}\right)
    \nonumber
    \\
    &
    \triangleq 
    \mathrm{D}_{\mathrm{KL}}\left( Q_{0}, P_{0} \right)
    +
    \sum_{t \in [T]}
    \mathbb{E}_{P_{0\text{:}t-1}} 
    \left\{ 
        \mathrm{D}_{\mathrm{KL}} \left( Q_{t \mid 0\text{:}t-1}, P_{t \mid 0\text{:}t - 1} \right)
    \right\} ,
\end{alignat*}
}%
satisfies both requirements and turns out to be a valid divergence between path measures:

\input{theorems/thm_path_divergence}

%
    

\vspace{-2ex}
\paragraph{Ideal Adaptation Scheme.}
Therefore, we propose to adapt SMC samplers by minimizing $\mathrm{D}_{\mathsfit{path}}$.
{%
\setlength{\belowdisplayskip}{1ex} \setlength{\belowdisplayshortskip}{1ex}
\setlength{\abovedisplayskip}{1ex} \setlength{\abovedisplayshortskip}{1ex}
\begin{align}
    \minimize_{\theta} \; 
    \mathrm{D}_{\mathsf{path}}\left(P_{0\text{:}T}^{\theta}, Q_{0\text{:}T}^{\theta}\right) \, .
    \label{eq:objective}
\end{align}
}%
The key convenience of this objective is that for most cases that we will consider, the tunable parameters \(\vtheta\) decompose into a sequence of subsets \(\vtheta = \left(\theta_1, \ldots, \theta_T\right)\), where at any \(t \in [T]\), \(K_t\) and \(G_t\) depend on only \(\theta_{1\text{:}t}\) while \(\theta_t\) dominates their variance contribution.
This suggests a greedy scheme where we solve for a subset of parameters at a time.
By fixing \(\vtheta_{1\text{:}t-1}\) from previous iterations, we solve for
{%
\setlength{\belowdisplayskip}{0ex} \setlength{\belowdisplayshortskip}{0ex}
\setlength{\abovedisplayskip}{1ex} \setlength{\abovedisplayshortskip}{1ex}
\begin{align}
    \theta_t
    =
    \argmin_{\theta_t} \; \mathbb{E}_{P_{t-1}^{\theta_{1\text{:}t-1}}} \left\{ \DKL{ Q_{t \mid t-1}^{\theta_{1\text{:}t}} }{ P_{t \mid t - 1}^{\theta_{1\text{:}t}} } \right\} \, .
    \label{eq:greedy_objective}
\end{align}
}%
This greedy strategy does not guarantee a solution to the joint optimization in \cref{eq:objective}.
However, as long as greedily setting \(\theta_t\) does not negatively influence future and past steps, which is reasonable for the kernels we consider, this strategy should yield a good approximate solution.

\vspace{-2ex}
\paragraph{Relation with Annealed Flow Transport.}
For the static model case, \citet{arbel_annealed_2021} noted that the objective in \cref{eq:greedy_objective} approximates 
{%
\setlength{\belowdisplayskip}{1ex} \setlength{\belowdisplayshortskip}{1ex}
\setlength{\abovedisplayskip}{1ex} \setlength{\abovedisplayshortskip}{1ex}
\begin{align}
    \hspace{-0.5em}
    \mathbb{E}_{\rvvx_{t-1} \sim \pi_{t-1}}
    \! \left\{
    \mathrm{D}_{\mathrm{KL}}
    (
        \pi_{t-1} \otimes K_{t}^{\vtheta_{1\text{:}t}} , \,
        \pi_t \otimes L_{t-1}^{\vtheta_{1\text{:}t}}
        \!\mid\!
        \rvvx_{t-1}
    )
    \right\} \, .
    \label{eq:static_model_objective_approximation}
\end{align}
}%
Furthermore,~\citet[\S 3]{matthews_continual_2022} showed that, when \(K_t\) is taken to be a normalizing flow \(\mathcal{F}_t\)~\citep{papamakarios_normalizing_2021} and \(L_{t-1} = \mathcal{F}_t^{-1}\), there exists a joint objective associated with \cref{eq:static_model_objective_approximation}, 
{%
\setlength{\belowdisplayskip}{1ex} \setlength{\belowdisplayshortskip}{1ex}
\setlength{\abovedisplayskip}{1ex} \setlength{\abovedisplayshortskip}{1ex}
\begin{align}
    \mathrm{D}_{\mathrm{KL}} \left( {\textstyle\prod^{T}_{t=1}} \mathcal{F}^{\#} \pi_{t-1} , \; {\textstyle\prod^T_{t=1}} \pi_t \right) \, ,
    \label{eq:static_model_joint_objective_approximation}
\end{align}
}%
where \(\mathcal{F}_t^{\#}\pi_{t-1}\) is the pushforward measure of \(\pi_{t-1}\) pushed through \(\mathcal{F}_t\).
Our derivation of \cref{eq:greedy_objective} shows that it is not just minimizing an approximation to some joint objective as~\cref{eq:static_model_joint_objective_approximation}, but a proper divergence between the joint target \(P\) and joint path \(Q\).
This general principle applies to all Feynman-Kac models, not just those for static models.

\vspace{-2ex}
\paragraph{Incremental KL Objective for Feynman-Kac Models.}
For Feynman-Kac models, \cref{eq:greedy_objective} takes the form
{\small%
\setlength{\belowdisplayskip}{1ex} \setlength{\belowdisplayshortskip}{1ex}
\setlength{\abovedisplayskip}{1ex} \setlength{\abovedisplayshortskip}{1ex}
\begin{align*}
    &\mathbb{E}_{P_{1\text{:}t-1}^{\theta_{1\text{:}t-1}}} \left\{ \DKL{ Q_{t \mid 1\text{:}t-1}^{\theta_{1\text{:}t}} }{ P_{t \mid 1\text{:}t - 1}^{\theta_{1\text{:}t}} } \right\}
    \\
    &=
    \int
    \int \log\left( \frac{ \mathrm{d} Q_{t \mid t-1}^{\theta_{1\text{:}t}} }{ \mathrm{d} P_{t \mid t-1}^{\theta_{1\text{:}t}} } \right)
    \, \mathrm{d} Q^{\theta_{1\text{:}t}}_{t \mid t - 1}
    \mathrm{d} P^{\theta_{1\text{:}t-1}}_{t-1}
    \\
    &=
    \int 
    \int
    \!
    - \log G_t^{\theta_{1\text{:}t}}\left(\vx_{t-1}, \vx_t\right)  K_t^{\theta_{1\text{:}t}}\left(\vx_{t-1}, \mathrm{d}\vx_t\right)
    \, 
    \mathrm{d} P^{\theta_{1\text{:}t-1}}_{t-1}
    \\
    &\qquad
    \!
    - \log \left( {Z^{\theta_{1\text{:}t}}_{t}}/{Z^{\theta_{1\text{:}t-1}}_{t-1}} \right) \, .
\end{align*}
}%
The normalizing constant ratio forms a telescoping sum such that the path divergence becomes
{%
\setlength{\belowdisplayskip}{0ex} \setlength{\belowdisplayshortskip}{0ex}
\setlength{\abovedisplayskip}{0ex} \setlength{\abovedisplayshortskip}{0ex}
\begin{alignat*}{4}
    &\mathrm{D}_{\mathsf{path}}\left(P_{0\text{:}T}^{\theta}, Q_{0\text{:}T}^{\theta}\right)
    = \DKL{Q_0}{P_0} - \log \frac{Z_{T}}{Z_{0}} 
    \\
    &\quad+
    \sum_{t \in [T]}
    \mathbb{E}_{(\rvvx_{t-1}, \rvvx_t) \sim P^{\theta_{1\text{:}t-1}}_{t-1} \otimes K^{\theta_{1\text{:}t}}_{t} }
    \left\{
    - \log G_t^{\theta_{1\text{:}t}}\left(\rvvx_{t-1}, \rvvx_t\right) 
    \right\}
    \, .
\end{alignat*}
}%
In practice, Feynman-Kac models are designed such that both \(Z_{T}\) and \(Z_{0}\) are fixed regardless of \(\vtheta\):
\(Z_0\) is the normalizing constant of \(q\), which is usually 1, and \(Z_T\) is the normalizing constant of the target \(P_{0\text{:}T}^{\vtheta}\).
Therefore, for such Feynman-Kac models, solving \cref{eq:greedy_objective} is equivalent to
{%
\setlength{\belowdisplayskip}{1ex} \setlength{\belowdisplayshortskip}{1ex}
\setlength{\abovedisplayskip}{1ex} \setlength{\abovedisplayshortskip}{1ex}
\begin{align}
    \hspace{-.3em}
    \theta_t
    \!=\!
    \argmin_{\theta_t}
    \mathbb{E}_{(\rvvx_{t-1}, \rvvx_t) \sim P^{\theta_{1\text{:}t-1}}_{t-1} \otimes K^{\theta_{1\text{:}t}}_{t} }
    \{
    - \!\log G_t^{\theta_{1\text{:}t}}\!\left(\rvvx_{t-1}, \rvvx_t\right) 
    \} .
    \label{eq:greedy_objective_feynman_kac}
\end{align}
}%

\subsection{General Adaptation Scheme}\label{section:adaptation}
\vspace{-1ex}
\paragraph{Estimating the Incremental KL Objective.}
Now that we have discussed our ideal objective for adaptation in \cref{eq:greedy_objective_feynman_kac}, we turn to estimating this objective in practice.
At each iteration \(t \in [T]\), we have access to a collection of weighted particles
{%
\setlength{\belowdisplayskip}{0ex} \setlength{\belowdisplayshortskip}{0ex}
\setlength{\abovedisplayskip}{0ex} \setlength{\abovedisplayshortskip}{0ex}
\[
    \sum_{n \in [N]} \frac{1}{Z_{t-1}^{\theta_{1\text{:}t-1}}} \rvw_{t-1}^n \delta_{\rvvx_{t-1}^n} \sim P_{t-1}^{\theta_{1\text{:}t-1}} \, 
\]
}%
up to a constant with respect to \(\theta_t\), \(Z_{t-1}^{\theta_{1\text{:}t-1}}\), where \(\delta_{\rvvx_{t-1}^n}\) is a Dirac measure centered on \(\rvvx_{t-1}^n\). 
Consider the case where sampling from \(K_t^{\theta_{1\text{:}t}}\) can be represented by a map \(M_t^{\vtheta_{1\text{:}t}} : \mathcal{X} \times \mathcal{E} \to \mathcal{X}\), where the randomness over the space \(\mathcal{E}\) following \(\psi : \mathcal{B}\left(\mathcal{E}\right) \to \mathbb{R}_{\geq 0}\) is captured by \(\rvvepsilon^n_t \sim \psi\) :
{%
\setlength{\belowdisplayskip}{1ex} \setlength{\belowdisplayshortskip}{1ex}
\setlength{\abovedisplayskip}{1ex} \setlength{\abovedisplayshortskip}{1ex}
\[
    \rvvx_{t}^n = M_t^{\vtheta_{1\text{:}t}}\left(\rvvx_{t-1}^n; \rvvepsilon^n_t\right) .
\]
}%
Then, up to a constant, we obtain a conditionally unbiased estimate of the expectation in \cref{eq:greedy_objective_feynman_kac} as a function of \(\vtheta\):
{%
\setlength{\belowdisplayskip}{0ex} \setlength{\belowdisplayshortskip}{0ex}
\setlength{\abovedisplayskip}{1ex} \setlength{\abovedisplayshortskip}{1ex}
\begin{align}
    &
    \hspace{-.5em}
    \widehat{\mathcal{L}}_t(\vtheta_t; \rvvx_{t-1}^{1\text{:}N}, \rvw^{1\text{:}N}_{t-1}, \rvvepsilon_t^{1\text{:}N})
    \nonumber
    \\
    &
    \hspace{-.5em}
    \;\triangleq
    -
    \sum_{n \in [N]} \bar{\rvw}^n_{t-1} \log G_t^{\vtheta_{1\text{:}t}}\left(\rvvx_{t-1}^n, M_t^{\vtheta_{1\text{:}t}}\left(\rvvx_{t-1}^n; \rvvepsilon^n_t\right)\right) .
    \label{eq:empirical_objective}
\end{align}
}%

\begin{figure}[t]
    \vspace{-1ex}
    \removelatexerror
    \input{algorithms/alg_adapted_smc}
    \vspace{-4ex}
\end{figure}

\vspace{-2ex}
\paragraph{Efficiently Optimizing the Objective.}
Directly optimizing \(\widehat{\mathcal{L}}_t\), however, is challenging:
\begin{enumerate*}
    \item[(i)] Evaluating \(\widehat{\mathcal{L}}_t\) takes \(\mathrm{O}(N)\) evaluations of the potential, which can be expensive.
    \item[(ii)] The expectation over the kernel \(K_t\) or,  equivalently, over \(\rvvepsilon_t^n \sim \psi\), is intractable. 
\end{enumerate*}
We address these issues as follows:
\begin{enumerate}[itemsep=0ex,leftmargin=1.5em]
    \vspace{-2ex}
    \item \textbf{Subsampling of Particles.}
    To reduce the \(\mathrm{O}(N)\) cost of evaluating \(\widehat{\mathcal{L}}_t\), we apply resampling over the particles according to the weights \(\rvw_{t-1}^{1\text{:}N}\) such that we end up with a smaller subset of particles of size \(B \ll N\), which remains a valid approximation of \(P_{t-1}^{\theta_{1\text{:}t-1}}\).
    Then, evaluating \(\widehat{\mathcal{L}}_t\) takes \(\mathrm{O}(B)\) evaluations of the potential.
    
    \item \textbf{Sample Average Approximation.}
    Properly minimizing the expectation over \(K_t\) requires stochastic optimization algorithms, which introduce numerous challenges related to convergence determination, step size tuning, handling instabilities, and such.
    Instead, we draw a single batch of randomness \({(\rvvepsilon_t^{b})}_{b \in [B]}\), and fix it throughout the optimization procedure.
    This sample average approximation (SAA; \citealp{kim_guide_2015})
    introduces bias in the optimized solution but enables the use of more reliable deterministic techniques.
    
    \item \textbf{Regularization.}
    Subsampling the particles results in a higher variance for estimating the objective.
    We counteract this by adding a weighted regularization term \(\tau \, \mathrm{reg}\left(\theta_t\right)\) to the objective.
    For example, for the case of step sizes at \(t > 1\) such that \(\theta_t\) contains \(h_t\), we will set \(\tau \, \mathrm{reg}\left(h_t\right) = \tau \abs{\log h_t - \log h_{t-1}}^2\), which has a smoothing effect over the tuned step size schedule.
    This also makes the objective ``more convex,'' easing optimization.
    For time \(t = 1\), where we don't have \(h_{t-1}\), we use a guess \(h_0\) instead. 
    Effective values of \(\tau\) depend on the type of kernel in question, but not much on the target problem.
    We thus used a fixed value (\cref{section:config}) throughout all our experiments.
    \vspace{-2ex}
\end{enumerate}
The high-level workflow of the proposed adaptive SMC scheme is shown in~\cref{alg:adaptive_smc}.
The notable change is the addition of the adaptation step in Line 3 (\textcolor{blue!50}{colored region}), where the tunable parameters to be used at time \(t\) are tuned to perform best at the \(t\)th SMC step, which follows the ``pre-tuning'' principle of \citet{buchholz_adaptive_2021}.
In contrast, retrospective tuning~\citep{fearnhead_adaptive_2013}, which uses parameters that performed well in the previous step, forces SMC to run with suboptimal parameters at all times.

\begin{figure}[t]
    \vspace{-1ex}
    \removelatexerror
    {\small
        \input{algorithms/alg_adapt_stepsize}
    }
    \vspace{-4ex}
\end{figure}

\vspace{-1ex}
\subsection{Algorithm for Step Size Tuning}\label{section:stepsize}
\vspace{-1ex}
Recall that for SMC samplers applied to static models (\cref{section:static_models}), the path proposal kernel is typically chosen to be an MCMC kernel.
For most popular MCMC kernels such as random walk MH~\citep{metropolis_equation_1953,hastings_monte_1970} or Metropolis-adjusted Langevin (MALA; \citealp{besag_comments_1994,rossky_brownian_1978}), the crucial tunable parameter is a scalar-valued parameter called the \textit{step size} denoted as \(h_t > 0\) for \(t \in [T]\).
In this section, we will describe a general procedure for tuning such step sizes.

\begin{figure}[t]
    \centering
    \scalebox{0.9}{
    \input{figures/tikz_assumption_visualization}
    }
    \vspace{-4ex}
    \caption{\textbf{Illustration of \cref{assumption:initialization}.} The solid line is the 
empirical objective \(\widehat{\mathcal{L}}_t\) for the LMC kernel computed using the Bones model from PosteriorDB at time \(t = 1\).}\label{fig:assumption}
    \vspace{-2ex}
\end{figure}
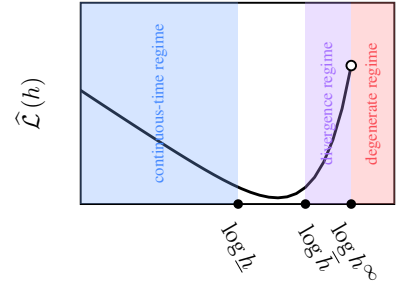

\vspace{-1ex}
\paragraph{\texttt{AdaptStepsize}.}
The adaptation routine is shown in \cref{alg:adapt_stepsize}.
First, in Line 1 and 2, we convert the optimization space to log-space; from \((0, \infty)\) to \((-\infty, \infty)\).
At the SMC iteration \(t = 1\), \(h_{\mathsf{guess}}\) is provided by the user.
Here, it is unsafe to immediately trust \(h_{\mathsf{guess}}\) to be non-degenerate (\(\mathcal{L}\left(h_{\mathsf{guess}}\right) < \infty\)).
Therefore, \(\mathrm{FindFeasible}\) in Line 4 ensures that \(\mathcal{L}\left(\exp\left(\ell\right)\right) < \infty\).
At time \(t > 1\), we set \(h_{\mathsf{guess}} = h_{t-1}\), which should be non-degenerate as long as adaptation at time \(t - 1\) went successfully.
Then we proceed to optimization in \(\mathtt{Minimize}\) (\cref{alg:minimize}), which mostly relies on the \emph{golden section search} algorithm (GSS; \citealp{avriel_golden_1968,kiefer_sequential_1953}), a gradient-free 1-dimensional optimization method.
GSS deterministically achieves an absolute tolerance of \(\epsilon > 0\).
Since we optimize in log-space, this translates to a natural \textit{relative} tolerance \(\mathrm{e}^{\pm\epsilon/2}\) with respect to the minimizer of \(\mathcal{L}\).
In our implementation and choice of \(r, c, \epsilon\) (described in \cref{section:config}), this procedure terminates after around 10 objective evaluations for \(t > 1\) and few tens of iterations for \(t = 1\).
For an in-depth discussion on the algorithm, please refer to \cref{section:algorithms}. 

\subsection{Analysis of the Algorithm for Step Size Tuning}\label{section:theory}

We provide quantitative performance guarantees of the presented step size adaptation procedures.
To theoretically model various degeneracies that can happen in the large step size regime, we will assume that the objective function \(\mathcal{L}\) takes the value of \(+\infty\) beyond some threshold.
In practice, whenever a numerical degeneracy is detected when evaluating \(\log \gamma\) (\(\mathtt{NaN}\) or \(-\infty\)), we ensure that the objective value is accordingly set as \(\infty\).
Our algorithm can deal with such cases by design, as reflected in the following assumptions:
\vspace{1ex}
\begin{assumption}\label{assumption:initialization}
    For the objective \(\mathcal{L} : (0, \infty) \to \mathbb{R} \cup \{+\infty\}\), we assume the following:
    \vspace{-2ex}
    \begin{enumerate}[itemsep=0ex]
        \item[(a)] There exists some \(h^{\infty} \in (0, \infty]\) such that \(\mathcal{L}\) is finite and continuous on \((0, h^{\infty})\) and \(+\infty\) on \([h^{\infty}, \infty)\).
        
        \item[(b)] There exists some \(\underline{h} \in (0, h^{\infty})\) such that \(\mathcal{L}\) is strictly monotonically decreasing on \((0, \underline{h}]\).
        
        \item[(c)] There exists some \(\overline{h} \in [\underline{h}, h^{\infty})\) such that \(\mathcal{L}\) is strictly monotonically increasing on \([\overline{h}, h^{\infty})\)
    \end{enumerate}
\vspace{-2ex}
\end{assumption}
Assumption (a) stipulates that degenerate regions are never disconnected and only exist in the direction of large step sizes.
Assumptions (b) and (c) represent the intuition that when the step size is too small or too large, the MCMC kernels degenerate predictably.
Most of the MCMC kernels used in practice are based on time-discretized diffusions. 
In these cases, (b) is satisfied as they approach the continuous-time regime, while (c) will be satisfied as the discretization becomes unstable (divergence).
\cref{fig:assumption} validates this intuition on one of the examples.

\input{theorems/thm_adapt_stepsize}

This suggests that, ignoring the dependence on \(r, c\), the objective query complexity of our optimization procedure is \(\mathrm{O}\left(\log \left(\Delta/\epsilon\right)\right)\).
Here, \(\Delta\) represents the difficulty of the problem, where \(\Delta \geq \abs{\log \overline{h} - \log \underline{h}}\).
In essence, \(\abs{\log \overline{h} - \log \underline{h}}\) represents how ``multimodal'' the problem is.
In practice, however, many problems result in less pessimistic objective surfaces such as follows:
\vspace{.5ex}
\begin{assumption}\label{assumption:unimodal}
    \(\mathcal{L}\) is unimodal on \((0, h^{\infty})\).
\end{assumption}
\vspace{-.5ex}
This is equivalent to assuming (b) and (c) in \cref{assumption:initialization} with \(\overline{h} = \underline{h}\) and implies there is a unique global minimum.
Then \cref{thm:initialize_stepsize} can be strengthened into the following:

\vspace{.5ex}
\begin{corollary}\label{thm:best_case_performance}
    Suppose \cref{assumption:general,assumption:unimodal} hold.
    Then \cref{thm:initialize_stepsize} holds, where \(\mathtt{AdaptStepsize}\left( \mathcal{L}, t, h_{\mathsf{guess}}, \delta, c, r, \epsilon \right)\) returns \(h \in (0, h^{\infty})\) that is \(\epsilon\)-close to the global optimum \(h^*\) and $\Delta = \abs{\log h^* - \log h_0}$.
\end{corollary}
\vspace{-.5ex}

Note that, at \(t > 1\), it is sensible to set \(h_{\mathsf{guess}} \leftarrow h_{t-1}\) since \(\pi_{t-1} \approx \pi_{t}\) by design.
Therefore, after \(t = 1\), \(\mathtt{AdaptStepsize}\) will run in a ``warm start'' regime where \(\Delta \approx 0\).
For instance, assume the initial guess is warm such that \(\abs{\log h_{\mathsf{guess}} - \log h^*} \leq \epsilon\) and $h_{\mathsf{guess}} \in (0, h^{\infty})$.
Then \cref{thm:best_case_performance} states that the number of objective evaluations will be $\mathrm{O}\left\{ \log_+\left( r^2 c \epsilon^{-1} + r^3\right) \right\}$.

The parameters of $\mathtt{AdaptStepsize}$, $r$ and $c$, must balance the performance of both the warm and cold start cases.
For a warm start, \(c r^{-1} = \mathrm{O}\left(\epsilon\right)\) optimizes performance.
For a cold start, \(r\) needs to be large enough to keep the \((\log r)^{-1}\) term in \(\mathcal{C}_{\mathsf{bm}}\) small.
Thus, leaning towards making \(c\) small and \(r\) moderately large balances both cases.
The values we use in the experiments are shown in \cref{section:setup_smc}.


\vspace{-1ex}
\section{Implementations}\label{section:implementations}
\vspace{-1ex}

Based on the procedure in \cref{section:stepsize}, we now describe complete implementations of adaptive SMC samplers.
Here, we will focus on the static model setting (\cref{section:static_models}), where the main objective is tuning of the MCMC kernels \({(K_t^{\vtheta})}_{t \in [T]}\).

\vspace{-1ex}
\subsection{SMC with Langevin Monte Carlo}\label{section:implementation_lmc}
First, we consider SMC with Langevin Monte Carlo (LMC;~\citealp{grenander_representations_1994,rossky_brownian_1978,parisi_correlation_1981}), also known as the unadjusted Langevin algorithm.
LMC forms a kernel \(K_t : \mathbb{R}^d \times \mathcal{B}\left(\mathbb{R}^d\right) \to \mathbb{R}_{>0}\) on the state space \(\mathcal{X} = \mathbb{R}^d\), which, for $s \geq 0$, simulates the Langevin stochastic differential equation (SDE)
{%
\setlength{\belowdisplayskip}{1ex} \setlength{\belowdisplayshortskip}{1ex}
\setlength{\abovedisplayskip}{1ex} \setlength{\abovedisplayshortskip}{1ex}
\begin{align}
    \mathrm{d}\rvvx_{s} = \nabla \log \pi_t\left(\rvvx_s\right) \mathrm{d}s + \sqrt{2} \, \mathrm{d}B_s,
    \label{eq:ula}
\end{align}
}%
where \({(B_s)}_{s \geq 0}\) is Brownian motion.
Under appropriate conditions on the target \(\pi_t\), it is well known that the stationary distribution of the process \({(\rvvx_s)}_{s \geq 0}\) is \(\pi_t\), where converges exponentially fast in total variation~\citep[Thm 2.1]{roberts_exponential_1996}.
The Euler-Maruyama discretization of \cref{eq:ula} yields a Markov kernel
{%
\setlength{\belowdisplayskip}{1ex} \setlength{\belowdisplayshortskip}{1ex}
\setlength{\abovedisplayskip}{1ex} \setlength{\abovedisplayshortskip}{1ex}
\[
    K_{t}^h\left(\vx, \mathrm{d}\vx^{\prime}\right) = \mathcal{N}\left(\mathrm{d}\vx^{\prime}; \, \vx + h \nabla \log \pi_t\left(\vx\right), 2 h \, \mathrm{I}_d \right) \, ,
\]
}%
where \(h > 0\) is the step size, which conveniently has a tractable density with respect to the Lebesgue measure.

Note that LMC is an \textit{approximate} MCMC algorithm; for any \(h > 0\), the stationary distribution of \(K_t^h\) is only \textit{approximately} \(\pi_t\). 
This contrasts with its MH-adjusted counterpart MALA~\citep{besag_comments_1994,roberts_exponential_1996,rossky_brownian_1978}, which is stationary on \(\pi_t\).

\vspace{-2ex}
\paragraph{Backward Kernel.}
For the sequence of backward kernels \({( L_{t-1}^{\vtheta} )}_{t = 2, \ldots, T}\), multiple choices are possible.
For instance, in the literature, a typical choice is \(L_{t-1}^{h_t} = K_{t}^{h_t}\).
In this work, we instead take the choice of
{%
\setlength{\belowdisplayskip}{1ex} \setlength{\belowdisplayshortskip}{1ex}
\setlength{\abovedisplayskip}{1ex} \setlength{\abovedisplayshortskip}{1ex}
\[
    L_{t-1}^{h_{t-1}}\left(\vx_t, \vx_{t-1}\right) \triangleq K_{t-1}^{h_{t-1}}\left(\vx_t, \vx_{t-1}\right) \, ,
\]
}%
which we call the ``time-correct forward kernel.''
Compared to more popular alternatives, this choice results in significantly lower variance.
(An in-depth discussion can be found in \cref{section:backward}.)
{%
\setlength{\belowdisplayskip}{1ex} \setlength{\belowdisplayshortskip}{1ex}
\setlength{\abovedisplayskip}{0ex} \setlength{\abovedisplayshortskip}{0ex}
The resulting potentials are
\begin{align*}
    G_1^{h_1}\left(\vx_{0}, \vx_1\right) &= \frac{\gamma_1\left(\vx_1\right)}{ K_1^{h_1}\left(\vx_0, \vx_1\right) } \\
    G_t^{h_{t-1},h_t}\left(\vx_{t-1}, \vx_t\right) &= 
    \frac{ \gamma_{t}\left(\vx_t\right) L^{h_{t-1}}_{t-1}\left(\vx_{t}, \vx_{t-1}\right) }{ \gamma_{t-1}\left(\vx_{t-1}\right) K_{t}^{h_t}\left(\vx_{t-1}, \vx_{t}\right) } \, ,
\end{align*}
}%
where, at each step \(t \in [T]\), we optimize for \(h_t\) using the general step size tuning procedure described in~\cref{section:stepsize} while the backward kernel re-uses the tuned parameter $h_{t-1}$ from the previous iteration.

\begin{figure}[t]
    \vspace{-2ex}
    \removelatexerror
    {\small
    \input{algorithms/alg_klmc_adapt_stepsize}
    }
    \vspace{-5ex}
\end{figure}

\begin{figure*}[h]
    \vspace{-1ex}
    \subfloat[Pines]{
        \hspace{-1em}
        \includegraphics[]{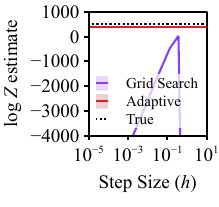}
        \vspace{-2ex}
    }
    \subfloat[Capture]{
        \hspace{-1em}
        \includegraphics[]{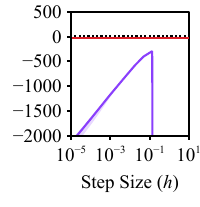}
        \vspace{-2ex}
    }
    \subfloat[Science]{
        \hspace{-1em}
        \includegraphics[]{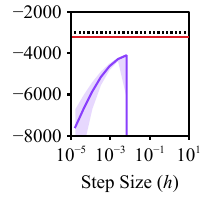}
        \vspace{-2ex}
    }
    \subfloat[Three Men]{
        \hspace{-1em}
        \includegraphics[]{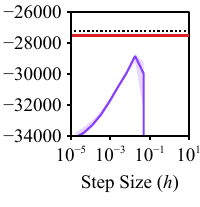}
        \vspace{-2ex}
    }
    \subfloat[Rats]{
        \hspace{-1em}
        \includegraphics[]{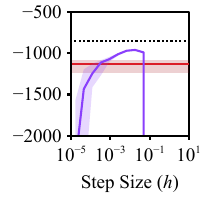}
        \vspace{-2ex}
    }
    \vspace{-1.5ex}
    \caption{
        \textbf{SMC-LMC with adaptive tuning v.s. fixed step sizes.}
        The solid lines are the median of the estimates of \(\log Z\), while the colored regions are the \(80\%\) empirical quantiles computed over 32 replications.
    }\label{fig:fixed_stepsize_smcula}
    \vspace{-2ex}
\end{figure*}

\begin{figure*}[h]
    \subfloat[Pines]{
        \hspace{-1em}
        \includegraphics[]{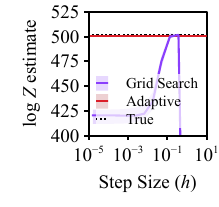}
        \vspace{-2ex}
    }
    \subfloat[Capture]{
        \hspace{-1em}
        \includegraphics[]{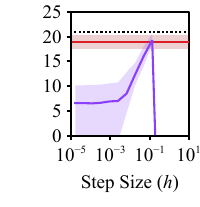}
        \vspace{-2ex}
    }
    \subfloat[Science]{
        \hspace{-1em}
        \includegraphics[]{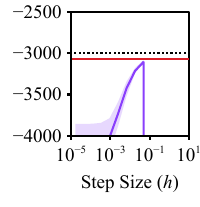}
        \vspace{-2ex}
    }
    \subfloat[Three Men]{
        \hspace{-1em}
        \includegraphics[]{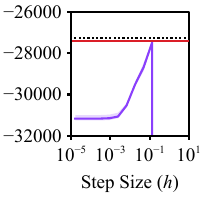}
        \vspace{-2ex}
    }
    \subfloat[Rats]{
        \hspace{-1em}
        \includegraphics[]{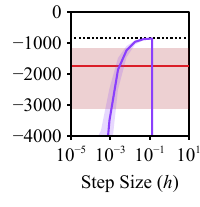}
        \vspace{-2ex}
    }
    \vspace{-1.5ex}
    \caption{
        \textbf{SMC-KLMC with adaptive tuning v.s. fixed step sizes and refreshment rates.}
        For SMC-KLMC with fixed parameters \(h, \rho\), we show the result of the best-performing refreshment rate.
        The solid lines are the median of the estimates of \(\log Z\), while the colored regions are the \(80\%\) empirical quantiles computed over 32 replications.
    }\label{fig:fixed_stepsize_smcuha}
    \vspace{-3ex}
\end{figure*}

\vspace{-1ex}
\subsection{SMC with Kinetic Langevin Monte Carlo}\label{section:klmc}
\vspace{-1ex}
Next, we consider a variant of the LMC that operates on the augmented state space \(\mathcal{Z} = \mathcal{X} \times \mathcal{X}\), where, for \(t \geq 0\), each state of the Feynman-Kac model is denoted as \(\vz_t = (\vx_t, \vv_t) \in \mathcal{Z}\), \(\vx_t, \vv_t \in \mathcal{X}\), \(\mathcal{X} = \mathbb{R}^d\), and the target is
{%
\setlength{\belowdisplayskip}{1ex} \setlength{\belowdisplayshortskip}{1ex}
\setlength{\abovedisplayskip}{1ex} \setlength{\abovedisplayshortskip}{1ex}
\[
    \pi_t^{\mathsf{klmc}}\left(\vx, \vv\right) \triangleq \pi_t\left(\vx\right) \mathcal{N}\left(\vv; \mathrm{0}_d, \mathrm{I}_d\right) \, .
\]
}%
Evidently, the \(\vx\)-marginal of the augmented target is \(\pi\).
Therefore, a Feynman-Kac model targeting \(\pi_t^{\mathsf{klmc}}\) is also targeting \(\pi\) by design.
Kinetic Langevin Monte Carlo (KLMC; \citealp{horowitz_generalized_1991,duane_hybrid_1987}), also referred to as underdamped Langevin, for $s \geq 0$, is given by
{%
\setlength{\belowdisplayskip}{.5ex} \setlength{\belowdisplayshortskip}{.5ex}
\setlength{\abovedisplayskip}{.5ex} \setlength{\abovedisplayshortskip}{.5ex}
\begin{align*}
    \mathrm{d}\rvvx_s &=  \rvvv_s \mathrm{d}s
    \\
    \mathrm{d}\rvvv_s &= \nabla \log \pi\left(\rvvv_{s}\right) \mathrm{d}s -  \eta \rvvv_s \mathrm{d}s + \sqrt{2 \eta} \, \mathrm{d}B_s \, ,
\end{align*}
}%
where \(\eta > 0\) is a tunable parameter called the \textit{damping} coefficient.
The stationary distribution of the joint process \({(\rvvx_s, \rvvv_s)}_{s \geq 0}\) is then \(\pi^{\mathsf{klmc}}\).
This continuous time process corresponds to the ``Nesterov acceleration~\citep{nesterov_method_1983,su_differential_2016}'' of~\cref{eq:ula}~\citep{ma_there_2021}, meaning that the process should converge faster.
We thus expect KLMC to reduce the required number of steps \(T\) compared to LMC.

To simulate this, we consider the OBABO discretization~\citep{leimkuhler_rational_2013}, which operates in a Gibbs scheme~\citep{geman_stochastic_1984}:
its kernel
{%
\setlength{\belowdisplayskip}{0.5ex} \setlength{\belowdisplayshortskip}{0.5ex}
\setlength{\abovedisplayskip}{0.5ex} \setlength{\abovedisplayshortskip}{0.5ex}
\begin{align*}
    &K_t\left( \vz_{t-1}, \mathrm{d}\vz_t\right)
    \\
    &\;=
    R^{\rho}\left( \vv_{t-1}, \mathrm{d}\vv_{t - \nicefrac{1}{2}} \right)
    S_t^{h, L}\left( 
        \left(\vx_{t-1}, \vv_{t - \nicefrac{1}{2}} \right), 
        \left( \mathrm{d} \vx_{t}, \mathrm{d} \vv_{t} \right)
    \right)
\end{align*}
}%
is a composition of the \textit{momentum refreshment kernel}
{%
\setlength{\belowdisplayskip}{0.5ex} \setlength{\belowdisplayshortskip}{0.5ex}
\setlength{\abovedisplayskip}{0.5ex} \setlength{\abovedisplayshortskip}{0.5ex}
\begin{align*}
    R^{\rho}\left(\vv_{t-1}, \mathrm{d}\vv_{t - \nicefrac{1}{2}}\right) \triangleq \mathcal{N}\left(\mathrm{d}\vv_{t - \nicefrac{1}{2}}; \; \sqrt{1 - \rho^2} \, \vv_{t-1}, \rho^2 \, \mathrm{I}_d \right) \, , 
\end{align*}
where \(\rho \triangleq 1 - \exp\left( - \eta h \right) \in (0, 1) \) is the ``momentum refreshment rate'' for some step size \(h > 0\), and the \textit{Leapfrog integrator kernel}
{
\begin{align*}
    S^{h}_t\left(\left(\vx_{t-1}, \vv_{t - \nicefrac{1}{2}}\right), \cdot \right)
    \triangleq
    \delta_{
    \Phi_{t}^h\left( \vx_{t-1}, \vv_{t - \nicefrac{1}{2}} \right)
    } \left( \cdot \right) \, ,
\end{align*}
}%
}%
where \(\Phi_{t}^h\) is a single step of leapfrog integration with step size \(h\) preserving the ``Hamiltonian energy'' \(-\log \pi_t^{\mathsf{klmc}}\).
This discretization also coincides with the unadjusted version of the generalized Hamiltonian Monte Carlo~\citep{duane_hybrid_1987,neal_mcmc_2011} with a single leapfrog step.


\vspace{-2ex}
\paragraph{Backward Kernel.}
Since the kernel \(S^{h_t}\) is a deterministic mapping, \(K_t^{\theta_t}\) does not admit a density with respect to the Lebesgue measure.
Therefore, we are restricted to a specific backward kernel that satisfies the condition in \cref{eq:backward_absolute_continuity}:
Since the leapfrog integrator \(\Phi_{t}^h\) is a diffeomorphism, its inverse map \({(\Phi_{t}^h)}^{-1}\) exists.
Therefore, we can choose
{%
\setlength{\belowdisplayskip}{1ex} \setlength{\belowdisplayshortskip}{1ex}
\setlength{\abovedisplayskip}{1ex} \setlength{\abovedisplayshortskip}{1ex}
\begin{align*}
    \hspace{-.5em}
    L_{t-1}^{h, \rho}\left(\vz_{t}, \cdot\right) 
    = 
    \delta_{
    {(\Phi_{t}^h)}^{-1}\left( \vx_{t}, \vv_{t} \right)
    } \left(\left(\mathrm{d} \vx_t, \mathrm{d}\vv_{t - \nicefrac{1}{2}} \right) \right)
    R^{\rho}\left(\vv_{t - \nicefrac{1}{2}}, \mathrm{d}\vv_{t-1}\right) 
\end{align*}
}%
This results in the deterministic component of $K_t$ and $L_{t-1}$ being supported on the same pair of points, ensuring absolute continuity~\citep{doucet_scorebased_2022,geffner_langevin_2023}.
Then the potential is given by the Radon-Nikodym derivative between the momentum refreshments $R^{\rho}$ as
{\small
\setlength{\belowdisplayskip}{1ex} \setlength{\belowdisplayshortskip}{1ex}
\setlength{\abovedisplayskip}{1ex} \setlength{\abovedisplayshortskip}{1ex}
\begin{align*}
    &G_t^{h_t, \rho_t}\left(\vz_{t-1}, \vz_{t}\right) 
    \\
    &=
    \frac{
        \gamma_t\left(\vx_t\right) \mathcal{N}\left(\vv_t; 0_d, \mathrm{I}_d\right) \mathcal{N}\left(\vv_{t-1}; \sqrt{1 - \rho_t^2} \, \vv_{t- \nicefrac{1}{2}},  \rho_t^2 \, \mathrm{I}_d \right)
    }{
        \gamma_{t-1}\left(\vx_{t-1}\right) \mathcal{N}\left(\vv_{t-1}; 0_d, \mathrm{I}_d\right) 
        \mathcal{N}\left(\vv_{t - \nicefrac{1}{2}};  \sqrt{1 - \rho_t^2} \, \vv_{t-1}, \rho_t^2 \mathrm{I}_d \right)
    } \, 
\end{align*}
}%
with two tunable parameters: \( (h_t, \rho_t ) \in \mathbb{R}_{>0} \times (0, 1)\).

\vspace{-2ex}
\begingroup
\setlength{\columnsep}{2ex}%
\begin{wrapfigure}{r}{0.40\columnwidth}
    \vspace{-4ex}
    \centering
    \includegraphics[scale=0.9]{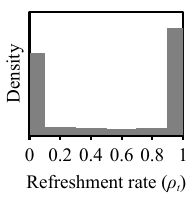}
    \vspace{-3ex}
    \caption{\textbf{Distribution of tuned refreshment rates \(\rho_t\).} The results were obtained by running adaptive SMC on the Sonar problem with \(T = 256\) and \(N = 1024\)}
    \vspace{-2ex}
    \label{fig:refreshment_rate}
\end{wrapfigure}
\paragraph{Adaptation Algorithm.}
As KLMC has two parameters, we cannot immediately apply the tuning procedure offered in \cref{section:stepsize}.
Thus, we will tailor it to KLMC.
At each \(t \in [T]\), we will minimize the incremental KL objective \(\widehat{\mathcal{L}}_t\left(h, \rho\right)\) through coordinate descent.
That is, we alternate between minimizing over \(h\) and \(\rho\).
This is shown in \cref{alg:klmc_adapt_stepsize}.
In particular, \(h_t\) is updated using the procedure used in~\cref{section:stepsize}, while \(\rho_t\) is directly minimized over a grid \(\Xi \in {(0, 1)}^{k}\) of \(k\) grid points.
As shown in \cref{fig:refreshment_rate}, empirically, the minimizers of \(\widehat{\mathcal{L}}_t\) with respect to \(\rho_t\) tend to concentrate on the boundary, as if the adaptation problem is determining to ``fully refresh'' or ``not refresh at all.''
Therefore, the grid \(\Xi\) can be made as coarse as \(\Xi = \left\{ 0.1, 0.9 \right\}\), which is what we use in the experiments.

\endgroup

%% file: theorems/thm_path_divergence.tex
\begin{theoremEnd}[category=pathdivergence]{proposition}\label{thm:pathdivergence}
    Consider joint distributions \(Q_{0\text{:}T}, P_{0\text{:}T}\). 
    Then \(\mathrm{D}_{\mathsf{path}}\) satisfies the following:
    \begin{enumerate}[itemsep=0ex]
        \vspace{-2ex}
        \item[(i)] \(\mathrm{D}_{\mathsf{path}}\left(P_{0\text{:}T}, Q_{0\text{:}T}\right) \geq 0\) for any \(Q_{0\text{:}T}, P_{0\text{:}T}\).
        \item[(ii)] \(\mathrm{D}_{\mathsf{path}}\left(P_{0\text{:}T}, Q_{0\text{:}T}\right) = 0\) if and only if \(P_{0\text{:}T} = Q_{0\text{:}T}\).
        \vspace{-2ex}
    \end{enumerate}
\end{theoremEnd}
\begin{proofEnd}
    (i) is trivial. (ii) follows from the fact that if \(P_{0\text{:}T} = Q_{0\text{:}T}\), the incremental KL divergences are all 0, while if \(P_{0\text{:}T} \neq Q_{0\text{:}T}\), \( \mathrm{D}_{\mathsf{path}}\left( Q_{0 \text{:} T}, P_{0 \text{:} T} \right) \geq \mathrm{D}_{\mathsf{path}}\left( Q_{t \mid 0 \text{:} t - 1}, P_{t \mid 0 : \text{t - 1}} \right) > 0\) for any \(t \in [T]\) by the fact that the conditional KL divergence is 0 if and only if \(Q_{t \mid 0 \text{:} t - 1} = P_{t \mid 0 \text{:} t - 1}\).
\end{proofEnd}

%% file: algorithms/alg_adapted_smc.tex
\begin{algorithm2e}[H]
\small
\caption{Adaptive Sequential Monte Carlo}\label{alg:adaptive_smc}
\RestyleAlgo{ruled} 
\LinesNumbered
\SetNoFillComment
\DontPrintSemicolon
\(x^n_0 \sim q, \quad w_0^n = 1, \quad \widehat{Z}_{0,N} \leftarrow 1\)\;
\For{\(t = 1, \ldots, T\)}{
    \tikzmk{A}
    \(\rvvepsilon^b_{t} \sim \psi\)\;
    \(\widetilde{\rva}_{t-1}^{1\text{:}B} = \mathtt{resample}_B\left(\rvw_{t-1}^{1\text{:}N}\right)\)\;
    \(\widetilde{\rvvx}_{t-1}^{b} = \rvvx_{t-1}^{\widetilde{a}_{t-1}^b}, \quad \widetilde{\rvw}_{t-1}^{b} = 1 \) \;
    \(\theta_t = \mathop{\argmin}_{\theta_t} \widehat{\mathcal{L}}_t\left(\theta_t; \widetilde{\rvvx}_{t-1}^{1\text{:}B}, \widetilde{\rvw}_{t-1}^{1\text{:}B}, \rvvepsilon_{t}^{1\text{:}B}\right) 
 + \tau \mathrm{reg}\left(\theta_t\right)\) \;
    \tikzmk{B}
    \boxit{blue!50}
    \(\rvvx^n_t \sim K_t^{\vtheta_{1\text{:}t}}\left(\rvvx_{t-1}^n, \cdot\right)\)\;
    \(\rvw_t^n \leftarrow \rvw_{t-1}^n G_t^{\vtheta_{1\text{:}t}}\left(\rvvx_{t-1}^n, \rvvx_t^n\right) \) \;
    \If{resampling is triggered or $t = T$}{
        \(\widehat{\rvZ}_{t,N} = \widehat{\rvZ}_{t-1,N} \frac{1}{N} \sum_{n \in [N]} \rvw_t^n \) \;
        \(\rva_t^{1\text{:}N} = \mathtt{resample}_N\left(\rvw_t^{1\text{:}N}\right)\)\;
        \(\rvvx^n_{t} \leftarrow \rvvx^{\rva_t^n}_{t}, \quad \rvw_t^n \leftarrow 1\)\;
    }
}
\end{algorithm2e}

%% file: algorithms/alg_adapt_stepsize.tex
\begin{algorithm2e}[H]
\caption{\(\mathtt{AdaptStepsize}\left( \mathcal{L}, t,  h_{\mathsf{guess}}, \delta, c, r, \epsilon \right)\)}\label{alg:adapt_stepsize}
\RestyleAlgo{ruled} 
\LinesNumbered
\DontPrintSemicolon
\KwIn{%
    Adaptation objective \(\mathcal{L} : (0, \infty) \to \mathbb{R} \cup \{+\infty\}\), \newline
    SMC iteration \(t \in [T]\), \newline
    initial guess \( h_{\mathsf{guess}} > 0 \), \newline
    backing-off step size \(\delta < 0\), \newline
    exponential search coefficient \(c > 0\), \newline
    exponential search exponent \(r > 1\), \newline
    absolute tolerance \(\epsilon > 0\).
}
\KwOut{ Adapted step size \(h\).}
\( \mathcal{L}^{\mathsf{log}}\left(\ell\right) \triangleq \mathcal{L}\left(\exp\left(\ell\right)\right) \) \;
\(\ell \leftarrow \log h_{\mathsf{guess}} \)\;
\If{\(t = 1\)}{
\( \ell \leftarrow \mathtt{FindFeasible}\left( \mathcal{L}^{\mathsf{log}}, \ell, \delta \right) \) \;
}
\( \ell \leftarrow \mathtt{Minimize}\left( \mathcal{L}^{\mathsf{log}}, \ell, c, r, \epsilon \right) \)\;
Return \(\exp\left(\ell\right)\)
\end{algorithm2e}

%% file: figures/tikz_assumption_visualization.tex
\tikzexternalenable

\begin{tikzpicture}
    \begin{semilogxaxis}[
        xtick style={draw=none},
        ytick style={draw=none},
        xtick = {0.35, 1.2, 2.1},
        xticklabels = {\(\log \underline{h}\), \(\log \overline{h}\), \(\log h^{\infty}\)},
        yticklabels = {},
        xtick align=outside,
        ytick align=outside,
        xlabel near ticks,
        ylabel near ticks,
        major tick length=2pt,
        minor tick length=1pt,
        axis line style = thick,
        every tick/.style={black,thick},
        log basis x=10,
        xtick pos=bottom,
        ytick pos=left,
        width=0.75\columnwidth,
        height=0.55\columnwidth,
        ylabel=$\widehat{\mathcal{L}}\left(h\right)$,
        xmin=3e-2,
        xmax=4,
        ymin=0,
        ymax=35,
        xticklabel style={
           tick label style={rotate=300}, 
       }
    ]
    
    \addplot[very thick, black] coordinates {
        (0.0009118819655545162, 46.45394600587368)
        (0.0010519151149398368, 45.32240355711508)
        (0.0012134524541956588, 44.19456958820021)
        (0.0013997962741296717, 43.07041644936007)
        (0.0016147559818205847, 41.94991451961002)
        (0.0018627259759257083, 40.83303209415609)
        (0.002148775480909723, 39.719735274677795)
        (0.0024787521766663585, 38.60998786617676)
        (0.002859401742022366, 37.50375128540637)
        (0.0032985057559390915, 36.400984487637956)
        (0.003805040775511363, 35.30164392080341)
        (0.0043893618427784395, 34.20568351904834)
        (0.005063414171757494, 33.11305475165227)
        (0.005840977343195246, 32.02370674841078)
        (0.006737946999085467, 30.937586529302344)
        (0.0077726597956035365, 29.8546393750796)
        (0.008966268257436799, 28.774809386978028)
        (0.01034317319661825, 27.698040298892643)
        (0.011931522535756143, 26.62427662526438)
        (0.013763786733050402, 25.553465254041498)
        (0.015877422572448632, 24.485557628391543)
        (0.01831563888873418, 23.42051270589356)
        (0.02112827988118328, 22.358300943083826)
        (0.0243728440732796, 21.298909630845795)
        (0.028115659748972035, 20.24235000801508)
        (0.032433240894795524, 19.188666714336225)
        (0.03741385136723659, 18.137950319617694)
        (0.0431593092614526, 17.09035389695494)
        (0.049787068367863944, 16.04611491240591)
        (0.05743261926761735, 15.005584105640196)
        (0.06625225915226167, 13.969263563337169)
        (0.07642628699076807, 12.937856871268528)
        (0.08816268936235745, 11.91233515737648)
        (0.10170139230422684, 10.894024184749734)
        (0.11731916609425078, 9.884719404135575)
        (0.1353352832366127, 8.886837704862947)
        (0.15611804531597107, 7.903616744471827)
        (0.1800923121479524, 6.939376662190448)
        (0.20774818714360085, 5.999866146844652)
        (0.2396510364417758, 5.092720409167542)
        (0.2764530466295643, 4.228065107823052)
        (0.31890655732397044, 3.4193145232987767)
        (0.36787944117144233, 2.6842226615457316)
        (0.42437284567695, 2.0462630898553837)
        (0.4895416595569531, 1.5364598040831245)
        (0.5647181220077593, 1.1958020920083237)
        (0.6514390575310556, 1.0784540487431291)
        (0.751477293075286, 1.2560073434497268)
        (0.8668778997501816, 1.8230784096770454)
        (1.0, 2.904689816213706)
        (1.1535649948951077, 4.666105765861593)
        (1.33071219744735, 7.325823079850404)
        (1.5350630092552098, 11.172786000261803)
        (1.770794952435155, 16.589245227535265)
        (2.042727070266142, 24.08097299729315)
    };
    \draw[fill=ibmblue40, fill opacity=0.3, draw=none] ({axis cs:0.03,0} |- {axis description cs:0,0}) 
             rectangle ({axis cs:0.35,0} |- {axis description cs:0,1}) node[pos=0.5,rotate=90,text opacity=1.0,text=ibmblue50] {\scriptsize{continuous-time regime}}
             ;
             
    \draw[fill=ibmpurple40, fill opacity=0.3, draw=none] ({axis cs:1.0,0} |- {axis description cs:0,0}) 
             rectangle ({axis cs:2.042727070266142,0} |- {axis description cs:0,1})
             node[pos=0.5,rotate=90,text opacity=1.0,text=ibmpurple50] {\scriptsize{divergence regime}}
             ;

    \draw[fill=ibmred40, fill opacity=0.3, draw=none] ({axis cs:2.042727070266142,0} |- {axis description cs:0,0}) 
             rectangle ({axis cs:4,0} |- {axis description cs:0,1})
             node[pos=0.5,rotate=90,text opacity=1.0,text=ibmred50] {\scriptsize{degenerate regime}}
             ;
             

   \addplot[mark=*,solid,fill=white,thick] coordinates {(2.042727070266142,24.08097299729315)};
   
    \addplot[dotted, thick, mark=*, black, draw opacity=0] coordinates{(2.042727070266142,0) (2.042727070266142,40)}; 
    
    \addplot[dotted, thick, mark=*, black, draw opacity=0] coordinates{(1.0,0) (1.0,40)};  
    
    \addplot[dotted, thick, mark=*, black, draw opacity=0] coordinates{(0.35,0) (0.35,40)}; 

\end{semilogxaxis}
\end{tikzpicture}

\tikzexternaldisable

%% file: theorems/thm_adapt_stepsize.tex
\begin{theoremEnd}[category=initializestepsize]{theorem}\label{thm:initialize_stepsize}
    Suppose \cref{assumption:initialization} holds.
    Then \(\mathtt{AdaptStepsize}\left( \mathcal{L}, t, h_{\mathsf{guess}}, \delta, c, r, \epsilon \right)\) returns a step size \(h \in (0, h^{\infty})\) that is \(\epsilon\)-close to a local minimum of \(\mathcal{L}\) in log-scale
    after \(\mathcal{C}_{\mathsf{feas}} + \mathcal{C}_{\mathsf{bm}} + \mathcal{C}_{\mathsf{gss}}\) objective evaluations, where, defining
    \(
        \Delta
        \triangleq
        \log_+\!\left({\overline{h} }/{ h_0 }\right)
        +
        \log_+\!\left({h_0}/{\underline{h}}\right)
    \)
    and
    \(h_0 \triangleq \min(h_{\mathsf{guess}}, h^{\infty})\),
{%
\setlength{\belowdisplayskip}{1ex} \setlength{\belowdisplayshortskip}{1ex}
\setlength{\abovedisplayskip}{1ex} \setlength{\abovedisplayshortskip}{1ex}
    \begin{align*}
        {\mathcal{C}}_{\mathsf{feas}}
        &= 
        \mathrm{O}\big\{ 
            {\delta}^{-1} \log_+ \left( \nicefrac{ h_{\mathrm{guess}}}{h^{\infty}} \right)
        \big\}
        \\
        {\mathcal{C}}_{\mathsf{bm}}
        &= 
        \mathrm{O}\big\{
        {\left(\log r\right)}^{-1}
        \log_+ \left(\Delta r c^{-1}\right)
        \big\}
        \\
        {\mathcal{C}}_{\mathsf{gss}}
        &= 
        \mathrm{O}\left\{ 
        \log_+ 
        \left(
        \left(
        r^3 \Delta + r^2 c
        \right)
        \epsilon^{-1}
        \right)
        \right\}  \; .
    \end{align*}
    }%
\end{theoremEnd}
\vspace{-1ex}
\begin{proofEnd}
    Since \cref{assumption:initialization} implies that the function \(\mathcal{L}^{\mathsf{log}}\left(h\right)\) satisfies \cref{assumption:general} with
    \begin{align*}
        f = \mathcal{L}^{\mathsf{log}}\left(h\right), \;\;
        \overline{x} = \log \overline{h}, \;\;
        \underline{x} = \log \underline{h}, \;\;
        x^{\infty} = \log h^{\infty} \; . 
    \end{align*}
    Then the result is a simple application of the lemmas in the previous sections.
    
    First, under \cref{assumption:initialization}, \cref{alg:find_feasible} can find a point \(\ell^{\prime} \in (-\infty, \log h^{\infty})\) that guarantees \(\mathcal{L}^{\mathsf{log}}\left(\ell\right) < \infty\) within 
    \begin{align*}
        \mathcal{C}_{\mathsf{feas}}
        &\leq
        \mathrm{O}\left(
        {\delta}^{-1} \log_+\left(h_{\mathrm{guess}} / h^{\infty}\right)  
        \right)
    \end{align*}
    steps.
    Furthermore, \(\ell^{\prime} = \log h_{\mathsf{guess}}\) if \(h_{\mathsf{guess}} < h^{\infty}\), and \(\ell^{\prime} < \log h^{\infty}\) otherwise.
    Then, \cref{thm:minimize} states that Line 6 of \cref{alg:adapt_stepsize} is guaranteed to find a local minimum of \(\mathcal{L}\) after \(\mathcal{C}_{\mathsf{bm}} + \mathcal{C}_{\mathsf{gss}}\) iterations, while
    \begin{align*}
        \Delta 
        &= 
        {\left[\overline{x} - x_0\right]}_+ + {\left[ x_0 - \underline{x} \right]}_+ 
        \\
        &= 
        {\left[\log \overline{h} - \ell^{\prime} \right]}_+ 
        + {\left[ \ell^{\prime} - \log \underline{h} \right]}_+ 
        \\
        &= 
        \log_+ \left( \overline{h}/h_0 \right)
        +
        \log_+ \left( h_0/\underline{h}\right) \; .
    \end{align*}
\end{proofEnd}

%% file: algorithms/alg_klmc_adapt_stepsize.tex
\begin{algorithm2e}[H]
\caption{\(\mathtt{AdaptKLMC}\left(\mathcal{L}, h_{\mathsf{guess}}, \rho_{\mathsf{guess}}, \delta, \Xi, c, r, \epsilon\right)\)}\label{alg:klmc_adapt_stepsize}
\RestyleAlgo{ruled} 
\LinesNumbered
\DontPrintSemicolon
\KwIn{%
    Adaptation objective \(\mathcal{L} : \mathbb{R}_{> 0} \times (0, 1) \to \mathbb{R} \cup \{\infty\}\),\newline
    initial guess \( \left( h_{\mathsf{guess}}, \rho_{\mathsf{guess}} \right) \in \mathbb{R}_{>0} \times (0, 1)\), \newline
    backing-off step size \(\delta < 0\), \newline
    grid of refreshment parameters \(\Xi \in {(0, 1)}^{k}\), \newline
    exponential search coefficient \(c > 0\), \newline
    exponential search exponent \(r > 1\), \newline
    absolute tolerance \(\epsilon > 0\).
}
\KwOut{ Adapted step size and refreshment rate \((h, \rho)\).}
\( \mathcal{L}^{\mathsf{log}}\left(\ell, \rho\right) \triangleq \mathcal{L}\left(\exp\left(\ell\right), \rho\right) \) \;
\(\ell \leftarrow \log h_{\mathsf{guess}}, \;\; \rho \leftarrow \rho_{\mathsf{guess}} \)\;
\If{\(t = 1\)}{
\(\ell \leftarrow \mathtt{FindFeasible}\left( \ell \mapsto \mathcal{L}^{\mathrm{log}}\left(\ell, \rho\right), \ell, \delta \right)\) \;
}
\While{ not converged }{
    {\( \ell^{\prime} \leftarrow \mathtt{Minimize}\left( \ell \mapsto \mathcal{L}^{\mathsf{log}}\left(\ell, \rho\right), \, \ell, c \, ,r \, , \epsilon \right) \)}
    \( \rho^{\prime} \leftarrow \argmin_{\rho \in \Xi} \mathcal{L}^{\mathsf{log}}\left(\ell^{\prime}, \rho\right) \). \;
    \If{ \(\max\left( \abs{\ell - \ell^{\prime}}, \abs{\rho - \rho^{\prime}}\right) \leq \epsilon \) } {
        Return \( (\exp\left(\ell^{\prime}\right), \rho^{\prime}) \)  \;
    }
\( \ell \leftarrow \ell^{\prime}, \;\; \rho \leftarrow \rho^{\prime} \)\;
}
Return \( (\exp\left(\ell^{\prime}\right), \rho^{\prime}) \)  \;
\end{algorithm2e}

%% file: section_experiments.tex
\vspace{-1ex}
\section{Experiments}\label{section:experiments}
\subsection{Implementation and General Setup}
We implemented our SMC sampler\footnote{
    Link to \textsc{GitHub} repository: \url{https://github.com/Red-Portal/ControlledSMC.jl/tree/v0.0.4}.
}
using the Julia language~\citep{bezanson_julia_2017}.
For resampling, we use the Srinivasan sampling process (SSP) by \citet{gerber_negative_2019}, which performs similarly to the popular systematic resampling strategy~\citep{carpenter_improved_1999,kitagawa_monte_1996}, while having stronger theoretical guarantees.
Resampling is triggered adaptively, which has been theoretically shown to work well~\citep{syed_optimised_2024} under the typical rule of resampling as soon as the effective sample size~\citep{kong_note_1992,elvira_rethinking_2022} goes below \(\nicefrac{N}{2}\).
In all cases, the reference distribution is a standard Gaussian \(q = \mathcal{N}\left(0_d, \mathrm{I}_d\right)\), while we use a quadratic annealing schedule \(\lambda_t = {\left(t/T\right)}^2 \).

\vspace{-2ex}
\paragraph{Evaluation Metric.}
We will compare the estimate $\log \widehat{\rvZ}_{T, N} $, where, for unbiased estimates of \(Z\) against a ground truth estimate obtained by running a large budget run with \(N = 2^{14}\) and \(T = 2^9\).
Due to adaptivity, our method only yields \textit{biased} estimates of \(Z\). 
Therefore, after adaptation, we run vanilla SMC with the tuned parameters, which yields unbiased estimates.

\vspace{-2ex}
\paragraph{Benchmark Problems.}
For the benchmarks, we ported some problems from the Inference Gym~\citep{sountsov_inference_2020} to Julia, where the rest of the problems are taken from PosteriorDB~\citep{magnusson_textttposteriordb_2025}.
Details on the problems considered in this work are in~\cref{section:problems}, while the configuration of our adaptive method is specified in~\cref{section:config}.

\begin{figure*}
    \centering
    \vspace{-1ex}
    \subfloat[\(\log Z\) Estimates]{
        \hspace{-.5em}
        \includegraphics[scale=0.95]{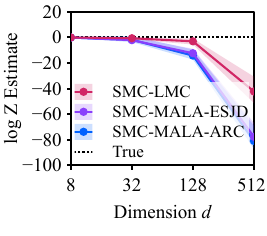}
        \label{fig:metropolis_scaling_logz}
        \vspace{-1.5ex}
    }
    \subfloat[SMC-LMC]{
        \hspace{-1em}
        \includegraphics[scale=0.95]{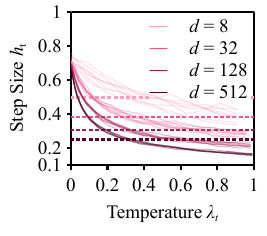}
        \label{fig:metropolis_scaling_smc_lmc}
        \vspace{-1.5ex}
    }
    \subfloat[SMC-MALA-ESJD]{
        \hspace{-1em}
        \includegraphics[scale=0.95]{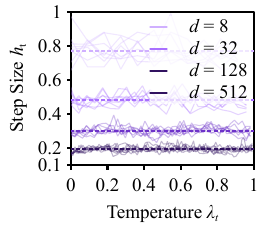}
        \label{fig:metropolis_scaling_smc_mala_esjd}
        \vspace{-1.5ex}
    }
    \subfloat[SMC-MALA-ARC]{
        \hspace{-1em}
        \includegraphics[scale=0.95]{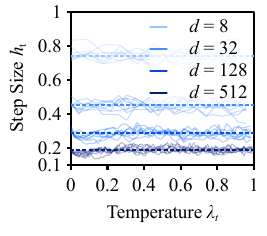}
        \label{fig:metropolis_scaling_smc_mala_arc}
        \vspace{-1.5ex}
    }
    \vspace{-1ex}
    \caption{
        \textbf{Dimensional scaling of adaptive SMC with Langevin-based kernels with (MALA) and without (LMC) MH adjustment.
        } 
        (a) Comparison of the \(\log Z\) estimates under growing dimensionality.
        The solid lines are the median, while the shaded regions are the \(80\%\) quantiles obtained from 32 replications.
        (b-d) Tuned step size schedules obtained under each sampler.
        SMC-MALA-ESJD uses ESJD maximization for adaptation, while SMC-MALA-ARC uses acceptance rate control (ARC).
        Each solid line is a step size schedule obtained from a single run (eight examples are shown), while the dotted lines are the average over $t$.
    }
    \label{fig:metropolis_scaling}
    \vspace{-1ex}
\end{figure*}

\vspace{-1ex}
\subsection{Comparison Against Fixed Step Sizes}
\vspace{-.5ex}
\paragraph{Setup.}
First, we evaluate the quality of the parameters tuned through our method.
For this, we compare the performance of SMC-LMC and SMC-KLMC against hand tuning a fixed step size \(h\), such that \(h_t = h\), over a grid of step sizes.
For KLMC, we also perform a grid search of the refreshment rate over \(\{0.1, 0.5, 0.9\}\).
The computational budgets are set as \(N = 1024\), \(B = 128\), and \(T = 64\).

\vspace{-2ex}
\paragraph{Results.}
A representative subset of the results is shown in \cref{fig:fixed_stepsize_smcula,fig:fixed_stepsize_smcuha}, while the full set of results is shown in \cref{section:additional_fixed_stepsizes}.
First, we can see that SMC with fixed step sizes is strongly affected by tuning.
On the other hand, our adaptive sampler obtains estimates that are closer or comparable to the best fixed step size on 20 out of 22 benchmark problems.
Our method performed poorly on the Rats problem, which is shown in the right-most panes in \cref{fig:fixed_stepsize_smcula,fig:fixed_stepsize_smcuha}.
Overall, our method results in estimates that are better or comparable to those obtained with the best fixed step size.

\begin{figure}[t]
    \vspace{-1ex}
    \subfloat[Sonar]{
        \hspace{-1em}
        \includegraphics[scale=0.9]{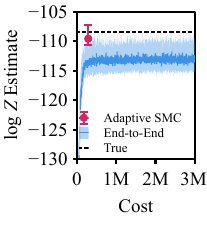}
        \vspace{-2ex}
    }
    \subfloat[Brownian]{
        \hspace{-1em}
        \includegraphics[scale=0.9]{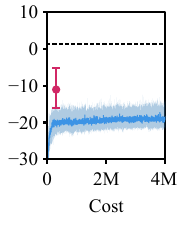}
        \vspace{-2ex}
    }
    \subfloat[Pines]{
        \hspace{-1em}
        \includegraphics[scale=0.9]{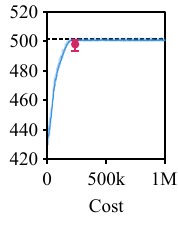}
        \vspace{-2ex}
    }
    \vspace{-1ex}
    \caption{
        \textbf{Comparison against end-to-end optimization.}
        The ``cost'' is the cumulative number of gradient evaluations of the target. 
        32 independent runs for end-to-end optimization are shown.
        The error bars/bands are \(80\%\) empirical quantiles of the cost and the estimates of $\log Z$ computed from 32 replications. 
    }\label{fig:sgd}
    \vspace{-2ex}
\end{figure}

\vspace{-1ex}
\subsection{Comparison Against End-to-End Optimization}
\vspace{-.5ex}
\paragraph{Setup.}
Now, we compare our adaptive tuning strategy against end-to-end optimization strategies.
In particular, we compare against differentiable AIS~\citep{geffner_langevin_2023,geffner_mcmc_2021,zhang_differentiable_2021} instead of SMC, as differentiating through resampling does not necessarily improve the results~\citep{zenn_resampling_2023}.
To focus on the tuning capabilities, we do not optimize the reference \(q\).
However, results with variational reference tuning can be found in~\cref{section:end_to_end_additional}.
Furthermore, we performed a grid search over the SGD step sizes \(\{10^{-4}, 10^{-3}, 10^{-2}\}\) and show the best results.
Additional implementation details can be found in \cref{section:setup_endtoend}.
Since end-to-end methods need to differentiate through the models, we only ran them on the problems with JAX~\citep{bradbury_jax_2018} implementations (Funnel, Sonar, Brownian, and Pine).
We use \(T = 32\) SMC iterations for all methods.
For adaptation, our method uses \(B = 128\) particles out of \(N = 1024\) particles, while end-to-end optimization uses an SMC sampler with \(32\) particles during optimization, and \(N = 1024\) particles when actually estimating \(\widehat{\rvZ}_{T,N}\).
For both methods, the cost of estimating the unbiased normalizing constant is excluded.

\vspace{-2ex}
\paragraph{Results.}
The results with the KLMC kernel are shown in \cref{fig:sgd}, while additional results can be found in \cref{section:end_to_end_additional}.
Our Adaptive SMC sampler achieves more accurate estimates than the best-tuned end-to-end tuning results on Sonar and Brownian, while the estimate on Pines is comparable.
This demonstrates that our SMC tuning approach achieves estimates that are better or on par with those obtained through end-to-end optimization.

\subsection{Dimensional Scaling with and without Metropolis-Hastings Adjustment}
We will now compare the tuned performance of unadjusted versus adjusted kernels, in particular, LMC versus MALA.
To maintain a non-zero acceptance rate, MH-adjusted methods generally require \(h\) to decrease with dimensionality \(d\).
Theoretical results suggest that, for MALA, the step size has to decrease as \(\mathrm{O}\left(d^{-1/3}\right)\)~\citep{chewi_optimal_2021,roberts_optimal_1998} for Gaussian targets and as \(\mathrm{O}\left(d^{-1/2}\right)\) in general~\citep{chewi_optimal_2021,wu_minimax_2022}.
In contrast, LMC only needs to reduce \(h\) to counteract the \textit{asymptotic} bias in the stationary distribution, which grows as \(\mathrm{O}\left(d\right)\) in squared Wasserstein distance~\citep{dalalyan_theoretical_2017,durmus_asymptotic_2024,durmus_highdimensional_2019}.
However, since SMC never operates in the stationary regime (except for the waste-free variant by~\citealt{dau_wastefree_2022}), we expect SMC-LMC to scale better than SMC-MALA with dimensionality \(d\).
Here, we will empirically verify this intuition.

\vspace{-2ex}
\paragraph{Setup.}
We set \(\pi = \mathcal{N}\left(3 \cdot 1_{d}, \mathrm{I}_d\right)\) and \(q = \mathcal{N}\left(0_d, \mathrm{I}_d\right)\) under varying dimensionality \(d\). 
The computational budgets are set as \(N = 1024\), \(T = 4 \lceil \sqrt{d} \rceil \), where the latter is suggested by~\citet[\S 4.7]{syed_optimised_2024}.
For MALA, we will consider two common adaptation strategies: controlling the acceptance rate~\citep{buchholz_adaptive_2021} such that it is \(0.574\)~\citep{roberts_optimal_1998} and maximizing the ESJD~\citep{pasarica_adaptively_2010,buchholz_adaptive_2021,fearnhead_adaptive_2013}.
For both, we use the tricks stated in \cref{section:adaptation}, such as subsampling and SAA.

\vspace{-2ex}
\paragraph{Results.}
The results are shown in \cref{fig:metropolis_scaling}.
For SMC-LMC, the step size schedule is shown to decrease with $t$ (\cref{fig:metropolis_scaling_smc_lmc}).
Since the smoothness constant of the target density does not change with $t$, this means our adaptation scheme is automatically performing a trade-off between convergence speed and asymptotic bias.
Also, the average step sizes decrease with $d$, which is expected since the bias grows with $d$.
However, for $d \geq 128$, the step sizes of SMC-LMC tend to be larger than those of SMC-MALA (\cref{fig:metropolis_scaling_smc_mala_esjd,fig:metropolis_scaling_smc_mala_arc}).
Consequently, SMC-ULA obtains more accurate estimates in higher dimensions (\cref{fig:metropolis_scaling_logz}).



%% file: section_discussions.tex
\vspace{-1ex}
\section{Conclusions}
\vspace{-1ex}
In this work, we established a methodology for tuning path proposal kernels in SMC samplers, which involves greedily minimizing an incremental KL divergence at each SMC step.
We also developed a specific instantiation of the method for tuning scalar-valued step sizes of MCMC kernels used in SMC samplers.
A potential future direction would be to investigate the consistency of the proposed scheme, possibly through the framework of~\citet{beskos_convergence_2016}.

%% file: section_problems.tex
\begin{table*}[t]
    \centering
    \def\arraystretch{1.5}
    \begin{tabular}{lp{0.9\columnwidth}rlc}
        \multicolumn{1}{c}{\textbf{Name}} & \multicolumn{1}{c}{\textbf{Description}} & \(d\) & \multicolumn{1}{c}{\textbf{Source}} & \multicolumn{1}{c}{\textbf{Reference}} 
        \\ \midrule
        Funnel & Neal's funnel distribution. & 10 & Inference Gym & \makecell[t]{\citealt{sountsov_inference_2020}\\\citealt{neal_slice_2003}}
        \\
        Brownian & Latent Brownian motion with missing observations. & 32 & Inference Gym & \makecell[t]{\citealt{sountsov_inference_2020}}
        \\
        Sonar & Bayesian logistic regression with the sonar classification dataset. & 61 & Inference Gym & \makecell[t]{\citealt{sountsov_inference_2020}\\\citealt{gorman_analysis_1988}}
        \\
        Pines & Log-Gaussian Cox process model of the concentration of Scotch pine saplings in Finland over a \(40 \times 40\) grid. & 1600 & Inference Gym & \makecell[t]{\citealt{sountsov_inference_2020}\\\citealt{moller_log_1998}}
    \end{tabular}
    \caption{Overview of Benchmark Problems}
    \label{tab:problems_gym}
\end{table*}

\begin{table*}[t]
    \centering
    \def\arraystretch{1.5}
    \begin{tabular}{lp{0.9\columnwidth}rlc}
        \multicolumn{1}{c}{\textbf{Name}} & \multicolumn{1}{c}{\textbf{Description}} & \(d\) & \multicolumn{1}{c}{\textbf{Source}} & \multicolumn{1}{c}{\textbf{References}} 
        \\ \midrule
        Bones & Latent trait model for multiple ordered categorical responses for quantifying skeletal maturity from radiograph maturity ratings with missing entries. (model: \(\mathtt{bones\_model}\); dataset: \(\mathtt{bones\_data}\)) & 13 & PosteriorDB & \makecell[t]{\citealt{magnusson_textttposteriordb_2025}\\\citealt{spiegelhalter_bugs_1996}}
        \\
        Surgical & Binomial regression model for estimating the mortality rate of pediatric cardiac surgery.  (model: \(\mathtt{surgical\_model}\); dataset: \(\mathtt{surgical\_data}\)) & 14 &  PosteriorDB & \makecell[t]{\citealt{magnusson_textttposteriordb_2025}\\\citealt{spiegelhalter_bugs_1996}}
        \\
        HMM & Hidden Markov model with a Gaussian emission applied to a simulated dataset.
        (model: \(\mathtt{hmm\_gaussian}\); dataset: \(\mathtt{hmm\_gaussian\_simulated}\))
         & 14 & PosteriorDB & \makecell[t]{\citealt{magnusson_textttposteriordb_2025}
         \\\citealt{cappe_inference_2005}}
        \\
        Loss Curves &
        Loss model of insurance claims.
        The model is the single line-of-business, single insurer (SISLOB) variant, where the dataset is the ``ppauto'' line of business, part of the ``Schedule P loss data'' provided by the Casualty Actuarial Society.
        (model: \(\mathtt{losscurve\_sislob}\); dataset: \(\mathtt{loss\_curves}\)) & 15 & PosteriorDB & \makecell[t]{\citealt{magnusson_textttposteriordb_2025}\\\citealt{cooney_modelling_2017}}
        \\
        Pilots & Linear mixed effects model with varying intercepts for estimating the psychological effect of pilots when performing flight simulations on various airports.
        (model: \(\mathtt{pilots}\); dataset: \(\mathtt{pilots}\)) & 18 & PosteriorDB & \makecell[t]{\citealt{magnusson_textttposteriordb_2025}\\\citealt{gelman_data_2007}}
        \\
        Diamonds & 
        Log-log regression model for the price of diamonds with highly correlated predictors.
        (model: \(\mathtt{diamonds}\); dataset: \(\mathtt{diamonds}\)) 
        & 26 & PosteriorDB & \makecell[t]{\citealt{magnusson_textttposteriordb_2025}\\\citealt{wickham_toolbox_2016}}
        \\
        Seeds &  Random effect logistic regression model of the seed germination proportion of seeds from different root extracts.
        We use the variant with a half-Cauchy prior on the scale.
        (model: \(\mathtt{seeds\_stanified\_model}\); dataset: \(\mathtt{seeds\_data}\)) & 26 & PosteriorDB & \makecell[t]{\citealt{magnusson_textttposteriordb_2025}\\\citealt{crowder_betabinomial_1978}\\\citealt{spiegelhalter_bugs_1996}}
        \\
        Downloads & Prophet time series model applied to the download count of \texttt{rstan} over time. The model is an additive combination of
        \begin{enumerate*}[label=(\roman*)]
            \item a trend model,
            \item a model of seasonality, and
            \item a model for events such as holidays.
        \end{enumerate*}
        (model: \(\mathtt{prophet}\); dataset: \(\mathtt{rstan\_downloads}\)) 
        & 62 & PosteriorDB & \makecell[t]{\citealt{magnusson_textttposteriordb_2025}\\\citealt{taylor_forecasting_2018}\\\citealt{bales_selecting_2019}}
        \\
        Rats & Linear mixed effects model with varying slopes and intercepts for modeling the weight of young rats over five weeks. (model: \(\mathtt{rats\_model}\); data: \(\mathtt{rats\_data}\)) & 65 & PosteriorDB & \makecell[t]{\citealt{magnusson_textttposteriordb_2025}\\\citealt{spiegelhalter_bugs_1996}\\\citealt{gelfand_illustration_1990}}
        \\
        Radon & Multilevel mixed effects model with log-normal likelihood and varying intercepts for modeling the radon level measured in U.S. households. We use the Minnesota state subset. (model: \(\mathtt{radon\_hierarchical\_intercept\_centered}\); dataset: \(\mathtt{radon\_mn}\)) & 90 & PosteriorDB & \makecell[t]{\citealt{magnusson_textttposteriordb_2025}\\\citealt{gelman_bayesian_2014}}
    \end{tabular}
    \caption{Overview of Benchmark Problems}
    \label{tab:problems_posteriordb1}
\end{table*}

\begin{table*}[t]
    \centering
    \def\arraystretch{1.5}
    \begin{tabular}{lp{0.9\columnwidth}rlc}
        \multicolumn{1}{c}{\textbf{Name}} & \multicolumn{1}{c}{\textbf{Description}} & \(d\) & \multicolumn{1}{c}{\textbf{Source}} & \multicolumn{1}{c}{\textbf{Reference}} 
        \\ \midrule
        Election88 & Generalized linear mixed effects model of the voting outcome of individuals at the 1988 U.S. presidential election. (model: \(\mathtt{election88\_full}\); dataset: \(\mathtt{election88}\)) & 90 & PosteriorDB & \makecell[t]{\citealt{magnusson_textttposteriordb_2025}\\\citealt{gelman_data_2007}}
        \\
        Butterfly & Multispecies occupancy model with correlation between sites. The dataset contains counts of butterflies from twenty grassland sites in south-central Sweden (model: \(\mathtt{butterfly}\); dataset: \(\mathtt{multi\_occupancy}\)) & 106 & PosteriorDB & \makecell[t]{\citealt{magnusson_textttposteriordb_2025}\\\citealt{dorazio_estimating_2006}}
        \\
        Birds & Mixed effects model with a Poisson likelihood and varying intercepts for modeling the occupancy of the Coal tit (\textit{Parus ater}) bird species during the breeding season in Switzerland. (model: \(\mathtt{GLMM1\_model}\); dataset: \(\mathtt{GLMM\_data}\)) & 237 & PosteriorDB & \makecell[t]{\citealt{magnusson_textttposteriordb_2025}\\\citealt{kery_bayesian_2012}}
        \\
        Drivers & 
        Time series model with seasonal effects of driving-related fatalities and serious injuries in the U.K. from Jan. 1969 to Dec. 1984. 
        (model: \(\mathtt{state\_space\_stochastic\_level\_stochastic\_seasonal}\); dataset: \(\mathtt{uk\_drivers}\)) & 389 & PosteriorDB & \makecell[t]{\citealt{magnusson_textttposteriordb_2025}\\\citealt{commandeur_introduction_2007}}
        \\
        Capture &
        Model of capture-recapture data for estimating the population size.
        This is the ``heterogeneity model,'' where the detection probability is assumed to be heterogeneous across the individuals.
        The data is simulated.
        (model: \(\mathtt{Mh\_model}\); dataset: \(\mathtt{Mh\_data}\)) & 388 & PosteriorDB & \makecell[t]{\citealt{magnusson_textttposteriordb_2025}\\\citealt{kery_bayesian_2012}}
        \\
        Science & 
        Item response model with generalized rating scale.
        The dataset was taken from the Consumer Protection and Perceptions of Science and Technology section of the 1992 Euro-Barometer Survey.
        (model: \(\mathtt{grsm\_latent\_reg\_irt}\); dataset: \(\mathtt{science\_irt}\)) & 408 & PosteriorDB & \makecell[t]{\citealt{magnusson_textttposteriordb_2025}\\\citealt{reif_eurobarometer_1993}\\\citealt{furr_rating_2017}}
        \\
        Three Men 
        &
        Latent Dirichlet allocation for topic modeling.
        The number of topics is set as \(K = 2\), while the dataset is corpus 3 among pre-processed multilingual corpora of the book ``Three Men and a Boat.''
        (model: \(\mathtt{ldaK2}\); dataset: \(\mathtt{three\_men3}\)) & 505 & PosteriorDB & \makecell[t]{\citealt{magnusson_textttposteriordb_2025}\\\citealt{farkas_three_2014}\\\citealt{blei_latent_2003}}
        \\
        TIMSS & 
        Item response model with generalized partial credit.
        The dataset is from the TIMSS 2011 mathematics assessment of Australian and Taiwanese students.
        (model: \(\mathtt{gpcm\_latent\_reg\_irt}\); dataset: \(\mathtt{timssAusTwn\_irt}\)) & 530 & PosteriorDB & \makecell[t]{\citealt{magnusson_textttposteriordb_2025}\\\citealt{muraki_generalized_1997}\\\citealt{mullis_timss_2012}}
    \end{tabular}
    \caption{Overview of Benchmark Problems (continued)}
    \label{tab:problems_posteriordb2}
\end{table*}

In this section, we provide additional details about the benchmark problems.
A full list of the problems is shown in \cref{tab:problems_gym,tab:problems_posteriordb1,tab:problems_posteriordb2}.
For the problems we ported from the Inference Gym, we provide additional details for clarity:

\paragraph{Funnel.}
This is the classic benchmark problem by \citet{neal_slice_2003}.
We use the formulation:
\begin{align*}
    y &\sim \mathcal{N}\left(0, 3^2\right) \\
    x &\sim \mathcal{N}\left(0_{d-1},  e^{y} \mathrm{I}_{d - 1} \right)  \, ,
\end{align*}
where \(d = 10\).

\paragraph{Sonar.}
This is a logistic regression problem with a standard normal prior on the coefficients.
That is, for $d = 61$, given a dataset \((X, y)\), where \(X \in \mathbb{R}^{n, d -1}\) and \(y \in \mathbb{R}^n\), the design matrix is augmented with a column containing 1s denoted with $\tilde{X}$ to include an intercept.
The data-generating process is 
\begin{align*}
    \beta &\sim \mathcal{N}\left(0_{d}, \mathrm{I}_d\right)  \\
    y &\sim \mathrm{Bernoulli}\big( \sigma\big( \tilde{X} \beta\big) \big) \, ,
\end{align*}
where \(\sigma\left(x\right) \triangleq 1 / \left(1 + \mathrm{e}^{-x}\right)\) is the logistic function.
Here, we use the sonar classification dataset by \citet{gorman_analysis_1988}.
The features are pre-processed with \(z\)-standardization following~\citet{phillips_particle_2024}.

\newpage
\paragraph{Pines.}
This is a log-Gaussian Cox process (LGCP;~\citealp{moller_log_1998}) model applied to a dataset of Scotch pine saplings in Finland~\citep{moller_log_1998}.
A LGCP is a nonparametric model of intensity fields, where the observations are assumed to follow a Poisson point process (PPP).
Consider a 2-dimensional grid of \(n\) cells indexed by \(i \in [n]\), each denoted by \(S_i \in \mathcal{S}\) and centered on the location \(x_i \in \mathbb{R}^d\).
The dataset is the number of points contained in the \(i\)th cell, \(y_i \in \mathbb{N}_{\geq 0}\), for all \(i \in [n]\), which is assumed to follow a PPP such that
\begin{align*}
    y_i \sim \mathrm{Poisson}\left(\int_{S_i} \lambda\left(x\right) \mathrm{d}x\right)
\end{align*}
with the intensity field 
\[
    \log \lambda \sim \mathcal{GP}\left(\mu, k\right) \, ,
\]
where \(\mathcal{GP}\left(\mu, k\right)\) is a Gaussian process prior (GP;~\citealp{rasmussen_gaussian_2005}) with mean \(\mu\) and covariance kernel \(k : \mathbb{R}^2 \times \mathbb{R}^2 \to \mathbb{R}_{>0}\).
We use the grid approximation 
\begin{align*}
    \int_{S_i} \lambda\left(x\right) \mathrm{d}x \approx A_i \exp\left( \log\lambda\left(x_i\right) \right) \, ,
\end{align*}
where \(A_i\) is the area of \(S_i\).
The likelihood is then
\[
    \ell\left(y_i, x_i, \lambda\right) = \exp\left\{ \lambda\left(\vx_i\right) y_i - A_i \exp\left(\lambda\left(\vx_i\right)\right) \right\} \, .
\]

Following~\citet{moller_log_1998}, the hyperparameters of the GP are set as
\begin{align*}
    \mu &= \log\left(126\right) - \frac{\sigma^2}{2}  \\
    k\left(\vx_i, \vx_j\right) &= \sigma^2 \exp\left(-\frac{ \norm{\vx_i - \vx_j }_2 }{ \sqrt{ \abs{\mathcal{S}} \beta^2 } } \right) \, ,
\end{align*}
where
\begin{align*}
    \sigma^2 = 1.91
    \quad\text{and}\quad
    \beta = \frac{1}{33} \, .
\end{align*}
The field \([0, 1]^2\) is discretized into a \(40 \times 40\) grid such that \(\abs{\mathcal{S}} = 40^2\) and \(A_i = 1/\abs{\mathcal{S}}\).
Furthermore, to improve the conditioning of the posterior, we whiten the GP prior~\citep[\S 2.1]{murray_slice_2010}.

%% file: section_config.tex
\subsection{Setup of Adaptive SMC Samplers}\label{section:setup_smc}

\paragraph{Configuration of the Adaptation Procedure.}
Here, we collected the specifications of the tunable parameters in our adaptive SMC samplers.
The parameters of SMC-LMC are set as in \cref{table:smclmc_params}:

\begin{table}[h]
    \centering
    \begin{tabular}{ccc}
        \multicolumn{1}{c}{\textbf{Name}} & \multicolumn{1}{c}{\textbf{Source}} & \multicolumn{1}{c}{\textbf{Value}}
        \\ \midrule
        \(\tau\) & \cref{section:adaptation} & 0.1 \\
        \(\epsilon\) & \cref{alg:gss}         & 0.01 \\
        \(c\)      & \cref{alg:find_minimum}  & 0.1 \\
        \(r\)      & \cref{alg:find_minimum}  & 2 \\
        \(\delta\) & \cref{alg:find_feasible} & \(-1\) \\
        \(h_{\mathsf{guess}}\) & \cref{alg:find_feasible} & \(\exp\left(-10\right) \approx 4.54 \times 10^{-5}\) 
    \end{tabular}
    \caption{Configuration of SMC-LMC}\label{table:smclmc_params}
\end{table}

The parameters of SMC-KLMC are set as in \cref{table:smcklmc_params}:

\begin{table}[h]
    \centering
    \begin{tabular}{ccc}
        \multicolumn{1}{c}{\textbf{Name}} & \multicolumn{1}{c}{\textbf{Source}} & \multicolumn{1}{c}{\textbf{Value}}
        \\ \midrule
        \(\tau\)    & \cref{section:adaptation} & 5 \\
        \(\epsilon\) & \cref{alg:gss}         & \(0.01\) \\
        \(c\)      & \cref{alg:find_minimum}  & \(0.01\)  \\
        \(r\)      & \cref{alg:find_minimum}  & 3 \\
        \(\delta\) & \cref{alg:find_feasible} & \(-1\) \\
        \(\Xi\) & \cref{alg:klmc_adapt_stepsize} & \(\{0.1, 0.9\}\) \\
        \(\rho_{\mathsf{guess}}\) & \cref{alg:klmc_adapt_stepsize} & \(0.1\) \\
        \(h_{\mathsf{guess}}\) & \cref{alg:find_feasible} & \(\exp\left(-7.5\right) \approx 5.53 \times 10^{-4}\) 
    \end{tabular}
    \caption{Configuration of SMC-KLMC}\label{table:smcklmc_params}
\end{table}

\paragraph{Schedule Adaptation.}
In some of the experimental results in the appendix, we evaluate the performance of our step size adaptation procedure when combined with an annealing temperature schedule (\({(\lambda_t)}_{t = 0, \ldots, T}\)) adaptation scheme.
In particular, we use the recently proposed method of \citet{syed_optimised_2024}, which is able to tune both the schedule \({(\lambda_t)}_{t = 0, \ldots, T}\) and the number of SMC steps \(T\).
Under regularity assumptions, the resulting adaptation schedule asymptotically ($N\to \infty$ and $T \to \infty$) approximates the optimal geometric annealing path that minimizes the variance of the normalizing constant estimator. 
For a detailed description, see \citet[Sec. 5]{syed_optimised_2024}. Below, we provide a concise description of the schedule adaptation process.

Under suitable regularity assumptions, the asymptotically optimal schedule is the one that makes the ``local communication barrier''
\[
    \text{LCB}(\lambda_{t-1}, \lambda_t) \approx \sqrt{\mathrm{R}(\pi_{t-1} \otimes K_{t}^{\vtheta} || \pi_{t} \otimes L_{t-1}^{\vtheta})}
\]
uniform across all adjacent steps \(\lambda_t, \lambda_{t-1}\) for all \(t \in [T]\)~\citep[\S 4.3]{syed_optimised_2024}.
As such, the corresponding adaptation scheme estimates the local communication barrier and uses it to obtain a temperature schedule that makes it uniform.
Intuitively, $\text{LCB}(\lambda_{t-1}, \lambda_t)$ quantifies the ``difficulty'' of approximating $\pi_{t} \otimes L_{t-1}^{\vtheta} $ using weighted particles drawn from $\pi_{t-1} \otimes K_{t}^{\vtheta}$. 

In addition, let us denote the local communication barrier accumulated up to time step 
\(t \in \{0, \ldots, T\}\), 
\[
    \Lambda(\lambda_t) \triangleq \sum_{s = 1}^t \text{LCB}(\lambda_{s-1}, \lambda_s) \; .
\]
This serves as a divergence measure for the ``length'' of the annealing path from $\lambda_0 = 0$ to $\lambda_t$.
Furthermore, the \textit{total} accumulated local barrier
\[
    \Lambda \triangleq \Lambda(\lambda_T) \; ,
\]
which is referred to as the \emph{global communication barrier}, quantifies the total difficulty of simulating the annealing path ${(\pi_t)}_{t \in \{0, \ldots, T\}}$. 
For the normalizing constant to be accurate, SMC needs to operate in what they call the ``stable discretization regime,'' which occurs at \(T = \mathrm{O}\left(\Lambda\right)\)~\citep{syed_optimised_2024}.
Therefore, for tuning the number of SMC steps, a good heuristic is to set \(T\) to be a constant multiple of the estimated global communication barrier.

The corresponding schedule adaptation scheme is as follows: From the estimates of the communication barrier ${ (\widehat{\Lambda}\left(\lambda_t\right)) }_{t \in [T]}$ obtained from a previous run, the updated schedule for the \textit{next} run of length \(T^{\prime}\) is set via mapping
\begin{align}
\lambda_t^\star = \widehat\Lambda_\text{inv}\left(\widehat\Lambda \times \frac{t}{T'} \right) , \quad t' = 0, \dots, T',
\label{eq:schedule_adaptation}
\end{align}
where the inverse mapping $\widehat\Lambda_\text{inv}$ is approximated using a monotonic spline with knots $\{(\widehat \Lambda(\lambda_t), \lambda_t)\}_{t = 0}^T$.
In our case, the length of the new schedule is set as $T' = 2 \widehat \Lambda$. 

Below, we summarize the general steps for adaptive SMC with round-based schedule adaptation:

\begin{algorithm2e}
\caption{Round-Based Annealing Schedule Adaptation}
\RestyleAlgo{ruled} 
\LinesNumbered
\DontPrintSemicolon
\KwIn{%
    Number of rounds $r_{\mathrm{max}}$,\newline
    Initial number of SMC iterations $T_1$.\newline
    Initial schedule ${(\lambda^{1}_t)}_{t = 0, \ldots, T_1}$.
}
\For{$r = 1 \ldots, r_{\mathrm{max}}$}{
Run adaptive SMC with the schedule ${(\lambda^{r}_t)}_{t = 0, \ldots, T_r}$.\;
Estimate $\widehat{\Lambda}$ and compute $\widehat{\Lambda}_{\mathrm{inv}}$.\;
Set $T_{r+1} = 2 \widehat{\Lambda}$.\;
Obtain ${(\lambda^{r+1}_t)}_{t = 0, \ldots, T_{r+1}}$ using \cref{eq:schedule_adaptation}.\;
}
\end{algorithm2e}

\newpage
\subsection{Setup of End-to-End Optimization}\label{section:setup_endtoend}

We provide additional implementation details for end-to-end optimization\footnote{Link to \textsc{GitHub} repository: \url{https://github.com/zuhengxu/dais-py/releases/tag/v1.1}}. We implemented two differentiable AIS methods: one based on LMC~\citep{thin_monte_2021} and another based on KLMC~\citep{geffner_mcmc_2021,doucet_scorebased_2022}. 
Both methods are implemented in JAX \citep{bradbury_jax_2018}, modified from the code provided by \citet{geffner_langevin_2023}. 

For optimization, we used the Adam optimizer \citep{kingma_adam_2015} with three different learning rates $\{10^{-4}, 10^{-3}, 10^{-2}\}$ for $5{,}000$ iterations, with a batch size of $32$. We evaluated two different annealing step sizes ($32$ and $64$), keeping the number of steps fixed during training while optimizing the annealing schedule (detailed in \cref{section:end_to_end_additional}). Each setting was repeated $32$ times, and we report results from the best-performing configurations.


Following the setup of \citet{doucet_scorebased_2022},
the vector-valued step sizes are amortized through a function $\epsilon_\theta(t): [0, 1] \to \mathbb{R}^{d}$.
This function is parametrized through a 2-layer fully connected neural network with 32 hidden units and ReLU activation, followed by a scaled sigmoid function which enforces $\epsilon_\theta(t) < 0.1$ for the ULA variant and $\epsilon_\theta(t) < 0.25$ for the KLMC variant. 
Enforcing these step size constraints is necessary to prevent numerical issues during training, which was acknowledged in prior works \citep{doucet_scorebased_2022,geffner_mcmc_2021}.

For schedule adaptation,  \citet{doucet_scorebased_2022,geffner_mcmc_2021}, parametrize the temperature schedule as 
\[
\lambda_t=\frac{\sum_{t^{\prime} \leq t} \sigma\left(b_{t^{\prime}}\right)}{\sum_{t^{\prime}=1}^T \sigma\left(b_{t^{\prime}}\right)} \; ,
\]
where $\sigma$ is the sigmoid function, $\lambda_0$ is fixed to be 0, and $b_{0:T-1}$ is subject to optimization.
Following \citet{doucet_scorebased_2022}, we additionally learn the momentum refreshment rate $\rho$ (shared across $t\in [T]$) for SMC-KLMC.
That is, we parametrize $\rho$ with a parameter $u$ as $\rho=.98 \sigma(u)+.01$, which ensures $\rho \in (0.01, 0.99)$, 




%% file: section_algorithms.tex
In this section, we will provide a detailed description of our proposed adaptation algorithms and their components.

\subsection{\(\mathtt{FindFeasible}\) (\Cref{alg:find_feasible})}

\begin{figure}[!t]
    \removelatexerror
    \input{algorithms/alg_find_feasible}
    \vspace{-2ex}
\end{figure}

Our adaptation schemes receive a guess from the user.
For robustness, it is safe to assume that this guess may not be non-degenerate.
As such, we must first check that it is non-degenerate, and if it is not, move it to somewhere that is.
This is done by \(\mathtt{FindFeasible}\left(f, x_0, \delta\right)\) shown in \cref{alg:find_feasible}.
If \(x_0\) is already non-degenerate, it immediately returns the initial point \(x_0\).
Otherwise, if \(x_0\) is degenerate, it increases or decreases \(x_0\) with a step size of \(\delta\) until the objective function becomes finite.


\subsection{\(\mathtt{GoldenSectionSearch}\) (\cref{alg:gss})}\label{section:gss}

\begin{figure}[!t]
    \removelatexerror
    \input{algorithms/alg_gss}
    \vspace{-4ex}
\end{figure}

The workhorse of our step size adaptation scheme is the golden section search (GSS) algorithm~\citep{avriel_golden_1968}, which is a variation of the Fibonacci search algorithm~\citep{kiefer_sequential_1953}.
In particular, we are using the implementation of~\citet[\S 10.1]{press_numerical_1992}, shown in \cref{alg:gss}, which uses a \textit{triplet}, \((a, b, c) \in \mathbb{R}^3\), for initialization.
This triplet requires the condition 
{%
\begin{equation}
    a < b < c , \quad
    f\left(b\right) < f\left(a\right), 
    \;\text{and}\;
    f\left(b\right) < f\left(c\right)
    \label{eq:triplet_condition}
\end{equation}
}%
to hold.
Then, by \cref{thm:minima}, this implies that the open interval \((a, c)\) contains a local minimum.
Then GSS is guaranteed to find the contained local minimum at a ``linear rate'' of \(\nicefrac{\left(1 - \sqrt{5}\right)}{2} \approx 1.62\), the golden ratio.
For finding a point $\epsilon$-close to a local minimum, this translates into an objective query complexity of \(\mathcal{O}\left(\log \abs{c - a}/\epsilon\right)\).
Furthermore, if \(f\) is unimodal (\cref{assumption:unimodal}), then the solution will be \(\epsilon\)-close to the global minimum~\citep[\S 7.1]{luenberger_linear_2008}.
The key is to find a triplet \((a, b, c)\) satisfying the condition in \cref{eq:triplet_condition}, which is done in \cref{alg:find_minimum} in the next section.

\clearpage
\begin{figure}[!t]
    \removelatexerror
    \input{algorithms/alg_find_minimum}
\end{figure}

\newpage
\subsection{\(\mathtt{BracketMinimum}\) (\Cref{alg:find_minimum})}\label{section:bracket_minimum}
The main difficulty of applying GSS in practice is setting the initial bracketing interval.
If the bracketing interval does not contain a local minimum, nothing can be said about what GSS is converging towards.
In our case, we require a triplet \((a, b, c)\) that satisfies the sufficient conditions in \cref{thm:gss}.
Therefore, an algorithm for finding such an interval is necessary.
Naturally, this algorithm should have a computational cost that is better or at least comparable to GSS.
Otherwise, a more naive way of setting the intervals would make more sense.
Furthermore, the width of the interval found by the algorithm should be as narrow as possible so that GSS can be run more efficiently.

While \citet[\S 10.1]{press_numerical_1992} presents an algorithm for finding such a bracket using parabolic interpolation, the efficiency and quality of the output of this algorithm are not analyzed. 
Furthermore, the presence of discontinuities in our objective function (\cref{assumption:general}) warrants a simpler algorithm that is provably robust.
Therefore, we use a specialized routine, \(\mathtt{BracketMinimum}\), shown in \cref{alg:find_minimum}.

\(\mathtt{BracketMinimum}\) works in two stages: Given an initial point \(x_0\), it expands the search interval to the right (towards \(+\infty\); Line 4-15) and then to the left (towards \(-\infty\); Line 17-28).
During this, it generates a sequence of exponentially increasing intervals (Lines 5 and 18) and, in the second stage, stops when it detects points that satisfy the condition in \Cref{eq:triplet_condition}.
This algorithm was inspired by a Stack Exchange post by \citet{lavrov_answer_2017}, which was in turn inspired by the exponential search algorithm~\citep{bentley_almost_1976}.

Most tunable parameters of the adaptation method in \cref{section:stepsize} come from \(\mathtt{BracketMinimum}\).
In fact, the parameters of \(\mathtt{BracketMinimum}\) most crucially affect the overall computational performance of our schemes.
Recall that the convergence speed of GSS depends on the width of the provided triplet, \(\abs{a - c}\).
Given, this \(c\) and \(r\) affect performance as follows:
\begin{enumerate*}
    \item[(a)] The width of the resulting triplet, \(\abs{a - c}\), increases with \(r\) and \(c\).
    \item[(b)] Smaller \(r\) and \(c\) requires more time to find a valid triplet.
\end{enumerate*}
For a discussion on how to set these parameters, see \cref{section:theory}.


\subsection{\(\mathtt{Minimize}\) (\Cref{alg:minimize})}\label{section:minimize}

\begin{figure}[!t]
    \removelatexerror
    \input{algorithms/alg_minimize}
\end{figure}

We finally discuss our complete optimization routine, which is shown in \cref{alg:minimize}.
Given an initial point \(x_0\) and suitable assumptions, \(\mathtt{Minimize}\left(f, x_0, c, r, \epsilon\right)\) finds a point that is \(\epsilon\)-close to a local minimum.
This is done by first finding an interval that contains the minimum (Line 1) by calling \(\mathtt{BracketMinimum}\), which is then used by \(\mathtt{GoldenSectionSearch}\) for proper optimization (Line 2).
As such, the computation cost of the routine is the sum of the two stages.

There are four parameters: \(x_0, c, r, \epsilon\).
Admissible values of \(\epsilon\) will depend on the requirements of the downstream task.
On the other hand, \(c\) and \(r\) can be optimized.
The effect of these parameters on the execution time is analyzed in~\cref{thm:minimize}, while a discussion on how to interpret the theoretical analysis is in \cref{section:theory}.

%% file: algorithms/alg_find_feasible.tex
\begin{algorithm2e}[H]
\caption{\(\mathtt{FindFeasible}\left(f, x_0, \Delta\right)\)}\label{alg:find_feasible}
\RestyleAlgo{ruled} 
\LinesNumbered
\DontPrintSemicolon
\KwIn{%
    Objective \(f : \mathbb{R} \to \mathbb{R}\),\newline
    initial guess  \(x_0 \in \mathbb{R} \), \newline
    backing off step size \(\delta \in \mathbb{R} \setminus \{0\}\).
}
\KwOut{ Feasible initial point \(x_0\) }
\(x \leftarrow x_0\) \;
\While{ \( f \left( x \right) = \infty \)}{ 
   \(x \leftarrow x + \delta \)\;
}
Return \(x\)
\end{algorithm2e}

%% file: algorithms/alg_gss.tex
\begin{algorithm2e}[H]
\caption{\(\mathtt{GoldenSectionSearch}\left(f, a, b, c, \epsilon\right)\)}\label{alg:gss}
\RestyleAlgo{ruled} 
\LinesNumbered
\DontPrintSemicolon
\KwIn{%
Objective \(f : \mathbb{R} \to \mathbb{R} \cup \{+\infty\}\),\newline
initial triplet \((a, b, c) \in \mathbb{R}^3\) satisfying \cref{eq:triplet_condition},\newline
absolute tolerance \(\epsilon > 0\).
}
\(\phi^{-1} \triangleq \nicefrac{ \left( \sqrt{5} - 1 \right) }{2}\)\;
\(x_0 \leftarrow a\)\;
\(x_3 \leftarrow c\)\;
\eIf{ \(\abs{c - b} > \abs{b - a}\) }{
    \(x_1 \leftarrow b\)\;
    \(x_2 \leftarrow b + \left(1 - \phi^{-1}\right) \left(c - b\right)\)\;
}{
    \(x_2 \leftarrow b\)\;
    \(x_1 \leftarrow b - \left(1 - \phi^{-1}\right) \left(b - a\right)\)\;
}
\(f_1 \leftarrow f\left(x_1\right)\)\;
\(f_2 \leftarrow f\left(x_2\right)\)\;
\While{ \(\abs{x_1 - x_2} > \epsilon/2\) }{
    \eIf{ \(f_2 < f_1 \) }{ 
        \(x_0 \leftarrow x_1\)\;
        \(x_1 \leftarrow x_2\)\;     
        \(x_2 \leftarrow \phi^{-1} x_2 +  \left(1 - \phi^{-1}\right) x_3 \)\;
        \(f_1 \leftarrow f_2\)\;
        \(f_2 \leftarrow f\left(x_2\right)\)\;
    }{
        \(x_3 \leftarrow x_2\)\;
        \(x_2 \leftarrow x_1\)\;
        \(x_1 \leftarrow \phi^{-1} x_1 +  \left(1 - \phi^{-1}\right) x_0 \)\;
        \(f_2 \leftarrow f_1\)\;
        \(f_1 \leftarrow f\left(x_1\right)\)\;
    }
}
\eIf{\(f_1 \leq f_2\)}{
    Return \(x_1\)\;
}{
    Return \(x_2\)\;
}
\end{algorithm2e}

%% file: algorithms/alg_find_minimum.tex
\begin{algorithm2e}[H]
\caption{\(\mathtt{BracketMinimum}\left(f, x_0, c, r\right)\)}\label{alg:find_minimum}
\RestyleAlgo{ruled} 
\LinesNumbered
\DontPrintSemicolon
\KwIn{%
    Objective \(f : \mathbb{R} \to \mathbb{R} \cup \{+\infty\}\),\newline
    initial point \(x_0 \in (-\infty, x^{\infty})\), \newline
    exponential search coefficient \(c > 0\), \newline
    exponential search base \(r > 1\).
}
\KwOut{Triplet \((x^-, x_{\mathsf{mid}}, x^+)\)}
\(x \leftarrow x_0 \) \;
\(y \leftarrow f(x) \)\;
\(k \leftarrow 0\)\;
\While{ true }{ 
    \( x^{\prime} \leftarrow x_0 + c r^k \)\;
    \(y^{\prime} \leftarrow f\left(x^{\prime}\right) \) \;
    \If{ \(y < y^{\prime}\)}{
        \(x^{+} \leftarrow x^{\prime}\)\;
        \(x_0 \gets x\) \;
        break \;
    }
    \(x \leftarrow x^{\prime} \)\;
    \(y \leftarrow y^{\prime}\)\;
    \(k \leftarrow k + 1\)\;
}
\(k \leftarrow 0\)\;
\While{ true  }{ 
    \(x^{\prime} \leftarrow x_0 - c r^k \) \;
    \(y^{\prime} \leftarrow f\left(x^{\prime}\right) \) \;
    \If{ \( y < y^{\prime} \)}{
        \(x^- \gets x^{\prime}\) \;
        \(x_{\mathsf{mid}} \gets x\) \;
        break\;
    }
    \(x \leftarrow x^{\prime} \)\;
    \(y \leftarrow y^{\prime}\)\; 
    \(k \leftarrow k + 1\)\;
}
Return \((x^-, x_{\mathsf{mid}}, x^+)\)
\end{algorithm2e}

%% file: algorithms/alg_minimize.tex
\begin{algorithm2e}[H]
\caption{\(\mathtt{Minimize}\left( f, x_0, c, r, \epsilon\right)\)}\label{alg:minimize}
\RestyleAlgo{ruled} 
\LinesNumbered
\DontPrintSemicolon
\KwIn{%
    objective \(f : \mathbb{R} \to \mathbb{R} \cup \{+\infty\} \), \newline
    initial point \( x_0 \in \mathbb{R} \) such that \(f\left(x_0\right) < \infty\), \newline
    exponential search coefficient \(c > 0\), \newline
    exponential search exponent \(r > 1\), \newline
    absolute tolerance \(\epsilon > 0\).
}
\( (x^-, x_{\mathsf{int}}, x^+) \leftarrow \mathtt{BracketMinimum}\left( f, x_0, c, r \right) \) \;
\(x^* \leftarrow \mathtt{GoldenSectionSearch}\left( f, \, x^-,\, x_{\mathsf{int}}, \, x^+, \epsilon \right) \)
\end{algorithm2e}

%% file: section_analysis.tex
In this section, we will provide a formal theoretical analysis of the algorithms presented in \cref{section:algorithms} as well as the omitted proof of the theorems in the main text.

\subsection{Definitions and Assumptions}

Formally, when we say ``local minimum,'' we follow the following definition:
\vspace{2ex}
\begin{definition}[Definition 7; \citealp{rudin_principles_1976}]
    Consider some continuous function \(f : \mathcal{X} \to \mathbb{R} \) on a metric space \(\mathcal{X} \subseteq \mathbb{R}\). 
    We say \(f\) has a local minimum at \(x^* \in \mathcal{X}\) if there exists some \(\epsilon > 0\) such that \(f\left(x^*\right) \leq f\left(x\right)\) for all \(x \in \mathcal{X}\) with \(\abs{x^{*} - x} < \epsilon\) .
\end{definition}

Also, unimodal functions are defined as follows:
\vspace{2ex}
\begin{definition}
    We say \(f : [a, b] \to \mathbb{R}\) is unimodal if there exists some point \(x^*\) such that \(f\) is monotonically strictly decreasing on \([a, x^*]\) and strictly increasing on \([x^*, b]\).
\end{definition}

Now, recall that our adaptation objective in \cref{section:adaptation} operates on \(\mathbb{R}_{>0}\).
During adaptation, however, the objectives are log-transformed so that optimization is performed on \(\mathbb{R}\).
Therefore, it is convenient to assume everything happens on \(\mathbb{R}\).
That is, instead of \cref{assumption:initialization,assumption:unimodal}, we will work with the following assumptions that are equivalent up to log transformation:
\vspace{1ex}
\begin{assumption}\label{assumption:general}
    For the objective \(f : \mathbb{R} \to \mathbb{R} \cup \{+\infty\}\), we assume the following:
    \vspace{-2ex}
    \begin{enumerate}[itemsep=0ex]
        \item[(a)] There exists some \(x^{\infty} \in (-\infty, \infty]\) such that \(x\) is finite and continuous on \((-\infty, x^{\infty})\) and \(\infty\) on \([x^{\infty}, \infty)\).
        
        \item[(b)] There exists some \(\underline{x} \in (-\infty, x^{\infty})\) such that \(f\) is strictly monotonically decreasing on \((-\infty, \underline{x}]\).
        
        \item[(c)] There exists some \(\overline{x} \in [\underline{x}, x^{\infty})\) such that \(f\) is strictly monotonically increasing on \([\overline{x}, x^{\infty})\)
    \end{enumerate}
\end{assumption}

\begin{assumption}\label{assumption:general_minima}
    \(f\) is unimodal on \((-\infty, x^{\infty})\).
\end{assumption}

Evidently, these assumptions are equivalent to \cref{assumption:initialization,assumption:unimodal} by setting
\begin{align*}
    f\left(x\right) &= \mathcal{L}\left(\exp\left(x\right)\right)  
    \\
    \overline{x} &= \log \overline{h}
    \\
    \underline{x} &= \log \underline{h}
    \\
    x^{\infty} &= \log h^{\infty} \, .
\end{align*}

\newpage
\subsection{Proof of \cref{thm:pathdivergence}}
\printProofs[pathdivergence]

\subsection{Sufficient Condition for an Interval to Contain a Local Minimum (\Cref{thm:minima})}

Our adaptation algorithms primarily rely on GSS (\Cref{alg:gss}) to identify a local optimum.
To guarantee this, however, GSS needs to be initialized on an interval that contains a local minimum.
For this, the following lemma establishes a sufficient condition for identifying such intervals.
As such, we will use these conditions as invariants during the execution of GSS, such that it finds narrower and narrower intervals that continue to contain a local minimum.

\vspace{2ex}

\begin{lemma}\label{thm:minima}
    Suppose $f$ satisfies \cref{assumption:general} and there exist some $a < b < c$
    such that
    \begin{align*}
        f\left(b\right) \leq f\left(a\right) < \infty \quad\text{and}\quad f\left(b\right) \leq f\left(c\right) \, .
    \end{align*}
    Then \((a, c)\) contains a local minimum of \(f\) on $(-\infty, x^\infty)$.
\end{lemma}
\begin{proof}
    Consider any triplet $(a', b', c')$ consisting of 
    three points $a' < b' < c'$ with
    $f(b') \leq f(a') < \infty$ and $f(b') \leq f(c') < \infty$.
    Then $(a', c')$ contains a local minimum of $f$:
    $f$ attains its minimum on $[a', c']$ by the extreme value theorem,
    which is either on $(a', c')$---in which case the result holds immediately---or on $\{a', c'\}$---in which case the result holds because 
    $b'$ is a local minimum since $f(b') \leq \min\{f(a'), f(c')\}$.
 
    We now apply this result to triplets contained in $[a,c]$.
    First, if $c < x^\infty$, use the triplet $(a', b', c') = (a,b,c)$.
    For the remaining cases, assume $c \geq x^\infty$.
    If $b \geq \overline{x}$, set the triplet $(a', b', c') = (a, b, d)$
    for any $d \in (b, x^\infty)$. If $b < \overline{x}$
    and $f(\overline{x}) \geq f(b)$, set the triplet $(a', b', c') = (a, b, \overline{x})$. Otherwise, if $b < \overline{x}$ and $f(\overline{x}) < f(b)$,
    set the triplet $(a', b', c') = (b, \overline{x}, d)$ for any $d \in (\overline{x}, x^\infty)$.
\end{proof}


    
    

\newpage
\subsection{\(\mathtt{GoldenSectionSearch}\) (\Cref{thm:gss})}

We first establish that, under suitable initialization, GSS is able to locate a local minimum.
Most existing results assume that \(f\) is unimodal~\citep[\S 7.1]{luenberger_linear_2008} and show that GSS converges to the unique global minimum.
Here, we prove a more general result that holds under weaker conditions: GSS can also find a local minimum even when unimodality doesn't hold.
For this, we establish that our assumptions in \cref{assumption:general} and initializing at a triplet \((a, b, c)\) satisfying the condition in \cref{thm:minima} are sufficient.
Furthermore, while it is well known that GSS achieves a linear convergence rate with coefficient \(\phi \triangleq (1 + \sqrt{5})/2\), we could not find a proof that exactly applied to the GSS variant by~\citet{press_numerical_1992}, which is the one we use.
Therefore, we also prove linear convergence rate with a proof that precisely applies to \Cref{alg:gss}.

\vspace{1ex}

\input{theorems/thm_gss}

\newpage
\subsection{\(\mathtt{BracketMinimum}\) (\Cref{lemma:findminimum})}

We now prove that \(\mathtt{BracketMinimum}\) returns a triplet \((x^-, x_{\mathsf{mid}}, x^+)\) satisfying the condition in \cref{thm:minima}.
Furthermore, under \cref{assumption:general}, we analyze the width of the initial search interval represented by the triplet, \(\abs{x^+ - x^-}\).
Note that, while \(\mathtt{BracketMinimum}\) is designed to be valid even if \(x_0 \geq x^{\infty}\), accommodating this complicates the analysis.
Therefore, in the analysis that will follow, we will assume \(x_0 < x^{\infty}\).

\vspace{1ex}
\input{theorems/thm_find_minimum}

\clearpage
\subsection{\(\mathtt{Minimize}\) (\Cref{thm:minimize})}

We prove that combining \(\mathtt{BracketMinimum}\) and \(\mathtt{GoldenSectionSearch}\), which we call \(\mathtt{Minimize}\), results in an optimization algorithm that finds a point \(\epsilon\)-close to local minimum in \(\mathrm{O}\left(\log \left(\Delta/\epsilon\right) \right)\) time.
\vspace{1ex}

\input{theorems/thm_minimize}

\newpage
\subsection{\(\mathtt{AdaptStepsize}\) (Proof of \Cref{thm:initialize_stepsize})}

We now present the proof for the theoretical guarantees of \cref{alg:adapt_stepsize} in the main text, \cref{thm:initialize_stepsize}.
Since most of the heavy lifting in \cref{alg:adapt_stepsize} is done by \cref{alg:minimize}, \cref{thm:initialize_stepsize} is almost a corollary of \cref{thm:minimize}.
The main difference is that \cref{alg:adapt_stepsize} invokes \cref{alg:find_feasible} at \(t = 1\) and operates in log-space.
Therefore, the proof incorporates these two modifications into the results of \cref{thm:minimize}.

\vspace{1ex}

\printProofs[initializestepsize]

%% file: theorems/thm_gss.tex
\begin{lemma}\label{thm:gss}
    Suppose \cref{assumption:general} holds.
    Then, for any triplet \((a, b, c)\) satisfying \(a < b < c\), \(f(b) \leq f(a) < \infty\), and \(f(b) \leq f(c)\), \(\mathtt{GoldenSectionSearch}\left(f, a, b, c, \epsilon\right)\) returns a point \(x^* \in (-\infty, x^{\infty})\) that is \(\epsilon\)-close to a local minimum after
    \[
        \mathrm{\Theta}\left(
        \log \abs{c - a} \frac{1}{\epsilon} 
        \right)
    \]
    objective evaluations, where \(\phi = \left(1 + \sqrt{5}\right)/2 \).
\end{lemma}
\begin{proof}
    For clarity, let us denote the value of the variables \(x_0, x_1, x_2, x_3\) set at iteration \(k \geq 1\) of the while loop in Line 13-27 as \(x_0^k, x_1^k, x_2^k, x_3^k\).
    Before the while loop at \( k = 0\), they are initialized as follows: If \(\abs{c - b} \geq \abs{b - a} \), 
    \begin{align*}
        \left(x_0^{0}, \; x_1^{0}, \; x_2^{0},\; x_3^{0} \right)
        = 
        \left(a, \; b , \; b + \left(1 - \phi^{-1} \right) \left(c - b\right) , \; c \right) 
    \end{align*}
    and 
    \begin{align*}
        \left(x_0^{0}, \; x_1^{0}, \; x_2^{0},\; x_3^{0} \right)
        = 
        \left(a, \; b + \left(1 - \phi^{-1} \right) \left(b - a\right) , \; b , \; c \right) 
    \end{align*}
    otherwise.
    For all \(k \geq 0\), the following set of variables is set as follows: If \(f \left(x_2^k\right) < f \left(x_1^k\right)\), the next set of variables is set as
    \begin{align}
        &\left(x_0^{k+1}, \; x_1^{k+1}, \; x_2^{k+1},\; x_3^{k+1} \right)
        \nonumber
        \\
        &\qquad
        \triangleq
        \left(x_1^k, \; x_2^k, \; \phi^{-1} x_2^k + \left(1 - \phi^{-1}\right) x_3^k, \; x_3^{k} \right) \, ,
        \label{eq:gss_recursion_cond1}
    \end{align}
    and
    \begin{align*}
        &(x_0^{k+1}, \, x_1^{k+1}, \, x_2^{k+1}, \, x_3^{k+1}) 
        \\
        &\qquad
        \triangleq
        (x_0^k, \, \phi^{-1} x_1^k + \left(1 - \phi^{-1}\right) x_0^k, \, x_1^k, \, x_2^{k}) 
    \end{align*}
    otherwise.
    We also denote \(f_2^k \triangleq f\left(x_2^k\right)\)  and \(f_1^k \triangleq f\left(x_1^k\right)\).
    
    Assuming the algorithm terminates at some \(k^* < \infty\), the algorithm outputs either \(x_1^{k^*}\) or \(x_2^{k^*}\).
    Therefore, it suffice to show that \(k^* < \infty\), \(\abs{x_2^{k^*} - x_1^{k^*}} \leq \epsilon/2\), and that the interval \((x_0^{k^*}, x_3^{k^*})\) contains a local minimum.

    First, let's establish that \(k^* < \infty\).
    GSS terminates as soon as \(\abs{x_3^{k} - x_0^{k}} \leq \epsilon\) for some \(0 \geq k < \infty\).
    We will establish this by showing that \(\abs{x_3^k - x_0^k}\) satisfies a contraction.
    For this, however, we first have to show that \(x_1^k, x_2^k\) satisfy 
    \begin{align}
        x_1^k 
        &= 
        \phi^{-1} x_0^k + \left( 1 - \phi^{-1} \right) x_3^k
        \label{eq:gss_x1k_assumption}
        \\
        x_2^k 
        &= 
        \left( 1 - \phi^{-1} \right) x_0^k + \phi^{-1} x_3^k
        \label{eq:gss_x2k_assumption}
    \end{align}
    at all \(k \geq 0\).
    We will show this via induction. 
    Before we proceed, notice that the name ``golden'' section search comes from the fact that \(\phi\), the golden ratio, is the solution to the equation
    \begin{equation}
        \phi^2 = \phi + 1 \quad\Rightarrow\quad 1 - \phi^{-1} = \phi^{-2} \, . \label{eq:golden_ratio_equation}
    \end{equation}
    Now, for some \(k > 0\), suppose \Cref{eq:gss_x1k_assumption,eq:gss_x2k_assumption} hold.
    Then, if \(f_2^k < f_1^k\), 
    \begin{alignat*}{3}
        x_1^{k+1} 
        &= 
        x_2^k 
        \\
        &= 
        \left( 1 - \phi^{-1} \right) x_0^k +  \phi^{-1} x_3^k ,
        \\
        &= 
        \phi^{-2} x_0^k + \left(1 - \phi^{-2}\right) x_3^k
        &&\text{ (\cref{eq:golden_ratio_equation}) }
        \\
        &= 
        \phi^{-2} x_0^k + \left(1 + \phi^{-1}\right) \left(1 - \phi^{-1} \right) x_3^k
        \\
        &= 
        \phi^{-1} \left( \phi^{-1} x_0^k + \left(1 - \phi^{-1}\right) x_3^k \right) 
        \\
        &
        \qquad\qquad
        + \left(1 - \phi^{-1}\right) x_3^{k} \, ,
        \\
        &= 
        \phi^{-1} x_1^{k} + \left(1 - \phi^{-1}\right) x_3^{k} 
        &&\text{ (\cref{eq:gss_x1k_assumption}) }
        \\
        &= 
        \phi^{-1} x_0^{k+1} + \left(1 - \phi^{-1}\right) x_3^{k+1} \, .
        &&\text{ (\cref{eq:gss_recursion_cond1}) }
    \end{alignat*}
    This establishes \Cref{eq:gss_x1k_assumption} for \(k+1\).
    Similarly,
    \begin{alignat*}{4}
        x_2^{k+1} 
        &= 
        \phi^{-1} x_2^k +  \left(1 - \phi^{-1}\right) x_3^k \, ,
        \\
        &=
        \phi^{-1} \left( \left(1 - \phi^{-1}\right) x_0^k + \phi^{-1} x_3^k \right) 
        \\
        &\qquad\qquad
        + 
        {\left(1 - \phi^{-1} \right)} x_3^k
        &&\quad \text{(\cref{eq:gss_x2k_assumption})}
        \\
        &=
        \phi^{-1} \left(1 - \phi^{-1}\right) x_0^k 
        \\
        &\qquad\qquad
        + 
        {\left(1 - \phi^{-1} + \phi^{-2} \right)} x_3^k
        \\
        &=
        \left(1 - \phi^{-1}\right) \left(\phi^{-1} x_0^k 
        + \left(1 - \phi^{-1}\right) x_3^k \right) 
        \\
        &\qquad\qquad
        + \phi^{-1} x_3^{k} \, ,
        \\
        &=
        \left(1 - \phi^{-1}\right) x_1^{k} + \phi^{-1} x_3^{k+1}
        &&\quad \text{(\cref{eq:gss_x1k_assumption})}
        \\
        &=
        \left(1 - \phi^{-1}\right) x_0^{k+1} + \phi^{-1} x_3^{k+1} \, .
        &&\quad \text{(\cref{eq:gss_recursion_cond1})}
    \end{alignat*}
    This establishes \Cref{eq:gss_x2k_assumption} for \(k+1\).
    The proof for the remaining case of \(f_2^k \geq f_1^k\) is identical due to symmetry.
    Furthermore, the base case for \(k = 0\) automatically holds due to the condition on \((a, b, c)\).
    Therefore, \Cref{eq:gss_x1k_assumption,eq:gss_x2k_assumption} hold for all \(k \geq 0\).
    
    From \Cref{eq:gss_x1k_assumption,eq:gss_x2k_assumption}, we now have a precise rate of decrease for the interval \(\abs{x_3^{k} - x_0^{k}}\).
    That is, for \(f_2^k < f_1^k\), 
    \begin{align*}
        \abs{ x_3^{k} - x_0^{k} }
        &=
        \abs{ x_3^{k-1} - x_1^{k-1} }
        \\
        &=
        \abs{ x_3^{k-1} - \left( \phi^{-1} x_0^{k-1} + \left(1 - \phi^{-1}\right) x_3^{k-1} \right) }
        \\
        &=
        \phi^{-1} \abs{ x_3^{k-1} - x_0^{k-1}  } \,
    \end{align*}
    and for \(f_2^k \geq f_1^k\),
    \begin{align*}
        \abs{ x_3^{k} - x_0^{k} }
        &=
        \abs{ x_2^{k-1} - x_0^{k-1} }
        \\
        &=
        \abs{ \left( \left(1 - \phi^{-1}\right) x_0^{k-1} + \phi^{-1} x_3^{k-1} \right) - x_0^{k-1} }
        \\
        &=
        \phi^{-1} \abs{ x_3^{k-1} - x_0^{k-1}  } \, .
    \end{align*}
    Furthermore, 
    This implies, for all \(k \geq 1\), the interval \([x_3^k, x_0^k]\) shrinks at a geometrical rate 
    \[
        \abs{x_3^k - x_0^k} = \phi^{-k} \abs{x_3^0 - x_0^0} \, .
    \]
    Then
    \begin{align*}
        \abs{x_2^k - x_1^k} 
        &=
        \left(2 \phi^{-1} -  1\right)
        \abs{  
          x_0^k - x_3^k 
        } 
        \leq \epsilon/2
    \end{align*}
    can be guaranteed by iterating until the smallest iteration count \(k \geq 1\) that satisfies 
    \[
        \phi^{-k} \abs{x_3^0 - x_0^0} \leq \frac{1}{2 \phi^{-1} -  1} \, \frac{\epsilon}{2} \, ,
    \]
    \[
        k
        =
        \left\lceil \frac{1}{\log \phi} \log \frac{ 2 \left(2 \phi^{-1} - 1\right)  \abs{c - a} }{ \epsilon } \right\rceil \, ,
    \]
    which yields the execution time complexity statement.

    We now prove that the interval \((x_0^{k^*}, x_0^{k^*})\) contains a local minima.
    For this, we will prove a stronger result that \((x_0^{k}, x_3^{k})\) contains a local minimum for all \(k \geq 0\) by induction.
    Suppose, for some \(k \geq 1\), 
    \begin{align}
        \min\left( f_1^k, f_2^k \right) &\leq f(x_0^k) < \infty
        \label{eq:gss_induction_cond1}
        \\
        \min\left( f_1^k, f_2^k \right) &\leq f(x_3^k)
        \label{eq:gss_induction_cond2}
    \end{align}
    hold.
    If \(f_2^k < f_1^k\), the next set of variables is set as
    \begin{equation}
        (x_0^{k+1}, x_1^{k+1}, x_3^{k+1}) = (x_1^k, x_2^k, x_3^{k}) \, ,
    \end{equation}
    which guarantees that the inequalities
    \begin{align*}
        \min\left( f_1^{k+1}, f_2^{k+1} \right) &\leq f_1^{k+1} = f_2^k \leq f(x_0^{k+1}) < \infty \\
        \min\left( f_1^{k+1}, f_2^{k+1} \right) &\leq f_1^{k+1} = f_2^k \leq f(x_3^{k+1})
    \end{align*}
    hold.
    Otherwise, if \(f_2^k \geq f_1^k\), 
    \begin{equation*}
        (x_0^{k+1}, x_2^{k+1}, x_3^{k+1}) = (x_0^k, x_1^k, x_2^{k}) \, ,
    \end{equation*}
    guarantee
    \begin{align*}
        \min\left( f_1^{k+1}, f_2^{k+1} \right) &\leq f_2^{k+1} = f_1^k \leq f(x_0^{k+1}) < \infty \\
        \min\left( f_1^{k+1}, f_2^{k+1} \right) &\leq f_2^{k+1} = f_1^k \leq f(x_3^{k+1})  \, .
    \end{align*}
    The base case \(k = 0\) trivially holds by assumption \(f(b) < f\left(x_0^0\right) < \infty\), \(f(b) < f(x_3^0)\), and the fact that either \(x_1^0\) or \(x_2^0\) is set as \(b\).
    Therefore, \Cref{eq:gss_induction_cond1,eq:gss_induction_cond2} hold for all \(k \geq 0\).
    Now, \Cref{eq:gss_induction_cond1,eq:gss_induction_cond2} imply that either \((x_0^k, x_1^k, x_3^k)\) or \((x_0^k, x_2^k, x_3^k)\) satisfy the condition in \Cref{thm:minima}.
    Therefore, a local minimum is contained in \((x_0^k, x_3^k)\) for all \(k \geq 0\).
\end{proof}

%% file: theorems/thm_find_minimum.tex
\begin{lemma}\label{lemma:findminimum}
    Suppose \cref{assumption:general} holds.
    Then \(\mathtt{BracketMinimum}\left(f, x_0, r, c\right)\) for \(x_0 \in (-\infty, x^{\infty})\) returns a triplet \((x^-, x_{\mathsf{mid}}, x^+)\), where \(x^- < x_{\mathsf{mid}} < x^+\),
    \begin{equation}
    f\left(x_{\mathsf{mid}}\right) \leq f\left(x^-\right) < \infty, \quad
    f\left(x_{\mathsf{mid}}\right) \leq f\left(x^+\right),
    \label{eq:bracket_minimum_output_guarantee}
    \end{equation}
    and
    \begin{align*}
        \abs{ x^+ - x^- } 
        \leq 
        r^2 \left( (r + 1) [\overline{x} - x_0]_+ + [x_0 - \underline{x} ]_+ \right) + 3 r^2 c
    \end{align*}
    after
    \begin{align*}
        \mathrm{O}\left\{ {\left(\log r\right)}^{-1}\log_+\left((r [\overline{x} - x_0]_+ + [x_0 - \underline{x}]_+  )/c \right)\right\}
    \end{align*}
    objective evaluations.
\end{lemma}
\begin{proof}
    \(\mathtt{BracketMinimum}\) has two stages: exponential search to the right (Stage I) and exponential search to the left (Stage II).
    In the worst case, Stage I must pass $\overline{x}$ moving to the right starting from $x_0$, which takes at most $\mathrm{O}(\bar{k}_r)$ iterations, where
    \begin{align*}
        \bar{k}_r = \left\lceil (\log r)^{-1}\log_+((\overline{x} - x_0 )/c) \right\rceil \; .
    \end{align*} 
    Similarly, in the worst case Stage II must pass $\underline{x}$ moving to the left starting from $x_0 + cr^{\bar{k}_r}$,
    which takes at most $\mathrm{O}(\bar{k}_\ell)$ iterations, where
    \begin{align*}
        \bar{k}_\ell &= \left\lceil {\left(\log r\right)}^{-1} \log_+\left(\left(x_0 + c r^{\bar{k}_r} - \underline{x} \right)/c\right)\right\rceil 
        \\
        &\leq \left\lceil {\left(\log r\right)}^{-1}\log_+\left(x_0 
 + \left(r {\left[\overline{x} - x_0\right]}_+ - \underline{x} \right)/c \right)\right\rceil \, .
    \end{align*}
    Adding these two costs yields the stated result.
    At the end of Stage I, by inspection, we know that $f(x) \leq f(x^+)$,
    and that $f(x) < \infty$. 
    Also, Stage II continues until the first increase in objective value,
    which guarantees that $\infty > f(x^-) \geq f(x_{\text{mid}})$
    and $f(x_{\text{mid}}) \leq f(x) \leq f(x^+)$.
    Finally, 
    \begin{align*}
        \abs{x^+ - x^-} 
        &\leq 
        \left(x_0 + c r^{\bar k_r+1}\right) - \left(x_0 + c r^{\bar{k}_r} - c r^{\bar{k}_\ell+1} \right)
        \\
        &\leq 
        r c \left( r^{\bar{k}_r} + r^{\bar{k}_\ell} \right)
        \\
        &\leq 
        r \left( r  [\overline{x} - x_0]_+ + r c + r [x_0 + c r^{\bar{k}_r} - \underline{x}]_+ + r c \right)
        \\
        &\leq r^2 \left( [\overline{x} - x_0]_+ + [x_0 - \underline{x} ]_+ + c r^{\bar{k}_r} + 2 c \right) 
        \\
        &\leq 
        r^2 \left( [\overline{x} - x_0]_+ + [x_0 - \underline{x} ]_+ + r [\overline{x} - x_0]_+ + 3 c \right) 
        \\
        &=
        r^2 \left( (r + 1) [\overline{x} - x_0]_+ + [x_0 - \underline{x} ]_+ \right) + 3 r^2 c \; .
    \end{align*}
\end{proof}

%% file: theorems/thm_minimize.tex
\begin{theorem}\label{thm:minimize}
    Suppose \cref{assumption:general} holds.
    Then \(\mathtt{Minimize}\left(f, x_0, c, r, \epsilon\right)\) returns a point that is \(\epsilon\)-close to a local minimum after \(\mathcal{C}_{\mathsf{bm}} + \mathcal{C}_{\mathsf{gss}}\) objective evaluations, where
    \begin{align*}
        \mathcal{C}_{\mathsf{bm}} 
        &= 
        \mathrm{O}\left\{
        \frac{1}{\log r}
        \log_+ \left(\Delta \frac{r}{c}\right)
        \right\}
        \\
        \mathcal{C}_{\mathsf{gss}} 
        &= 
        \mathrm{O}\left\{ 
        \log_+ 
        \left(
        r^3 \Delta + r^2 c
        \right)
        \frac{1}{\epsilon}
        \right\}   
        \, ,
    \end{align*}
    where
    \(
        \Delta \triangleq
        {\left[ x_0 - \underline{x} \right] }_+
        +
        {\left[ \overline{x} - x_0 \right]}_+ \; .
    \)
\end{theorem}
\begin{proof}
    \(\mathcal{C}_{\mathsf{bm}}\) immediately follows from \cref{lemma:findminimum}, while
    \(\mathcal{C}_{\mathsf{gss}}\), on the other hand, follows from \cref{thm:gss} as
    \begin{align*}
        \mathcal{C}_{\mathsf{gss}}
        &=
        \mathrm{O}\left\{ \log_+ \left( x^+ - x^- \right) \frac{1}{\epsilon}  \right\} 
        \\
        &=
        \mathrm{O}\left\{ \log_+ \left( r^3 \Delta + r^2 c\right) \frac{1}{\epsilon}  \right\} 
        \; ,
    \end{align*}
    where we plugged in the bound on \(\abs{x^+ - x^-}\) from \cref{lemma:findminimum}.
    This yields the stated result.
    Furthermore, since \(\mathtt{BracketMinimum}\) returns a triplet \((x^-, x_{\mathsf{mid}}, x^+)\) that satisfies the requirement of \(\mathtt{GoldenSectionSearch}\) as stated in \cref{thm:gss}, the output \(x^* \in (-\infty, x^{\infty})\) is \(\epsilon\)-close to a local minimum.
\end{proof}

\begin{remark}
    In \cref{thm:minimize}, the ``difficulty'' of the problem is represented by \(\Delta \geq 0\), where the magnitude of \({[x_0 - \underline{x}]}_+\) and \({[\overline{x} - x_0]}_+\) represent the quality of the initialization \(x_0\) (how much \(x_0\) undershoots or overshoots \(\overline{x}\) and \(\underline{x}\)).
    Furthermore, we have \(\Delta \geq \abs{\overline{x} - \underline{x}}\), where \(\abs{\overline{x} - \underline{x}}\) can be thought as the quantitative multimodality of the problem.
    Therefore, the execution time of \(\mathtt{Minimize}\) becomes longer as the problem becomes more multimodal and the initialization is far from \([\overline{x}, \underline{x}]\).
\end{remark}
\vspace{1ex}
\begin{remark}
    The execution time of \(\mathtt{Minimize}(f, x_0, c, r, \epsilon)\) depends on \(r\) and \(c\).
    In general, the best-case performance (\(\Delta = 0\)) can only become worse as \(c\) increases.
    On the other hand, in the worst-case when \(\Delta\) is large, increasing \(r\) reduces \(\mathcal{C}_{\mathsf{bm}}\), while slowly making \(\mathcal{C}_{\mathrm{gss}}\) worse.
    Therefore, a large \(r\) improves the worst-case performance.
\end{remark}

%% file: section_backward_kernels.tex
\subsection{Some backward kernels are not like the others}

Here, we will discuss some options for the ``backward kernel'' used in SMC samplers in the static model setting~\citep{delmoral_sequential_2006,neal_annealed_2001}.

\vspace{-2ex}
\paragraph{Detailed Balance Formula.}
In the literature, the choice
\begin{alignat}{4}
    L_{t-1}^{\mathsf{dbf}}\left(\vx_t, \vx_{t-1}\right) &\triangleq \frac{ \gamma_{t}\left(\vx_{t}\right) K_{t}\left(\vx_{t-1}, \vx_t\right) }{ \gamma_t\left(\vx_{t-1}\right) } \, ,
    \label{eq:dbf_backward}
\end{alignat}
which we will refer to as the ``detailed balance formula backward kernel,'' has been the widely used~\citep{dai_invitation_2022,bernton_schrodinger_2019,heng_controlled_2020}.
Conveniently, \cref{eq:dbf_backward} results in a simple expression for the potential
\[
    G_t\left(\vx_{t-1}, \vx_t\right) = \frac{ \gamma_t\left(\vx_{t-1}\right) }{ \gamma_{t-1}\left(\vx_{t-1}\right) } \, ,
\]
which does not involve the densities of \(K_t\).
The origin of this backward kernel is that many \(\pi_t\)-invariant MCMC kernels used in practice satisfy the detailed balance formula~\citep[Def. 6.45]{robert_monte_2004} with \(\pi_t\),
\[
    \pi_t\left(\vx_{t-1}\right)  K_t\left(\vx_{t-1}, \vx_{t}\right)
    =
    \pi_t\left(\vx_{t}\right) K_t\left(\vx_{t}, \vx_{t-1}\right)  \, ,
\]
which, given \(L_{t-1}^{\mathsf{dbf}}\left(\vx_{t-1}, \vx_{t}\right) = K_t\left(\vx_{t-1}, \vx_t\right) \), yields \cref{eq:dbf_backward} after re-aranging.

\vspace{-2ex}
\paragraph{The detailed balance formula backward kernel is biased.}
Now, let's focus on the fact that \(K_t = L_{t-1}^{\mathsf{dbf}}\) only holds under the detailed balance condition.
Said differently, \(L_{t-1}^{\mathsf{dbf}}\) is a properly \textit{normalized} kernel only when \(K_t\) satisfies detailed balance.
This implies that, for non-reversible kernels like LMC, using the detailed balance formula kernel with \(h > 0\) may result in biased normalized constant estimates.
This bias can be substantial, as we will see in \cref{section:backward_kernel_experiment}.
Fortunately, this bias does diminish as \(h_t \to 0\) and \(T \to \infty\) since the continuous Langevin dynamics is reversible under stationarity~\citep{heng_controlled_2020}.
However, the need for smaller step sizes means that a larger number of SMC steps \(T\) has to be taken for the Markov process to converge.

\vspace{-2ex}
\paragraph{Forward Kernel.}
The properly normalized analog of \(L^{\mathsf{dbf}}_{t-1}\) at time \(t \geq 1\) is the ``forward'' kernel
\[
    L_{t-1}^{\mathsf{fwd}}\left(\vx_t, \vx_{t-1}\right) \triangleq K_{t}\left(\vx_t, \vx_{t-1}\right) \, .
\]
This has been used, for example, by \citet{thin_monte_2021}.
Recall that the ``optimal'' \(L_{t-1}\) is a kernel that transports the particles following \(P_t\) to follow \(P_{t-1}\).
The fact that we are using \(K_t\) to do this implies that we are assuming \(P_t \approx P_{t-1}\), which is only true if \(T\) is large.
We propose a different option, which should work even when \(T\) is moderate or small.

\vspace{-2ex}
\paragraph{Time-Correct Forward Kernel.}
In \cref{section:implementation_lmc}, we used the forward kernel at time \(t - 1\), 
\[
    L_{t-1}^{\text{\textsf{tc-fwd}}}\left(\vx_t, \vx_{t-1}\right) \triangleq K_{\textcolor{purple}{t-1}}\left(\vx_t, \vx_{t-1}\right) \, ,
\]
which we will refer to as the time-correct forward kernel.
Unlike \(L^{\mathsf{fwd}}_{t-1}\), the stationary distribution of this transport map is properly \(\pi_{t-1}\).
Informally, the reasoning is that
\begin{align*}
    \frac{
        \left( \pi_t \otimes K_{t-1} \right) \left(x_{t}, x_{t-1}\right)
    }{
        \left( \pi_{t-1} \otimes K_{t} \right) \left(x_{t-1}, x_t\right)
    }
    \approx
    \frac{
        \left( \pi_t \otimes \pi_{t-1} \right)  \left(x_{t}, x_{t-1}\right)
    }{
        \left( \pi_{t-1} \otimes \pi_{t} \right) \left(x_{t-1}, x_t\right)
    }
    =
    1 \, .
\end{align*}
Therefore, this should result in lower variance.

\subsection{Empirical Evaluation}\label{section:backward_kernel_experiment}

\begin{figure}
    \vspace{-1ex}
    \centering
    \includegraphics[scale=0.95]{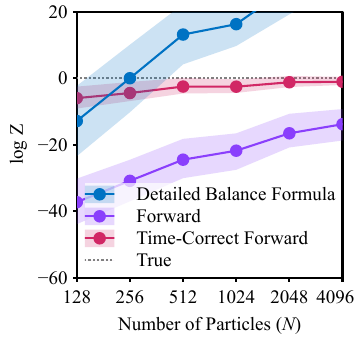}
    \vspace{-3ex}
    \caption{
    \textbf{Comparison of backward kernels for SMC-LMC.}
    The solid lines are the median, while the colored bands mark the \(80\%\) empirical quantile over 256 replications.
    }
    \label{fig:backward_kernel_comparison}
\vspace{-2ex}
\end{figure}

\paragraph{Setup.}
We compare the three backward kernels on a toy problem with \(d = 10\) dimensional Gaussians: \(\pi = \mathcal{N}\left(30 \cdot \mathrm{1}_d, \mathrm{I}_d\right)\), \(q_0 = \mathcal{N}\left(0_d, \mathrm{I}_d\right)\).
Since the scale of the target distribution is constant under geometric annealing, a fixed step size \(h = h_t = 0.5\) should work well.
We use a linear schedule with \(T = 64\).

\vspace{-2ex}
\paragraph{Results.}
The results are shown in \cref{fig:backward_kernel_comparison}.
The backward kernel from the detailed balance formula severely overestimates the normalizing constant due to bias, while the forward kernel exhibits significantly higher variance than the time-correct forward kernel.

\subsection{Conclusions}
We have demonstrated that caution must be taken when using the popular backward kernel based on the detail-balance formula.
Instead, we have proposed the ``time-correct forward kernel,'' which is not only valid but also results in substantially lower variance.
Unfortunately, the time-correct forward kernel is only available for MCMC kernels that have a tractable density, which may not be the case; for instance, the KLMC kernel used in \cref{section:klmc} does not have this option.
However, whenever it is available, it should be preferred.

%% file: section_additional_results.tex
\subsection{Comparison Against Fixed Step Sizes}\label{section:additional_fixed_stepsizes}

\subsubsection{SMC-LMC}

\begin{figure}[H]
    \vspace{-2ex}
    \centering
    \subfloat[Funnel]{
        \hspace{-1em}
        \includegraphics[]{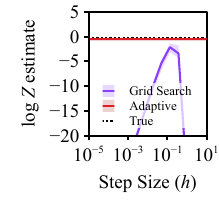}
        \vspace{-2ex}
    }
    \subfloat[Brownian]{
        \hspace{-1em}
        \includegraphics[]{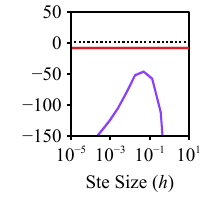}
        \vspace{-2ex}
    }
    \subfloat[Sonar]{
        \hspace{-1em}
        \includegraphics[]{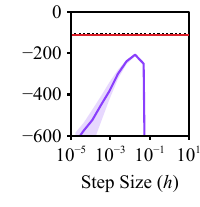}
        \vspace{-2ex}
    }
    \subfloat[Pines]{
        \hspace{-1em}
        \includegraphics[]{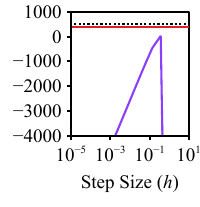}
        \vspace{-2ex}
    }
    \subfloat[Bones]{
        \hspace{-1em}
        \includegraphics[]{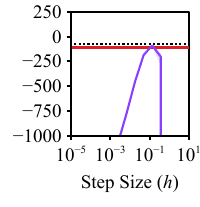}
        \vspace{-2ex}
    }
    \\
    \subfloat[Surgical]{
        \hspace{-1em}
        \includegraphics[]{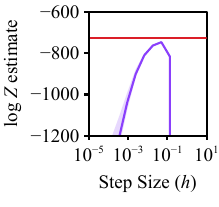}
        \vspace{-2ex}
    }
    \subfloat[HMM]{
        \hspace{-1em}
        \includegraphics[]{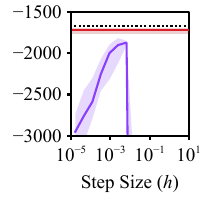}
        \vspace{-2ex}
    }
    \subfloat[Loss Curve]{
        \hspace{-1em}
        \includegraphics[]{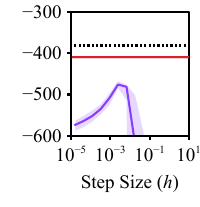}
        \vspace{-2ex}
    }
    \subfloat[Pilots]{
        \hspace{-1em}
        \includegraphics[]{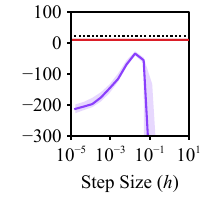}
        \vspace{-2ex}
    }
    \subfloat[Diamonds]{
        \hspace{-1em}
        \includegraphics[]{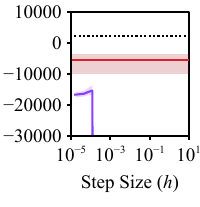}
        \vspace{-2ex}
    }
    \\
    \subfloat[Seeds]{
        \hspace{-1em}
        \includegraphics[]{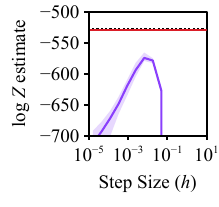}
        \vspace{-2ex}
    }
    \subfloat[Downloads]{
        \hspace{-1em}
        \includegraphics[]{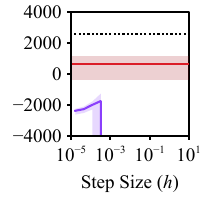}
        \vspace{-2ex}
    }
    \subfloat[Rats]{
        \hspace{-1em}
        \includegraphics[]{figures/gridsearch_rats_smcula_02}
        \vspace{-2ex}
    }
    \subfloat[Radon]{
        \hspace{-1em}
        \includegraphics[]{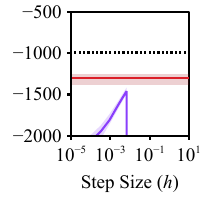}
        \vspace{-2ex}
    }
    \subfloat[Election88]{
        \hspace{-1em}
        \includegraphics[]{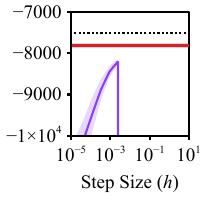}
        \vspace{-2ex}
    }
    \\
    \subfloat[Butterfly]{
        \hspace{-1em}
        \includegraphics[]{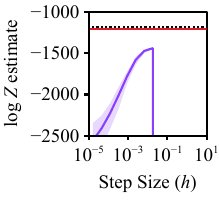}
        \vspace{-2ex}
    }
    \subfloat[Birds]{
        \hspace{-1em}
        \includegraphics[]{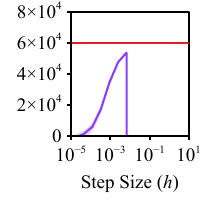}
        \vspace{-2ex}
    }
    \subfloat[Drivers]{
        \hspace{-1em}
        \includegraphics[]{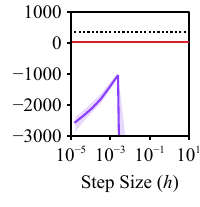}
        \vspace{-2ex}
    }
    \subfloat[Capture]{
        \hspace{-1em}
        \includegraphics[]{figures/gridsearch_capture_smcula_02}
        \vspace{-2ex}
    }
    \subfloat[Science]{
        \hspace{-1em}
        \includegraphics[]{figures/gridsearch_science_smcula_02}
        \vspace{-2ex}
    }
    \\
    \subfloat[Three Men]{
        \hspace{-1em}
        \includegraphics[]{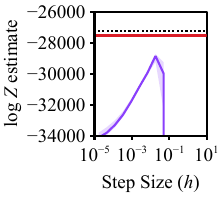}
        \vspace{-2ex}
    }
    \subfloat[TIMSS]{
        \hspace{-1em}
        \includegraphics[]{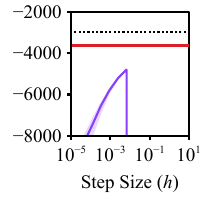}
        \vspace{-2ex}
    }
    \vspace{-1ex}
    \caption{
        \textbf{SMC-LMC with adaptive tuning v.s. fixed step sizes.}
        The solid lines are the median estimate of \(\log Z\), while the colored regions are the \(80\%\) empirical quantiles computed over 32 replications.
    }
    \vspace{-2ex}
\end{figure}

\newpage
\subsubsection{SMC-KLMC}

\begin{figure}[H]
    \centering
    \subfloat[Funnel]{
        \hspace{-1em}
        \includegraphics[]{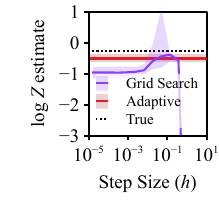}
        \vspace{-2ex}
    }
    \subfloat[Brownian]{
        \hspace{-1em}
        \includegraphics[]{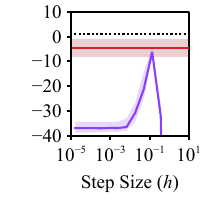}
        \vspace{-2ex}
    }
    \subfloat[Sonar]{
        \hspace{-1em}
        \includegraphics[]{figures/gridsearch_sonar_smcula_02}
        \vspace{-2ex}
    }
    \subfloat[Pines]{
        \hspace{-1em}
        \includegraphics[]{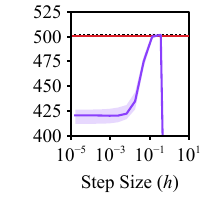}
        \vspace{-2ex}
    }
    \subfloat[Bones]{
        \hspace{-1em}
        \includegraphics[]{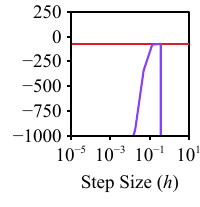}
        \vspace{-2ex}
    }
    \\
    \subfloat[Surgical]{
        \hspace{-1em}
        \includegraphics[]{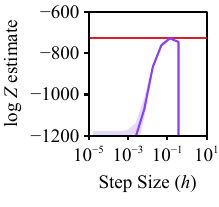}
        \vspace{-2ex}
    }
    \subfloat[HMM]{
        \hspace{-1em}
        \includegraphics[]{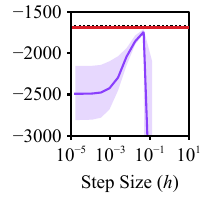}
        \vspace{-2ex}
    }
    \subfloat[Loss Curve]{
        \hspace{-1em}
        \includegraphics[]{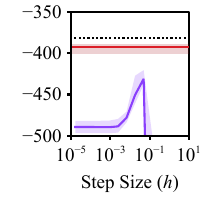}
        \vspace{-2ex}
    }
    \subfloat[Pilots]{
        \hspace{-1em}
        \includegraphics[]{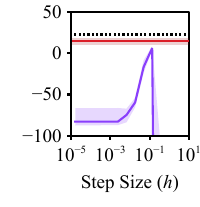}
        \vspace{-2ex}
    }
    \subfloat[Diamonds]{
        \hspace{-1em}
        \includegraphics[]{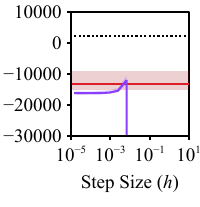}
        \vspace{-2ex}
    }
    \\
    \subfloat[Seeds]{
        \hspace{-1em}
        \includegraphics[]{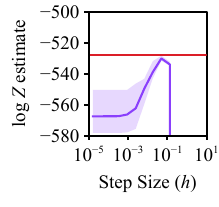}
        \vspace{-2ex}
    }
    \subfloat[Downloads]{
        \hspace{-1em}
        \includegraphics[]{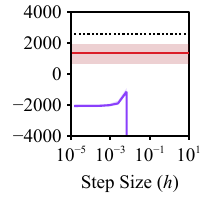}
        \vspace{-2ex}
    }
    \subfloat[Rats]{
        \hspace{-1em}
        \includegraphics[]{figures/gridsearch_rats_smcuhmc_02}
        \vspace{-2ex}
    }
    \subfloat[Radon]{
        \hspace{-1em}
        \includegraphics[]{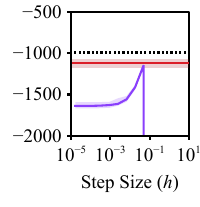}
        \vspace{-2ex}
    }
    \subfloat[Election88]{
        \hspace{-1em}
        \includegraphics[]{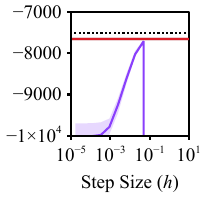}
        \vspace{-2ex}
    }
    \\
    \subfloat[Butterfly]{
        \hspace{-1em}
        \includegraphics[]{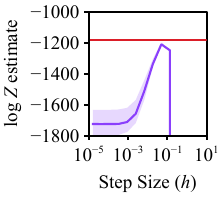}
        \vspace{-2ex}
    }
    \subfloat[Birds]{
        \hspace{-1em}
        \includegraphics[]{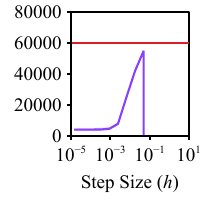}
        \vspace{-2ex}
    }
    \subfloat[Drivers]{
        \hspace{-1em}
        \includegraphics[]{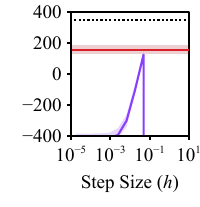}
        \vspace{-2ex}
    }
    \subfloat[Capture]{
        \hspace{-1em}
        \includegraphics[]{figures/gridsearch_capture_smcuhmc_02}
        \vspace{-2ex}
    }
    \subfloat[Science]{
        \hspace{-1em}
        \includegraphics[]{figures/gridsearch_science_smcuhmc_02}
        \vspace{-2ex}
    }
    \\
    \subfloat[Three Men]{
        \hspace{-1em}
        \includegraphics[]{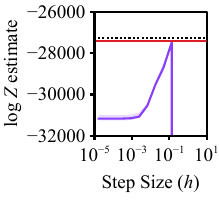}
        \vspace{-2ex}
    }
    \subfloat[TIMSS]{
        \hspace{-1em}
        \includegraphics[]{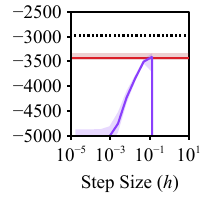}
        \vspace{-2ex}
    }
    \vspace{-1ex}
    \caption{
        \textbf{SMC-KLMC with adaptive tuning v.s. fixed step sizes and refreshment rates.}
        For SMC-KLMC with fixed parameters \(h, \rho\), we show the result of the best-performing refreshment rate.
        The solid lines are the median estimate of \(\log Z\), while the colored regions are the \(80\%\) empirical quantiles computed over 32 replications.
    }
    \vspace{-2ex}
\end{figure}

\clearpage
\subsection{Comparison Against End-to-End Optimization}\label{section:end_to_end_additional}
\vspace{-1ex}

For all results, the ``cost'' is calculated as the cumulative number of gradients and Hessian evaluations used by each method. 
Also, end-to-end optimization with variational tuning of the reference $q$ is referred as ``End-to-End + VI.''
(End-to-end optimization methods require Hessians.)
For all figures, the error bars/bands are \(80\%\) empirical quantiles computed from 32 replications, while \(\gamma\) is the Adam step size used for end-to-end optimization.

\vspace{-1ex}
\subsubsection{SMC-LMC}

\begin{figure}[H]
    \vspace{-2ex}
    \centering
    \subfloat[\(\gamma = 10^{-4}\)]{
        \includegraphics[scale=0.9]{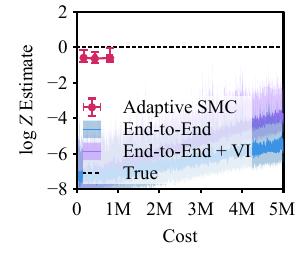} 
        \vspace{-2ex}
    }
    \subfloat[\(\gamma = 10^{-3}\)]{
        \includegraphics[scale=0.9]{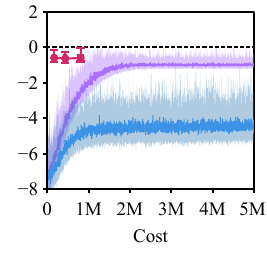} 
        \vspace{-2ex}
    }
    \subfloat[\(\gamma = 10^{-2}\)]{
        \includegraphics[scale=0.9]{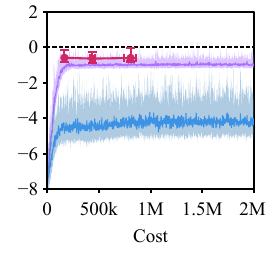} 
        \vspace{-2ex}
    }
    \vspace{-1.5ex}
    \caption{\textbf{Comparison against end-to-end optimization on Funnel.}}
\end{figure}

\begin{figure}[H]
    \vspace{-4ex}
    \centering
    \subfloat[\(\gamma = 10^{-4}\)]{
        \includegraphics[scale=0.9]{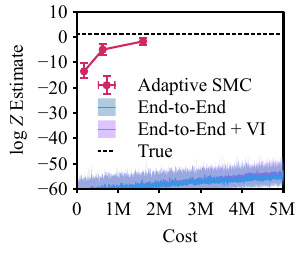} 
        \vspace{-2ex}
    }
    \subfloat[\(\gamma = 10^{-3}\)]{
        \includegraphics[scale=0.9]{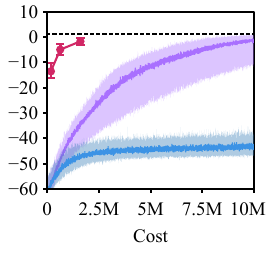} 
        \vspace{-2ex}
    }
    \subfloat[\(\gamma = 10^{-2}\)]{
        \includegraphics[scale=0.9]{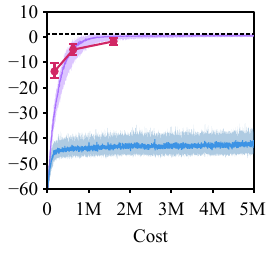} 
        \vspace{-2ex}
    }
    \vspace{-1.5ex}
    \caption{
        \textbf{Comparison against end-to-end optimization on Brownian.}
        The result for variational reference tuning (End-to-End + VI) for \(\gamma = 10^{-2}\) is omitted as most runs diverged.
    }
\end{figure}

\begin{figure}[H]
    \vspace{-4ex}
    \centering
    \subfloat[\(\gamma = 10^{-4}\)]{
        \includegraphics[scale=0.9]{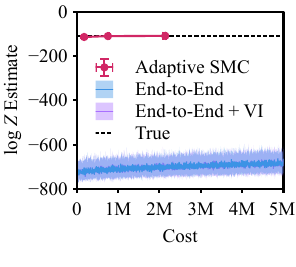} 
        \vspace{-2ex}
    }
    \subfloat[\(\gamma = 10^{-3}\)]{
        \includegraphics[scale=0.9]{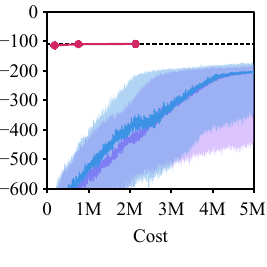} 
        \vspace{-2ex}
    }
    \subfloat[\(\gamma = 10^{-2}\)]{
        \includegraphics[scale=0.9]{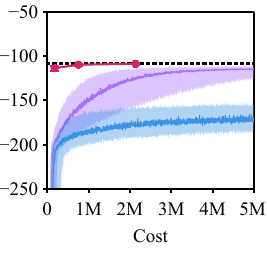} 
        \vspace{-2ex}
    }
    \vspace{-1.5ex}
    \caption{\textbf{Comparison against end-to-end optimization on Sonar.}}
\end{figure}

\begin{figure}[H]
    \vspace{-4ex}
    \centering
    \subfloat[\(\gamma = 10^{-4}\)]{
        \includegraphics[scale=0.9]{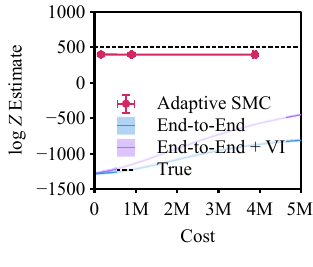} 
        \vspace{-2ex}
    }
    \subfloat[\(\gamma = 10^{-3}\)]{
        \includegraphics[scale=0.9]{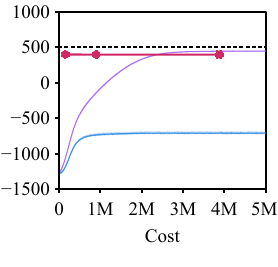} 
        \vspace{-2ex}
    }
    \subfloat[\(\gamma = 10^{-2}\)]{
        \includegraphics[scale=0.9]{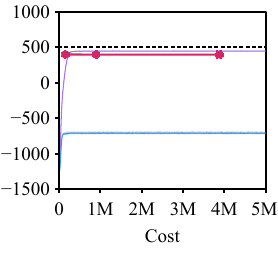} 
        \vspace{-2ex}
    }
    \vspace{-1.5ex}
    \caption{\textbf{Comparison against end-to-end optimization on Pines.}}
\end{figure}

\newpage
\subsubsection{SMC-KLMC}

\begin{figure}[H]
    \vspace{-1ex}
    \centering
    \subfloat[\(\gamma = 10^{-4}\)]{
        \includegraphics{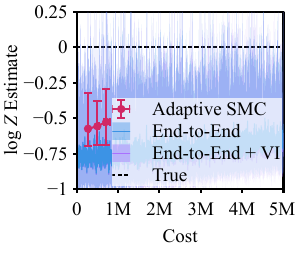} 
        \vspace{-2ex}
    }
    \subfloat[\(\gamma = 10^{-3}\)]{
        \includegraphics{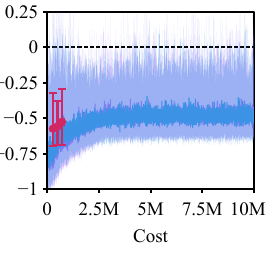} 
        \vspace{-2ex}
    }
    \subfloat[\(\gamma = 10^{-2}\)]{
        \includegraphics{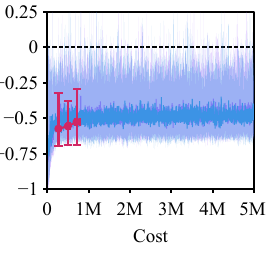} 
        \vspace{-2ex}
    }
    \vspace{-1ex}
    \caption{\textbf{Comparison against end-to-end optimization on Funnel.}}
\end{figure}

\begin{figure}[H]
    \vspace{-4ex}
    \centering
    \subfloat[\(\gamma = 10^{-4}\)]{
        \includegraphics{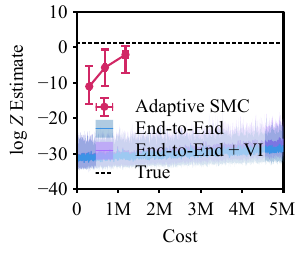} 
        \vspace{-2ex}
    }
    \subfloat[\(\gamma = 10^{-3}\)]{
        \includegraphics{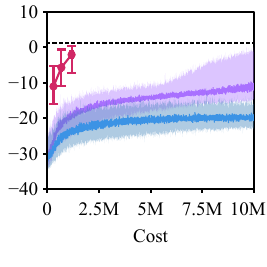} 
        \vspace{-2ex}
    }
    \subfloat[\(\gamma = 10^{-2}\)]{
        \includegraphics{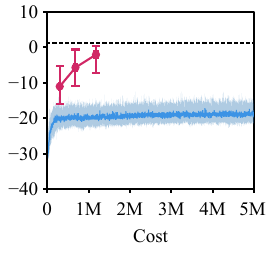} 
        \vspace{-2ex}
    }
    \vspace{-1ex}
    \caption{\textbf{Comparison against end-to-end optimization on Brownian.}}
\end{figure}

\begin{figure}[H]
    \vspace{-4ex}
    \centering
    \subfloat[\(\gamma = 10^{-4}\)]{
        \includegraphics{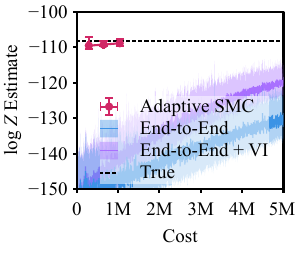} 
        \vspace{-2ex}
    }
    \subfloat[\(\gamma = 10^{-3}\)]{
        \includegraphics{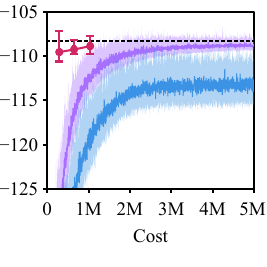} 
        \vspace{-2ex}
    }
    \subfloat[\(\gamma = 10^{-2}\)]{
        \includegraphics{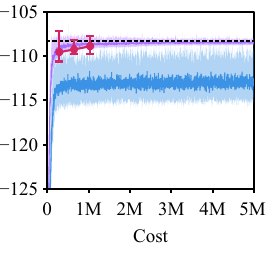} 
        \vspace{-2ex}
    }
    \vspace{-1ex}
    \caption{\textbf{Comparison against end-to-end optimization on Sonar.}}
\end{figure}

\begin{figure}[H]
    \vspace{-4ex}
    \centering
    \subfloat[\(\gamma = 10^{-4}\)]{
        \includegraphics{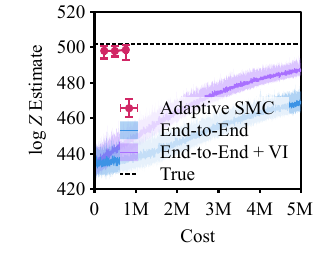} 
        \vspace{-2ex}
    }
    \subfloat[\(\gamma = 10^{-3}\)]{
        \includegraphics{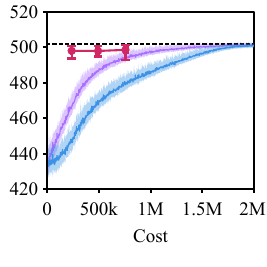} 
        \vspace{-2ex}
    }
    \subfloat[\(\gamma = 10^{-2}\)]{
        \includegraphics{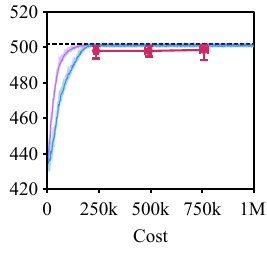} 
        \vspace{-2ex}
    }
    \vspace{-1ex}
    \caption{\textbf{Comparison against end-to-end optimization on Pines.}}
\end{figure}

\newpage
\subsection{Adaptation Cost}\label{section:adaptation_cost}
\vspace{-1ex}

In this section, we will visualize the cost of adaptation of our algorithm.
In particular, we show the number of objective evaluations used at each SMC iteration during adaptation.
Recall that the cost of evaluating our objective is in the order of \(\mathcal{O}\left(B\right)\) unnormalized log-density evaluations (\(\log \gamma\)) and its gradients (\(\nabla \log \gamma\)), where \(B\) is the number of subsampled particles (\cref{section:adaptation}).
Therefore, the cost of \(N/B\) adaptation objective evaluations at every SMC step roughly amounts to the cost of a single vanilla SMC run with \(N\) particles.
That is, for \(N = 1024\) and \(B = 128\), the cost of running our adaptive SMC sampler is comparable to two to three times that of a vanilla SMC sampler.
The exact number of objective evaluations spent at each SMC iteration is shown in the figures that will follow.
All experiments used \(N = 1024\), \(B = 128\), and \(T = 64\)

\vspace{-2ex}
\subsubsection{SMC-LMC}

\begin{figure}[H]
    \vspace{-1ex}
    \centering
    \subfloat[Funnel]{
        \includegraphics[]{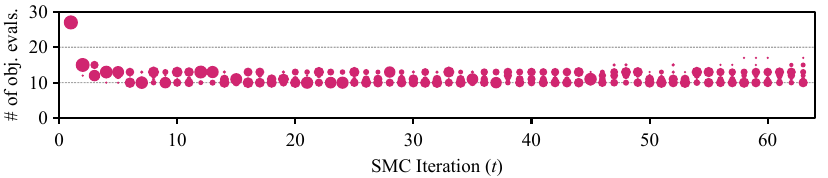}
    }
    \\
    \subfloat[Brownian]{
        \includegraphics[]{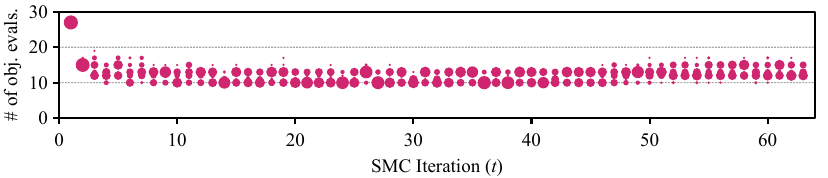}
    }
    \\
    \subfloat[Sonar]{
        \includegraphics[]{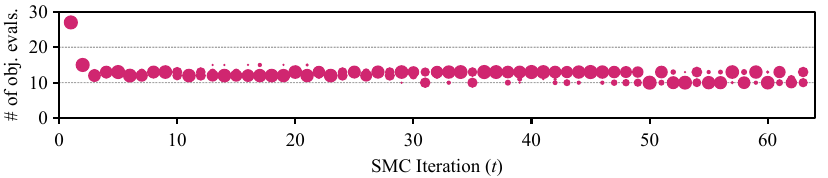}
    }
    \\
    \subfloat[Pines]{
        \includegraphics[]{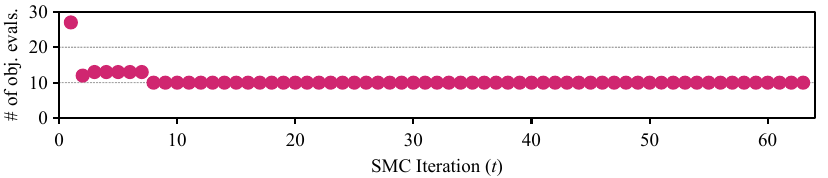}
    }
    \\
    \subfloat[Bones]{
        \includegraphics[]{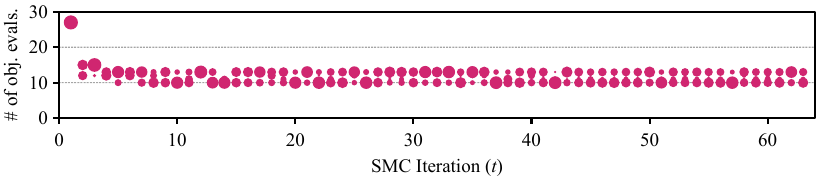}
    }
    \vspace{-2ex}
    \caption{
    \textbf{Number of objective evaluations spent during adaptation at each SMC iteration.}
    The size of the markers represents the proportion of runs that spent each respective number of evaluations among 32 independent runs.
    }
\end{figure}

\begin{figure}[H]
    \centering
    \subfloat[Surgical]{
        \includegraphics[]{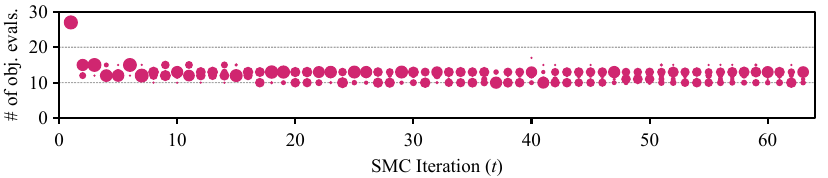}
    }
    \\
    \subfloat[HMM]{
        \includegraphics[]{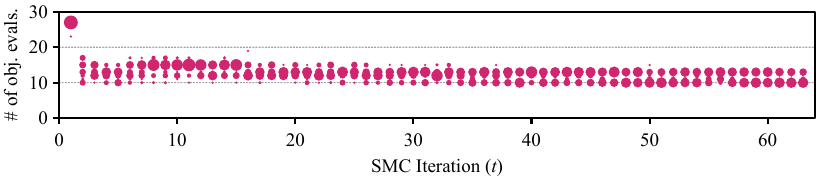}
    }
    \\
    \subfloat[Loss Curves]{
        \includegraphics[]{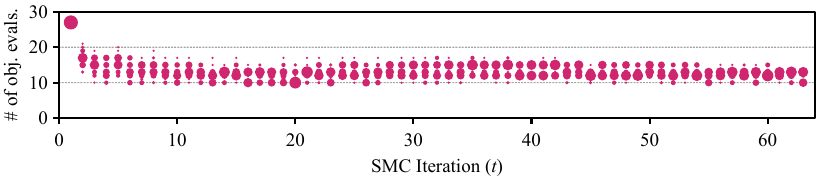}
    }
    \\
    \subfloat[Pilots]{
        \includegraphics[]{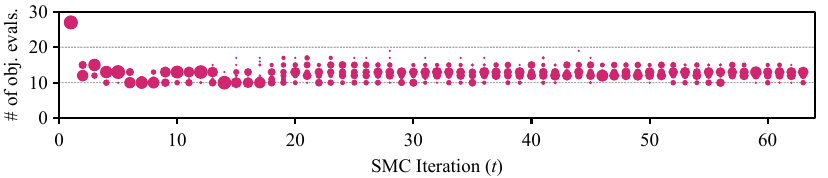}
    }
    \\
    \subfloat[Diamonds]{
        \includegraphics[]{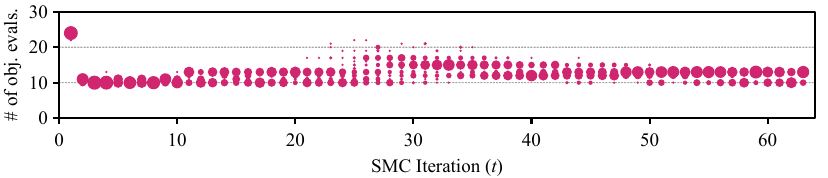}
    }
    \\
    \subfloat[Seeds]{
        \includegraphics[]{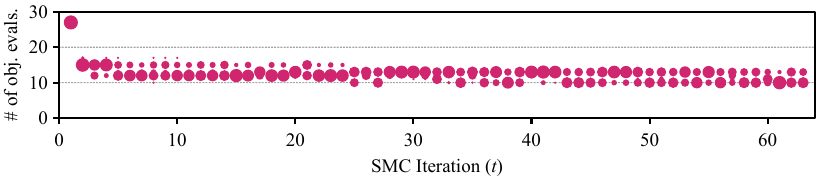}
    }
    \caption{
    \textbf{Number of objective evaluations spent during adaptation at each SMC iteration.}
    The size of the markers represents the proportion of runs that spent each respective number of evaluations among 32 independent runs.
    }
\end{figure}
    
\begin{figure}[H]
    \centering
    \subfloat[Downloads]{
        \includegraphics[]{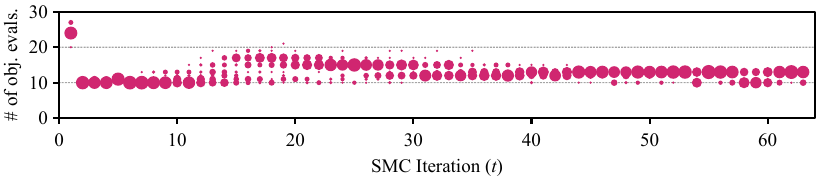}
    }
    \\
    \subfloat[Rats]{
        \includegraphics[]{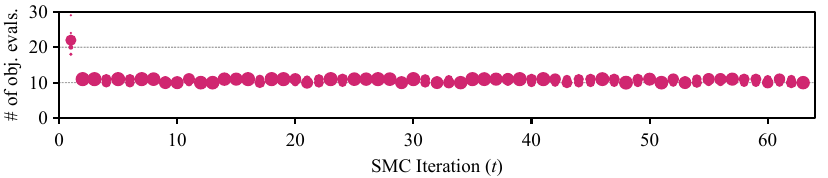}
    }
    \\
    \subfloat[Radon]{
        \includegraphics[]{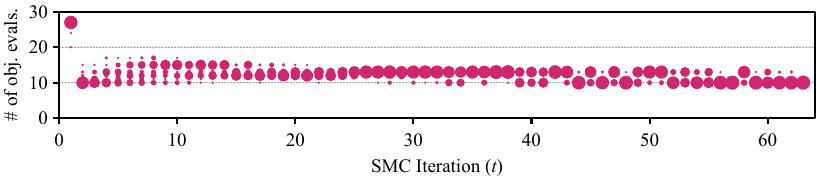}
    }
    \\
    \subfloat[Election88]{
        \includegraphics[]{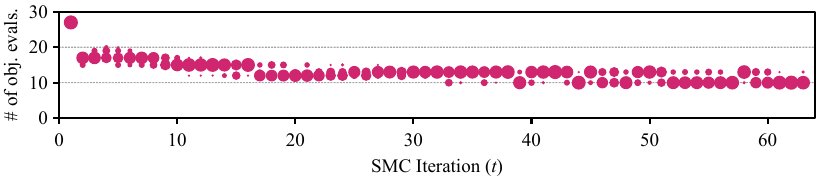}
    }
    \\
    \subfloat[Butterfly]{
        \includegraphics[]{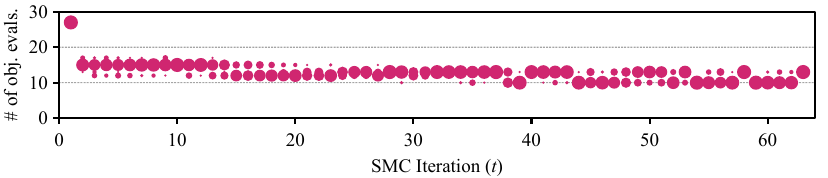}
    }
    \\
    \subfloat[Birds]{
        \includegraphics[]{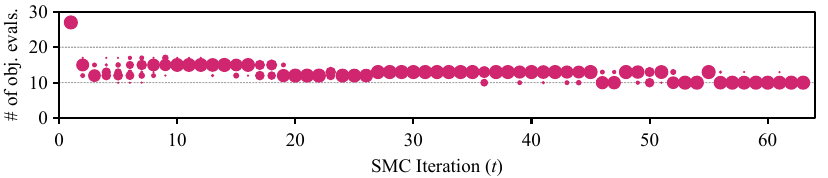}
    }
    \caption{
    \textbf{Number of objective evaluations spent during adaptation at each SMC iteration (continued).}
    The size of the markers represents the proportion of runs that spent each respective number of evaluations among 32 independent runs.
    }
\end{figure}

\begin{figure}[H]
    \vspace{4ex}
    \centering
    \subfloat[Drivers]{
        \includegraphics[]{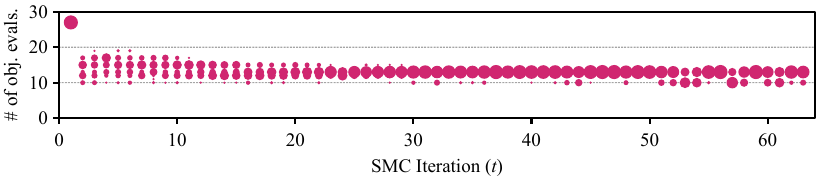}
    }
    \\
    \subfloat[Capture]{
        \includegraphics[]{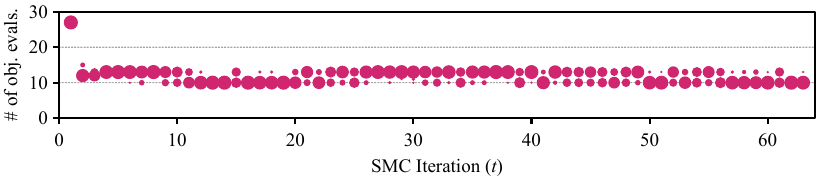}
    }
    \\
    \subfloat[Science]{
        \includegraphics[]{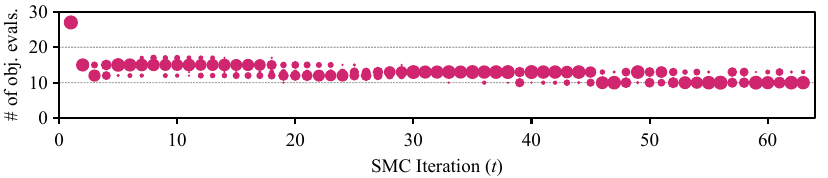}
    }
    \\
    \subfloat[Three Men]{
        \includegraphics[]{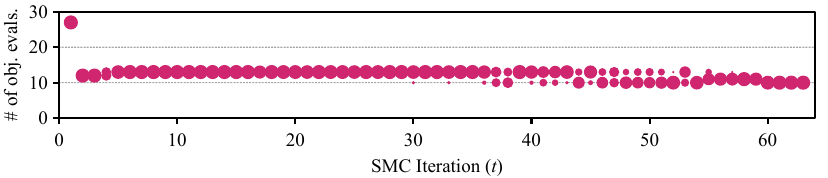}
    }
    \\
    \subfloat[TIMSS]{
        \includegraphics[]{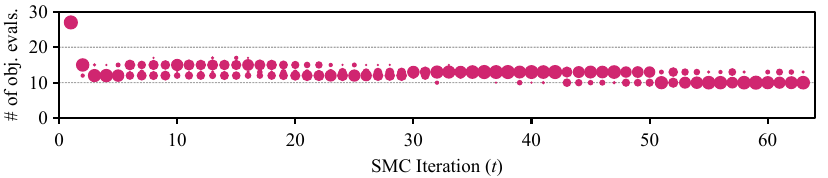}
    }
    \caption{
    \textbf{Number of objective evaluations spent during adaptation at each SMC iteration (continued).}
    The size of the markers represents the proportion of runs that spent each respective number of evaluations among 32 independent runs.
    }
\end{figure}

\vspace{-1ex}
\subsubsection{SMC-KLMC}

\begin{figure}[H]
    \vspace{-1ex}
    \centering
    \subfloat[Funnel]{
        \includegraphics[]{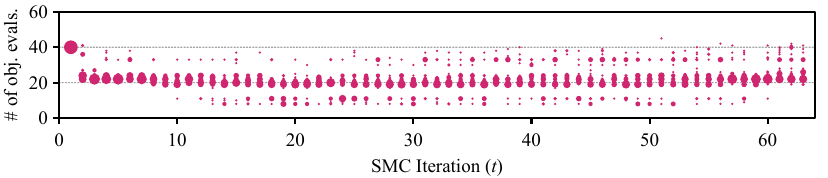}
        \vspace{-0.1ex}
    }
    \\
    \subfloat[Brownian]{
        \includegraphics[]{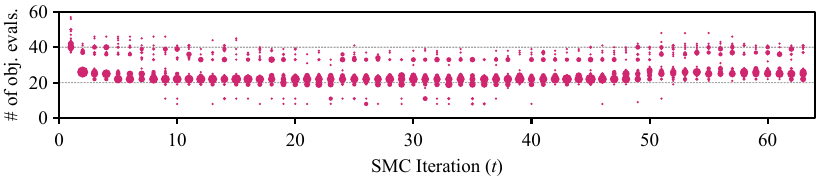}
        \vspace{-0.1ex}
    }
    \\
    \subfloat[Sonar]{
        \includegraphics[]{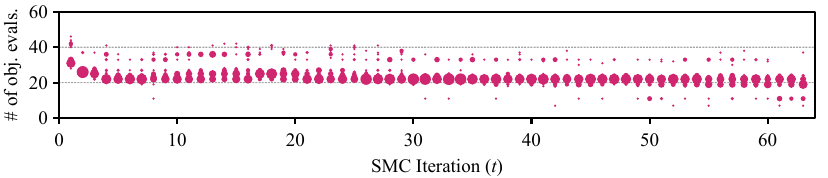}
        \vspace{-0.1ex}
    }
    \\
    \subfloat[Pines]{
        \includegraphics[]{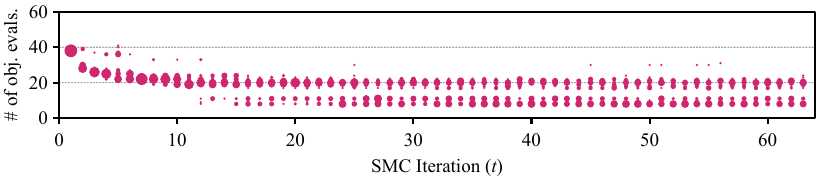}
        \vspace{-0.1ex}
    }
    \\
    \subfloat[Bones]{
        \includegraphics[]{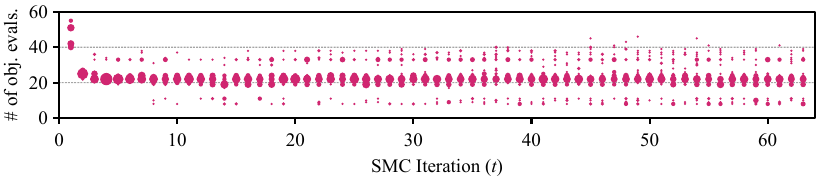}
        \vspace{-0.1ex}
    }
    \\
    \subfloat[Surgical]{
        \includegraphics[]{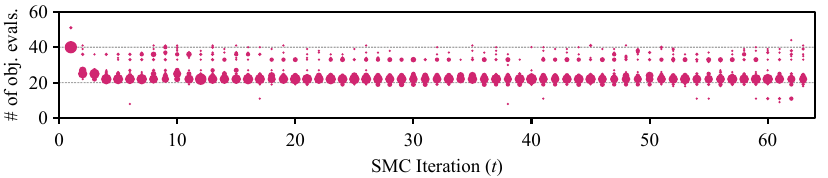}
        \vspace{-0.1ex}
    }
    \vspace{-1ex}
    \caption{
    \textbf{Number of objective evaluations spent during adaptation at each SMC iteration.}
    The size of the markers represents the proportion of runs that spent each respective number of evaluations among 32 independent runs.
    }
    \vspace{-1ex}
\end{figure}

\begin{figure}[H]
    \centering
    \subfloat[HMM]{
        \includegraphics[]{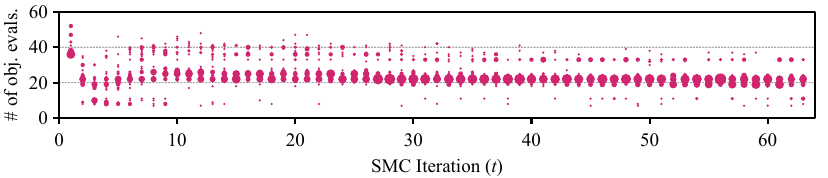}
    }
    \\
    \subfloat[Loss Curves]{
        \includegraphics[]{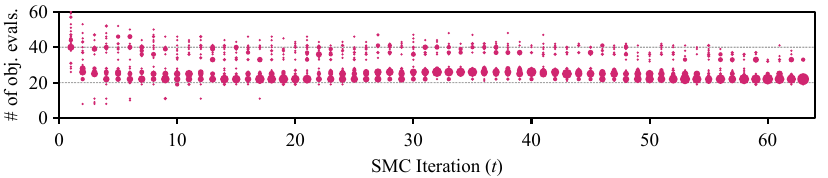}
    }
    \\
    \subfloat[Pilots]{
        \includegraphics[]{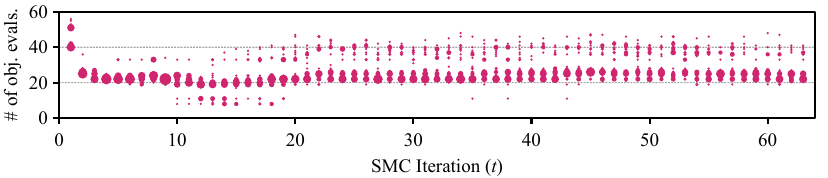}
    }
    \\
    \subfloat[Diamonds]{
        \includegraphics[]{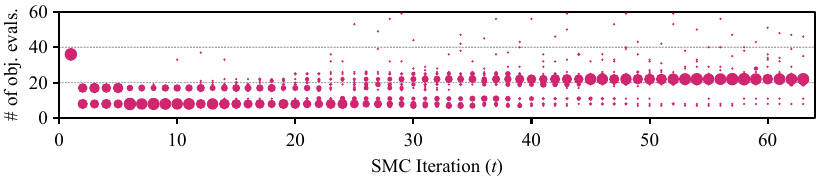}
    }
    \\
    \subfloat[Seeds]{
        \includegraphics[]{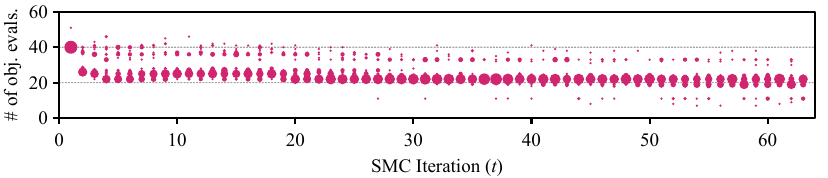}
    }
    \\
    \subfloat[Downloads]{
        \includegraphics[]{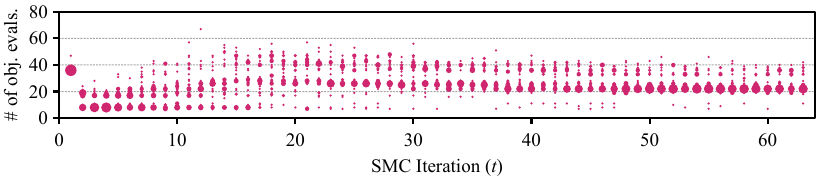}
    }
    \caption{
    \textbf{Number of objective evaluations spent during adaptation at each SMC iteration (continued).}
    The size of the markers represents the proportion of runs that spent each respective number of evaluations among 32 independent runs.
    }
\end{figure}
    
\begin{figure}[H]
    \centering
    \subfloat[Rats]{
        \includegraphics[]{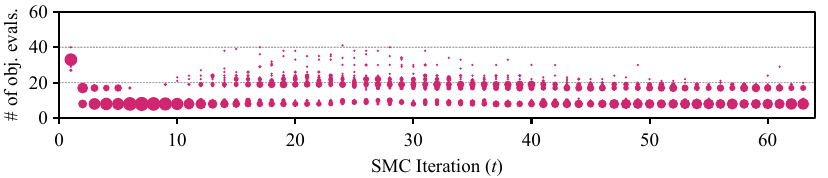}
    }
    \\
    \subfloat[Radon]{
        \includegraphics[]{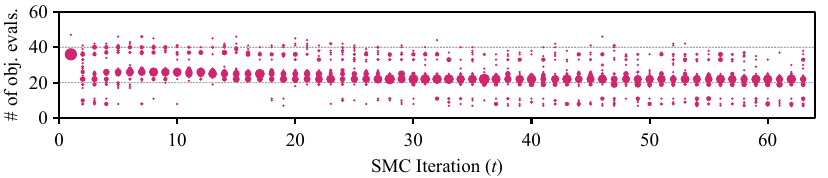}
    }
    \\
    \subfloat[Election88]{
        \includegraphics[]{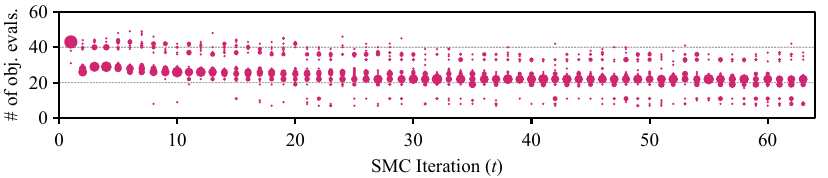}
    }
    \\
    \subfloat[Butterfly]{
        \includegraphics[]{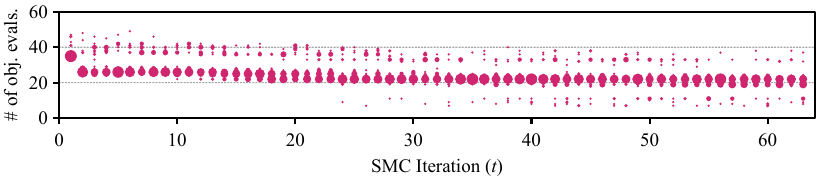}
    }
    \\
    \subfloat[Birds]{
        \includegraphics[]{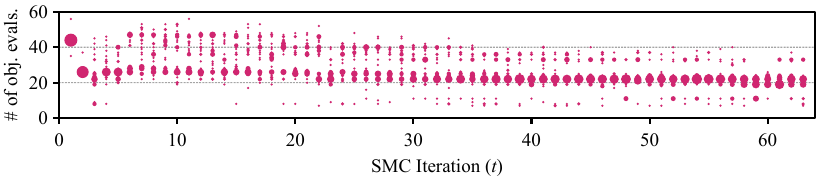}
    }
    \\
    \subfloat[Drivers]{
        \includegraphics[]{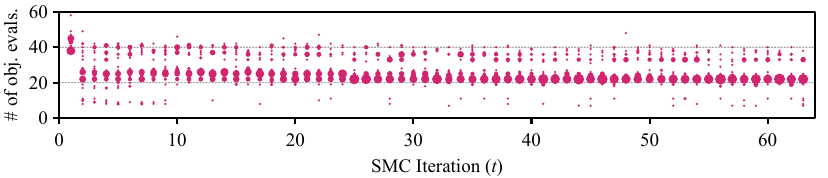}
    }
    \caption{
    \textbf{Number of objective evaluations spent during adaptation at each SMC iteration (continued).}
    The size of the markers represents the proportion of runs that spent each respective number of evaluations among 32 independent runs.
    }
\end{figure}

\begin{figure}[H]
    \centering
    \subfloat[Capture]{
        \includegraphics[]{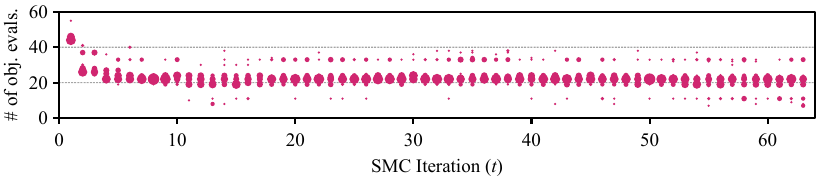}
    }
    \\
    \subfloat[Science]{
        \includegraphics[]{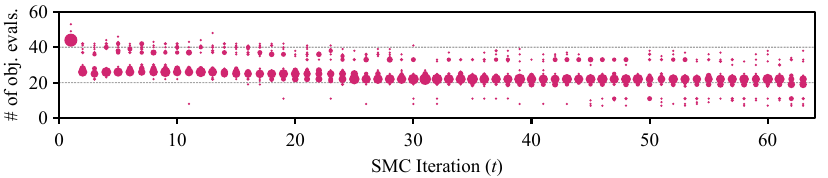}
    }
    \\
    \subfloat[Three Men]{
        \includegraphics[]{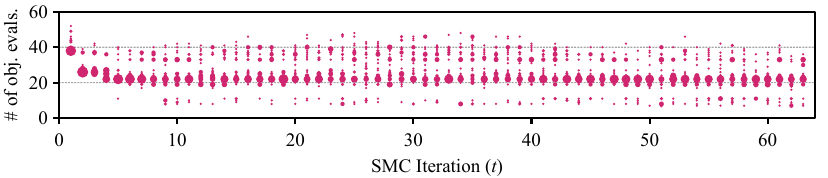}
    }
    \\
    \subfloat[TIMSS]{
        \includegraphics[]{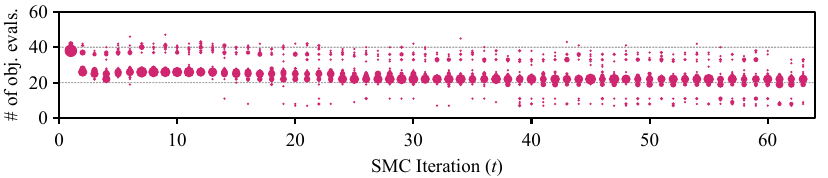}
    }
    \caption{
    \textbf{Number of objective evaluations spent during adaptation at each SMC iteration (continued).}
    The size of the markers represents the proportion of runs that spent each respective number of evaluations among 32 independent runs.
    }
\end{figure}

\newpage
\subsection{Adaptation Results from the Adaptive SMC Samplers}

Finally, we will present additional results generated from our adaptive SMC samplers, including the adapted temperature schedule, step size schedule, and normalizing constant estimates.
The computational budgets are set as \(T_1 = 64\)  and \(N = 1024\) with \(B = 256\).

\subsubsection{SMC-LMC}

\begin{figure}[H]
    \centering
    \subfloat[Funnel]{
        \hspace{-2em}
        \includegraphics{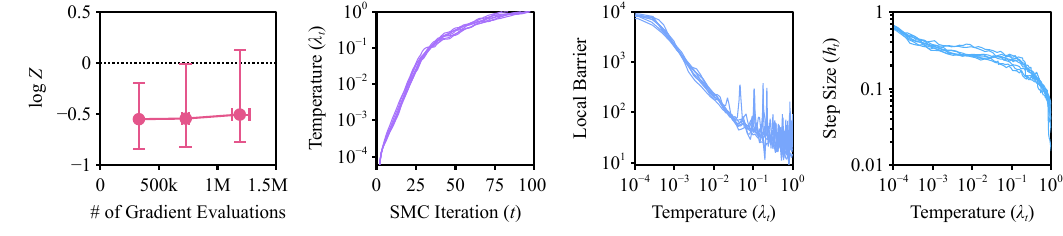}
        \vspace{-1ex}
    }
    \vspace{2ex}
    \\
    \subfloat[Brownian]{
        \hspace{-2em}
        \includegraphics{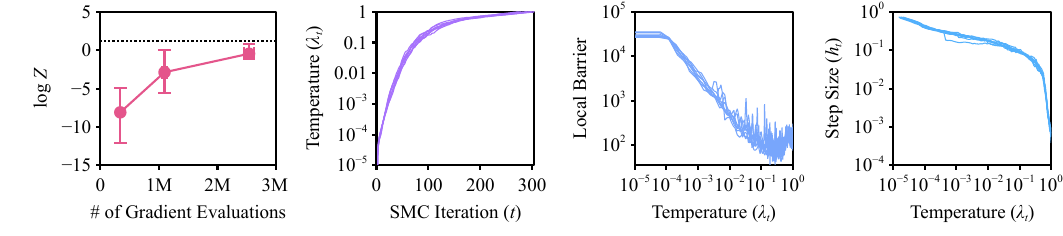}
        \vspace{-1ex}
    }
    \\
    \subfloat[Sonar]{
        \hspace{-2em}
        \includegraphics{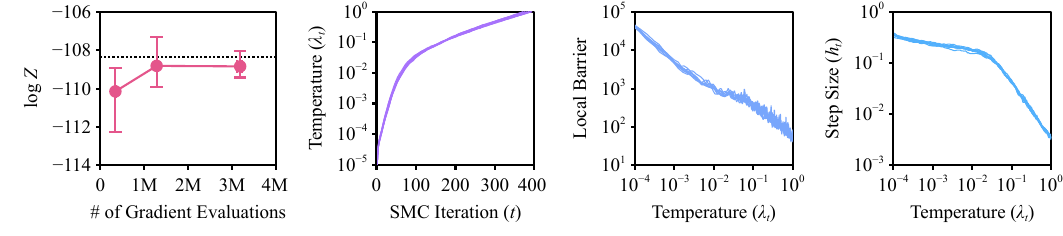}
        \vspace{-1ex}
    }
    \\
    \subfloat[Pines]{
        \hspace{-2em}
        \includegraphics{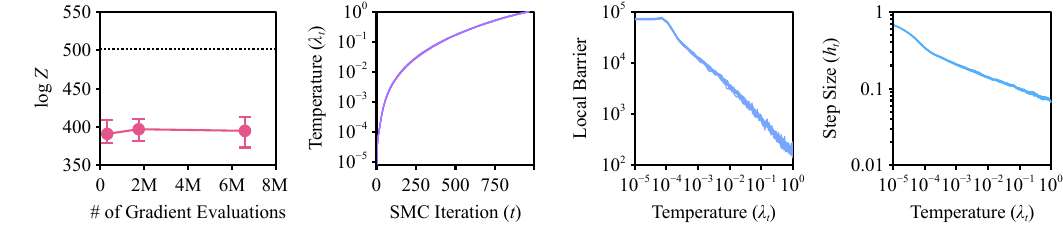}
        \vspace{-1ex}
    }
    \vspace{-1ex}
    \caption{
    \textbf{Normalizing constant estimate, temperature schedule, local communication barrier, and step size schedules obtained by running SMC-LMC.} 
    The dotted line is the ground truth value obtained from a large budget run.
    For the normalizing constant estimate, the confidence intervals in the vertical and horizontal directions are the \(80\%\) quantiles obtained from 32 replications.
    The temperature schedule, local communication barriers, and the step sizes from a subset of 8 runs are shown.
    }
\end{figure}

\newpage
\begin{figure}[H]
    \vspace{2ex}
    \subfloat[Bones]{
        \hspace{-2em}
        \includegraphics{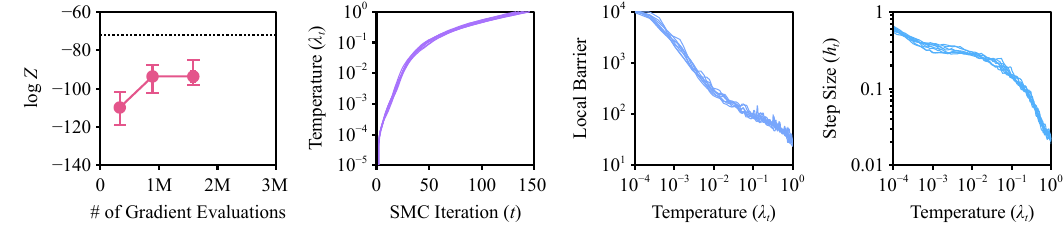}
    }
    \vspace{2ex}
    \\
    \vspace{2ex}
    \subfloat[Surgical]{
        \hspace{-2em}
        \includegraphics{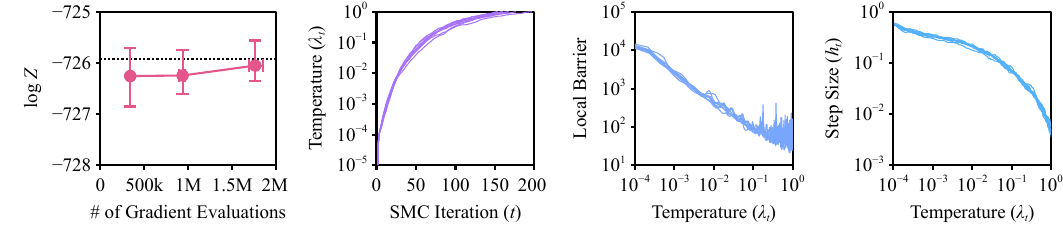}
    }
    \\
    \vspace{2ex}
    \subfloat[HMM]{
        \hspace{-2em}
        \includegraphics{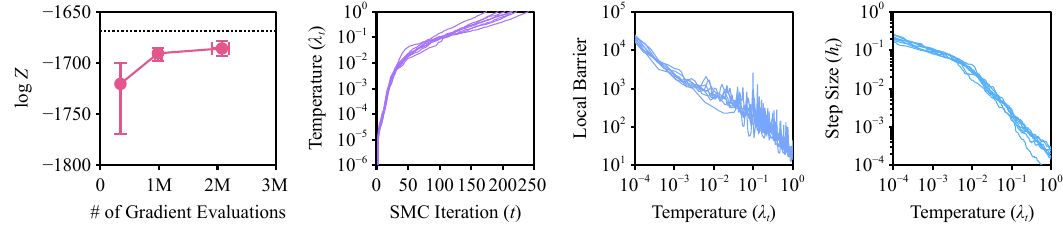}
    }
    \\
    \vspace{2ex}
    \subfloat[Loss Curves]{
        \hspace{-2em}
        \includegraphics{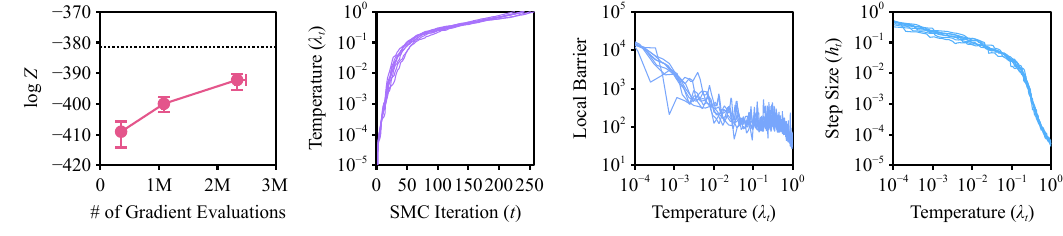}
    }
    \caption{
    \textbf{Normalizing constant estimate, temperature schedule, local communication barrier, and step size schedules obtained by running SMC-LMC (continued).} 
    The dotted line is the ground truth value obtained from a large budget run.
    For the normalizing constant estimate, the confidence intervals in the vertical and horizontal directions are the \(80\%\) quantiles obtained from 32 replications.
    The temperature schedule, local communication barriers, and the step sizes from a subset of 8 runs are shown.
    }
\end{figure}

\newpage
\begin{figure}[H]
    \vspace{2ex}
    \subfloat[Pilots]{
        \hspace{-2em}
        \includegraphics{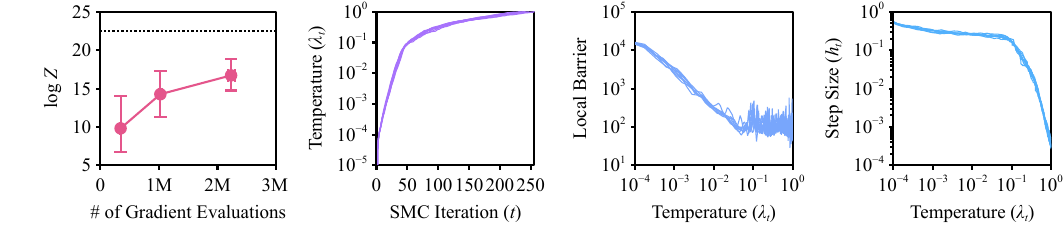}
    }
    \vspace{2ex}
    \\
    \vspace{2ex}
    \subfloat[Diamonds]{
        \hspace{-2em}
        \includegraphics{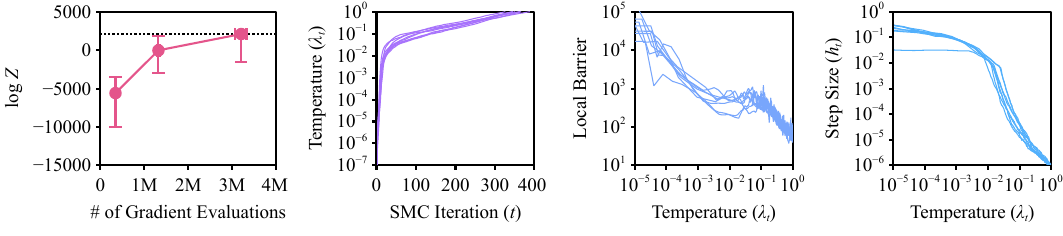}
    }
    \\
    \vspace{2ex}
    \subfloat[Seeds]{
        \hspace{-2em}
        \includegraphics{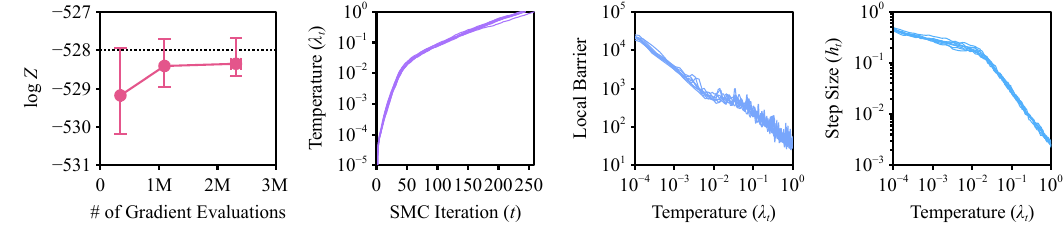}
    }
    \\
    \vspace{2ex}
    \subfloat[Downloads]{
        \hspace{-2em}
        \includegraphics{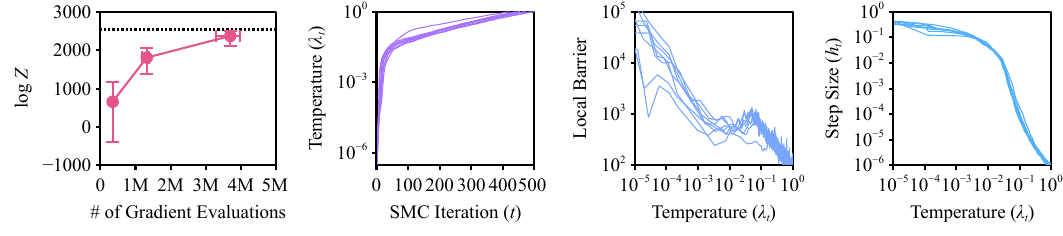}
    }
    \caption{
    \textbf{Normalizing constant estimate, temperature schedule, local communication barrier, and step size schedules obtained by running SMC-LMC (continued).} 
    The dotted line is the ground truth value obtained from a large budget run.
    For the normalizing constant estimate, the confidence intervals in the vertical and horizontal directions are the \(80\%\) quantiles obtained from 32 replications.
    The temperature schedule, local communication barriers, and the step sizes from a subset of 8 runs are shown.
    }
\end{figure}

\newpage
\begin{figure}[H]
    \vspace{2ex}
    \subfloat[Rats]{
        \hspace{-2em}
        \includegraphics{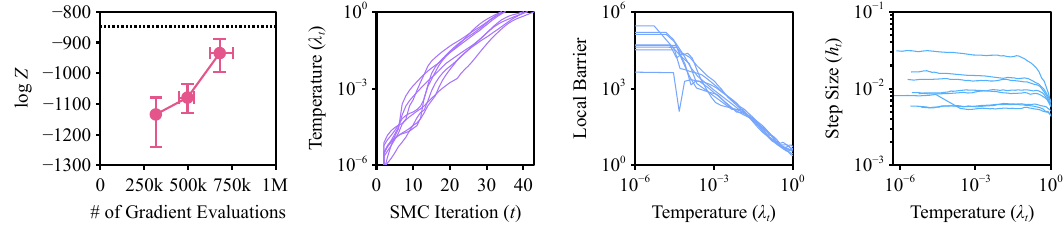}
    }
    \vspace{2ex}
    \\
    \vspace{2ex}
    \subfloat[Radon]{
        \hspace{-2em}
        \includegraphics{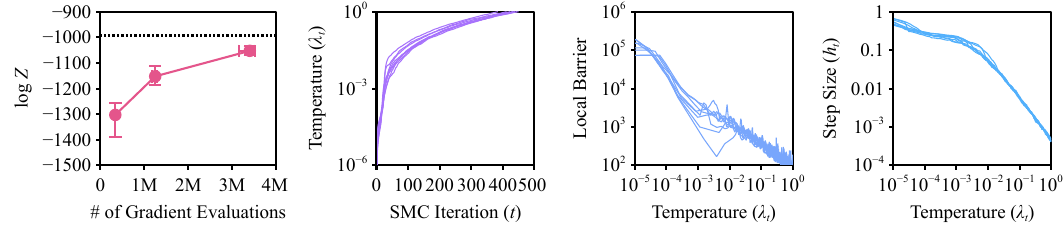}
    }
    \\
    \vspace{2ex}
    \subfloat[Election88]{
        \hspace{-2em}
        \includegraphics{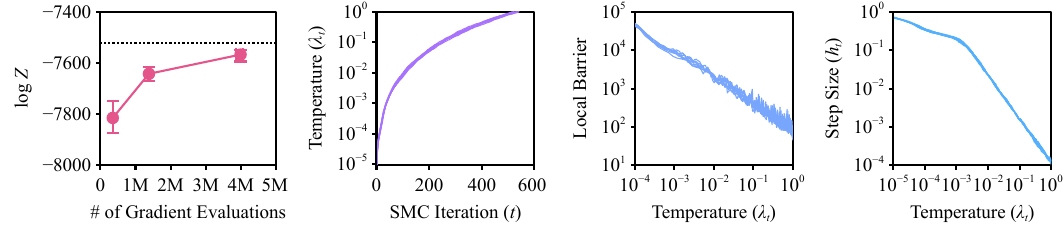}
    }
    \\
    \vspace{2ex}
    \subfloat[Butterfly]{
        \hspace{-2em}
        \includegraphics{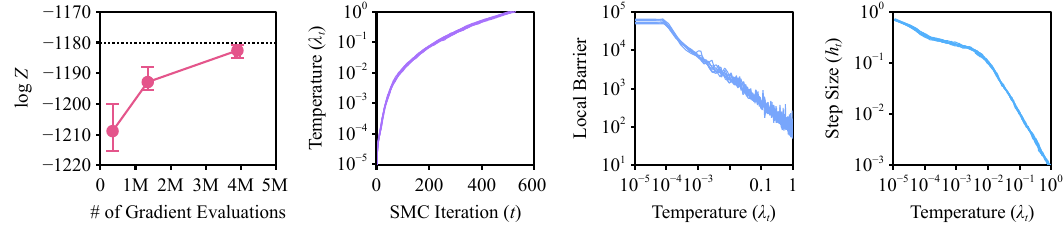}
    }
    \caption{
    \textbf{Normalizing constant estimate, temperature schedule, local communication barrier, and step size schedules obtained by running SMC-LMC (continued).} 
    The dotted line is the ground truth value obtained from a large budget run.
    For the normalizing constant estimate, the confidence intervals in the vertical and horizontal directions are the \(80\%\) quantiles obtained from 32 replications.
    The temperature schedule, local communication barriers, and the step sizes from a subset of 8 runs are shown.
    }
\end{figure}

\newpage
\begin{figure}[H]
    \vspace{2ex}
    \subfloat[Birds]{
        \hspace{-2em}
        \includegraphics{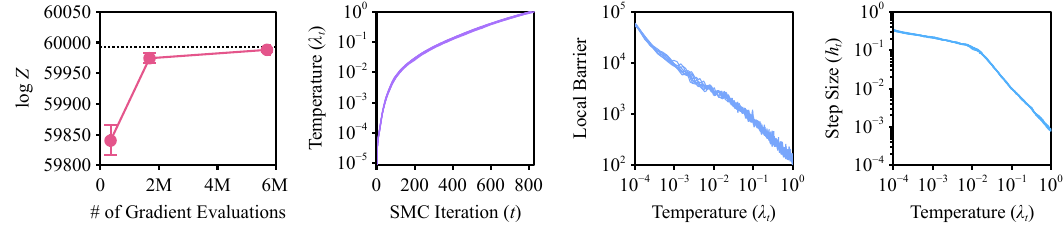}
    }
    \vspace{2ex}
    \\
    \vspace{2ex}
    \subfloat[Drivers]{
        \hspace{-2em}
        \includegraphics{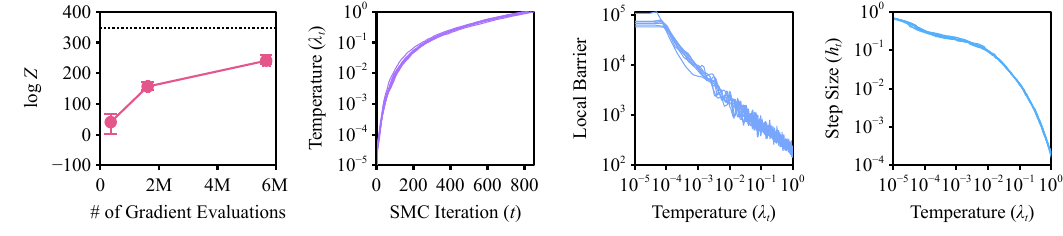}
    }
    \\
    \vspace{2ex}
    \subfloat[Capture]{
        \hspace{-2em}
        \includegraphics{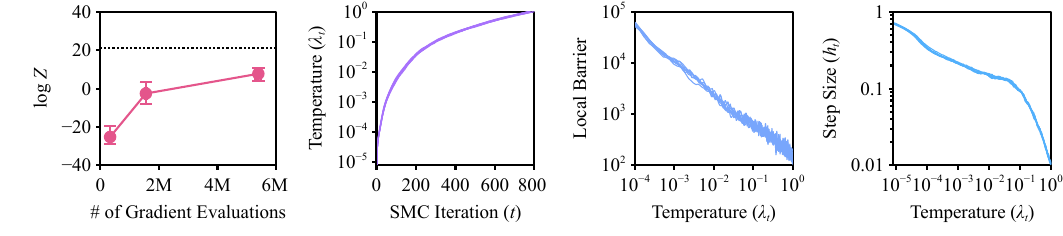}
    }
    \\
    \vspace{2ex}
    \subfloat[Science]{
        \hspace{-2em}
        \includegraphics{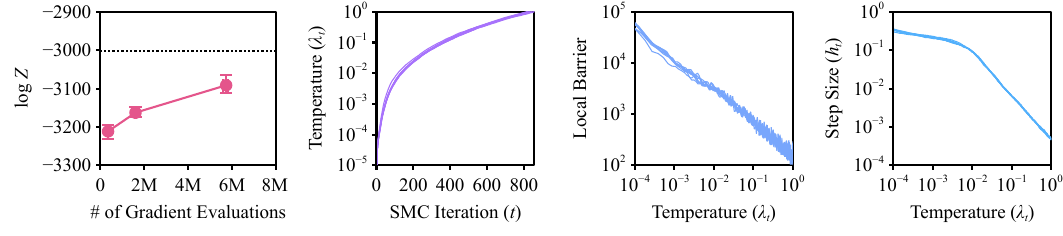}
    }
    \caption{
    \textbf{Normalizing constant estimate, temperature schedule, local communication barrier, and stepsize schedules obtained by running SMC-LMC (continued).} 
    The dotted line is the ground truth value obtained from a large budget run.
    For the normalizing constant estimate, the confidence intervals in the vertical and horizontal directions are the \(80\%\) quantiles obtained from 32 replications.
    The temperature schedule, local communication barriers, and the step sizes from a subset of 8 runs are shown.
    }
\end{figure}

\newpage
\begin{figure}[H]
    \vspace{2ex}
    \subfloat[Three Men]{
        \hspace{-2em}
        \includegraphics{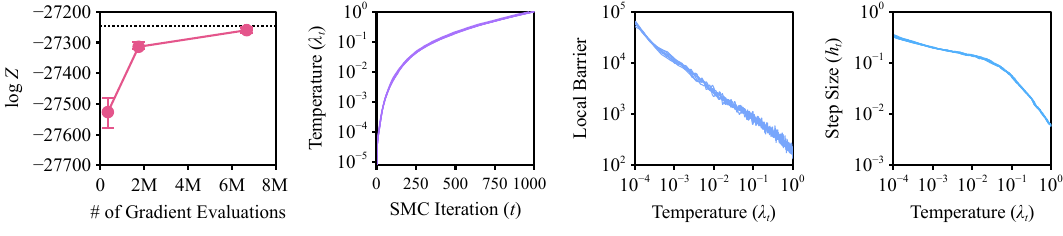}
    }
    \vspace{2ex}
    \\
    \vspace{2ex}
    \subfloat[TIMSS]{
        \hspace{-2em}
        \includegraphics{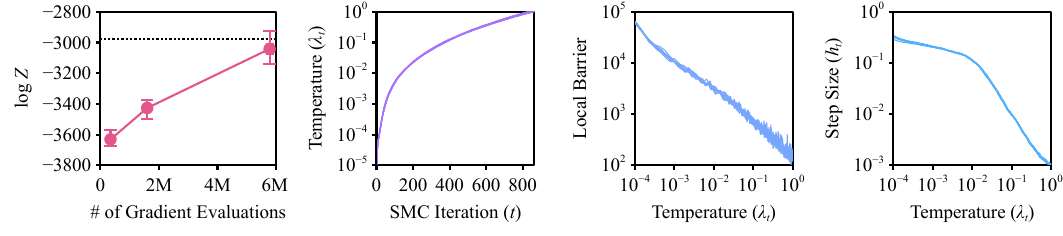}
    }
    \caption{
    \textbf{Normalizing constant estimate, temperature schedule, local communication barrier, and step size schedules obtained by running SMC-LMC (continued).} 
    The dotted line is the ground truth value obtained from a large budget run.
    For the normalizing constant estimate, the confidence intervals in the vertical and horizontal directions are the \(80\%\) quantiles obtained from 32 replications.
    The temperature schedule, local communication barriers, and the step sizes from a subset of 8 runs are shown.
    }
\end{figure}

\newpage
\subsubsection{SMC-KLMC}

\begin{figure}[H]
    \centering
    \vspace{2ex}
    \subfloat[Funnel]{
        \hspace{-2em}
        \includegraphics{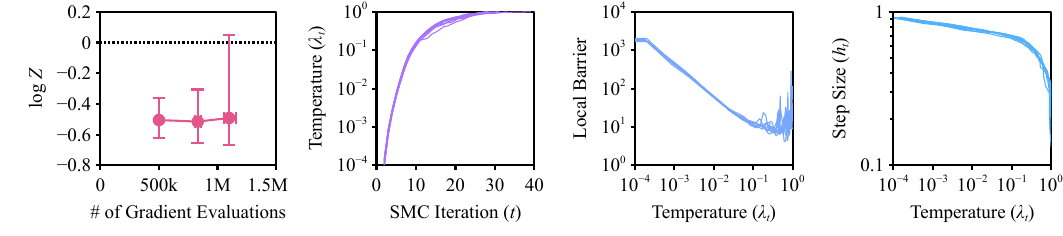}
    }
    \\
    \vspace{2ex}
    \subfloat[Brownian]{
        \hspace{-2em}
        \includegraphics{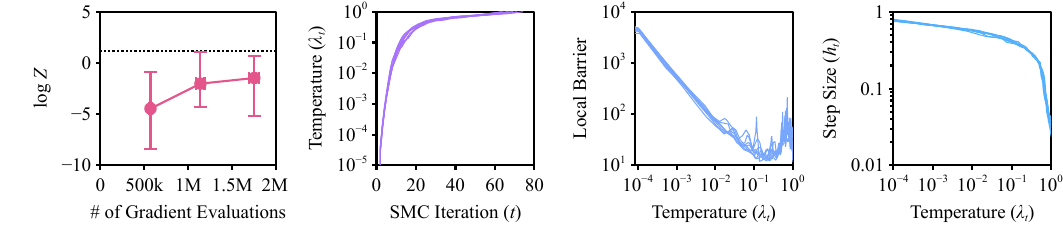}
    }
    \\
    \vspace{2ex}
    \subfloat[Sonar]{
        \hspace{-2em}
        \includegraphics{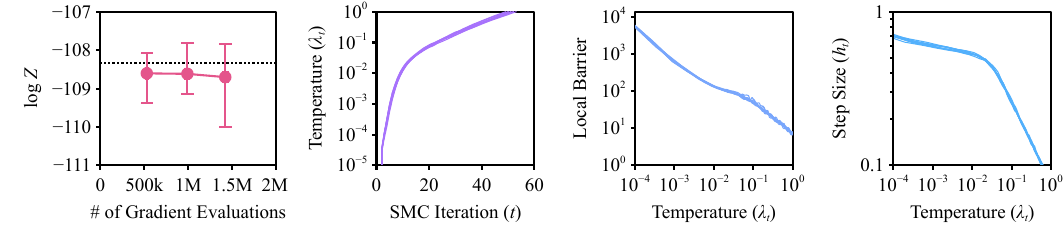}
    }
    \\
    \vspace{2ex}
    \subfloat[Pines]{
        \hspace{-2em}
        \includegraphics{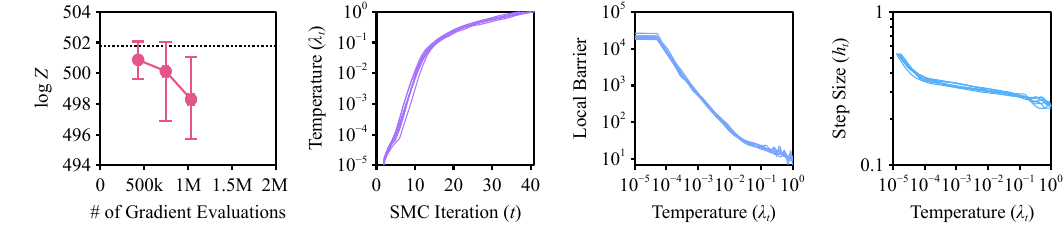}
    }
    \caption{
    \textbf{Normalizing constant estimate, temperature schedule, local communication barrier, and step size schedules obtained by running SMC-KLMC.} 
    The dotted line is the ground truth value obtained from a large budget run.
    For the normalizing constant estimate, the confidence intervals in the vertical and horizontal directions are the \(80\%\) quantiles obtained from 32 replications.
    The temperature schedule, local communication barriers, and the step sizes from a subset of 8 runs are shown.
    }
\end{figure}

\newpage
\begin{figure}[H]
    \vspace{2ex}
    \subfloat[Bones]{
        \hspace{-2em}
        \includegraphics{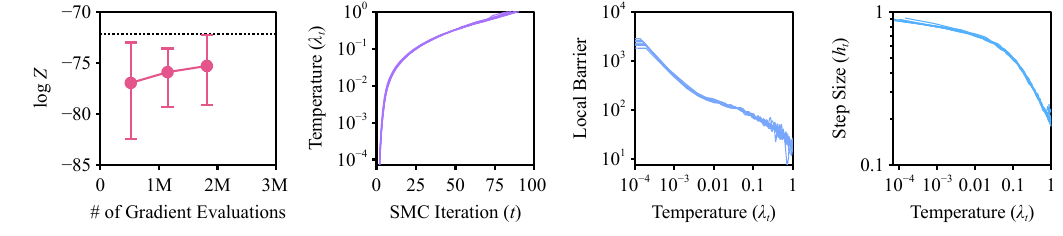}
    }
    \vspace{2ex}
    \\
    \vspace{2ex}
    \subfloat[Surgical]{
        \hspace{-2em}
        \includegraphics{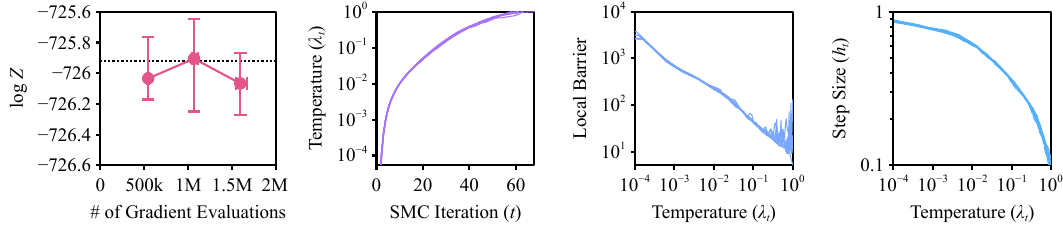}
    }
    \\
    \vspace{2ex}
    \subfloat[HMM]{
        \hspace{-2em}
        \includegraphics{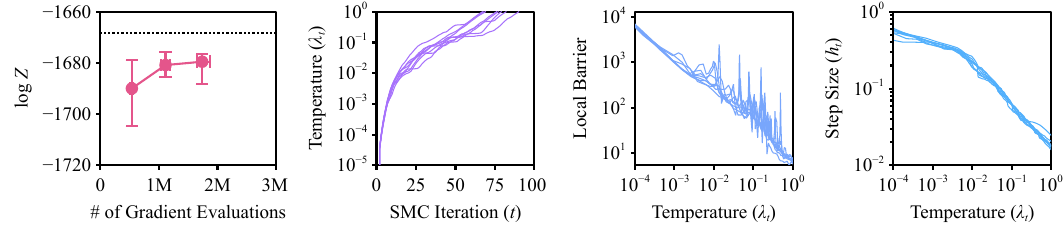}
    }
    \\
    \vspace{2ex}
    \subfloat[Loss Curves]{
        \hspace{-2em}
        \includegraphics{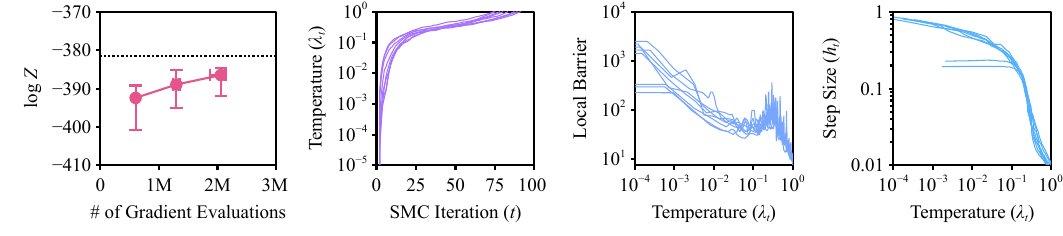}
    }
    \caption{
    \textbf{Normalizing constant estimate, temperature schedule, local communication barrier, and step size schedules obtained by running SMC-KLMC (continued).} 
    The dotted line is the ground truth value obtained from a large budget run.
    For the normalizing constant estimate, the confidence intervals in the vertical and horizontal directions are the \(80\%\) quantiles obtained from 32 replications.
    The temperature schedule, local communication barriers, and the step sizes from a subset of 8 runs are shown.
    }
\end{figure}

\newpage
\begin{figure}[H]
    \vspace{2ex}
    \subfloat[Pilots]{
        \hspace{-2em}
        \includegraphics{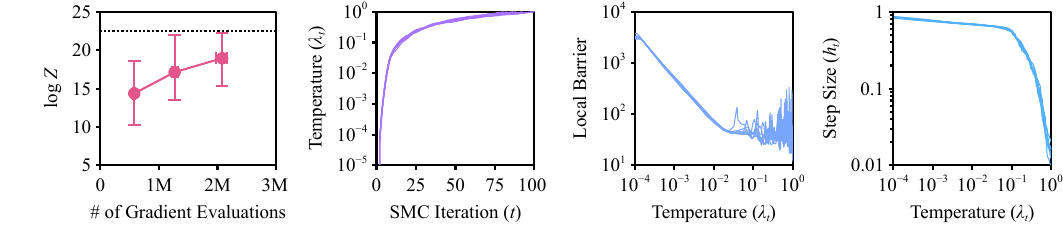}
        \vspace{-1ex}
    }
    \vspace{2ex}
    \\
    \vspace{2ex}
    \subfloat[Diamonds]{
        \hspace{-2em}
        \includegraphics{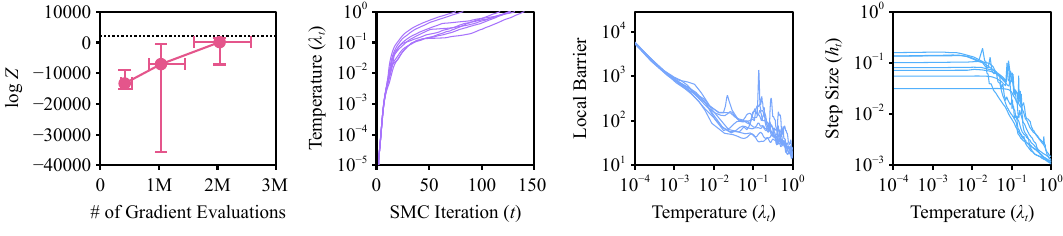}
    }
    \\
    \vspace{2ex}
    \subfloat[Seeds]{
        \hspace{-2em}
        \includegraphics{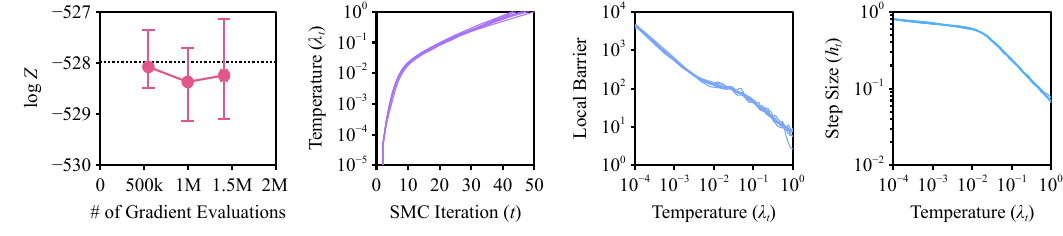}
    }
    \\
    \vspace{2ex}
    \subfloat[Downloads]{
        \hspace{-2em}
        \includegraphics{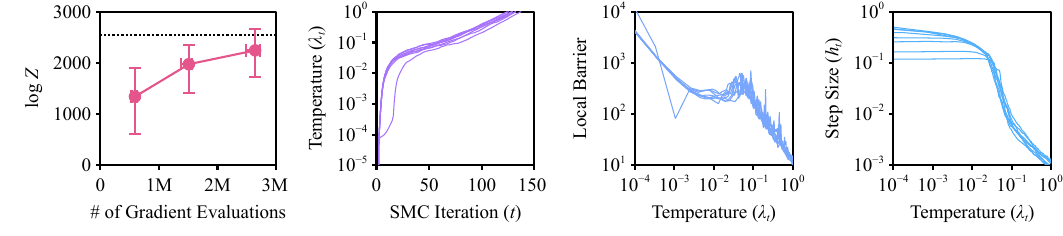}
    }
    \caption{
    \textbf{Normalizing constant estimate, temperature schedule, local communication barrier, and step size schedules obtained by running SMC-KLMC (continued).} 
    The dotted line is the ground truth value obtained from a large budget run.
    For the normalizing constant estimate, the confidence intervals in the vertical and horizontal directions are the \(80\%\) quantiles obtained from 32 replications.
    The temperature schedule, local communication barriers, and the step sizes from a subset of 8 runs are shown.
    }
\end{figure}

\newpage
\begin{figure}[H]
    \vspace{2ex}
    \subfloat[Rats]{
        \hspace{-2em}
        \includegraphics{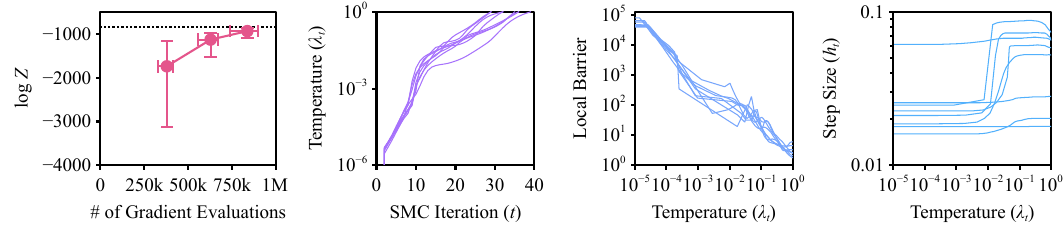}
    }
    \vspace{2ex}
    \\
    \vspace{2ex}
    \subfloat[Radon]{
        \hspace{-2em}
        \includegraphics{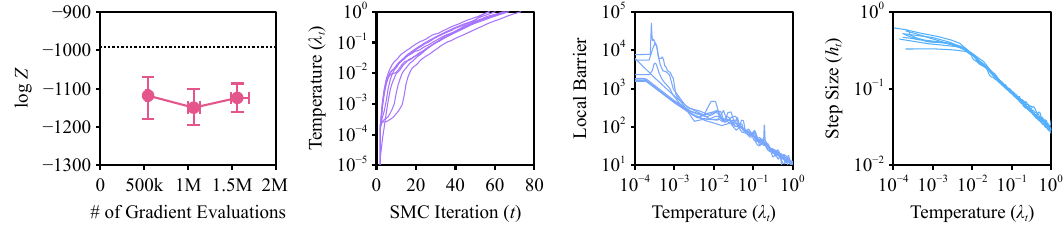}
    }
    \\
    \vspace{2ex}
    \subfloat[Election88]{
        \hspace{-2em}
        \includegraphics{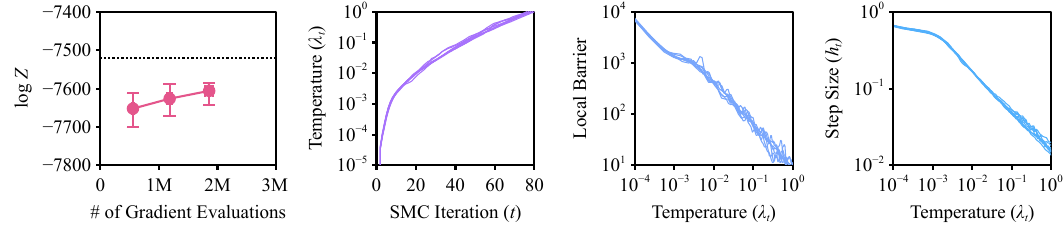}
    }
    \\
    \vspace{2ex}
    \subfloat[Butterfly]{
        \hspace{-2em}
        \includegraphics{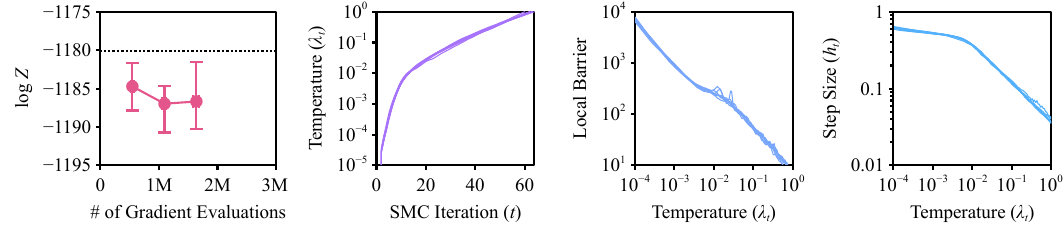}
    }
    \caption{
    \textbf{Normalizing constant estimate, temperature schedule, local communication barrier, and step size schedules obtained by running SMC-KLMC (continued).} 
    The dotted line is the ground truth value obtained from a large budget run.
    For the normalizing constant estimate, the confidence intervals in the vertical and horizontal directions are the \(80\%\) quantiles obtained from 32 replications.
    The temperature schedule, local communication barriers, and the step sizes from a subset of 8 runs are shown.
    }
\end{figure}

\newpage
\begin{figure}[H]
    \vspace{2ex}
    \subfloat[Birds]{
        \hspace{-2em}
        \includegraphics{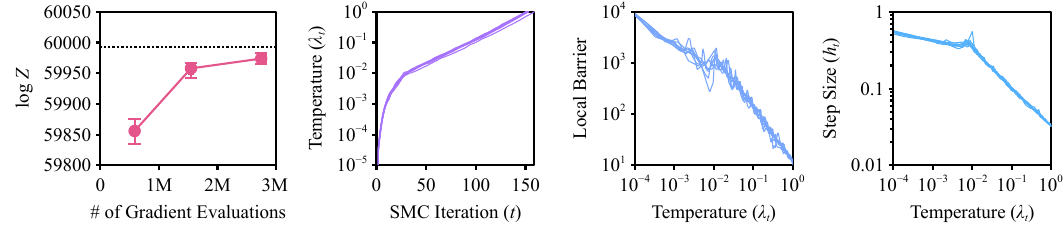}
        \vspace{-1ex}
    }
    \vspace{2ex}
    \\
    \vspace{2ex}
    \subfloat[Drivers]{
        \hspace{-2em}
        \includegraphics{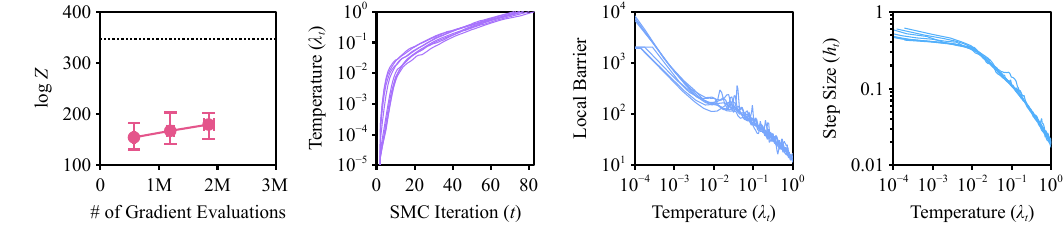}
    }
    \\
    \vspace{2ex}
    \subfloat[Capture]{
        \hspace{-2em}
        \includegraphics{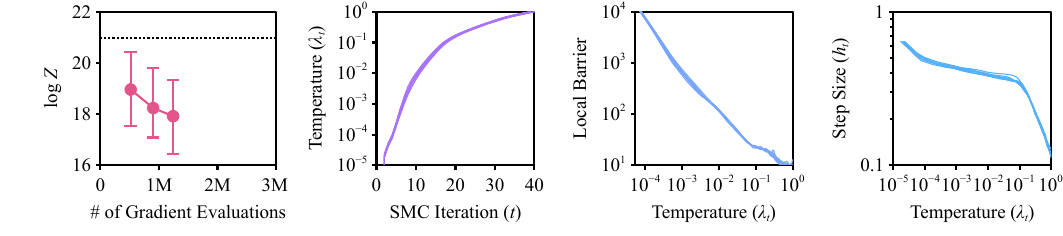}
    }
    \\
    \vspace{2ex}
    \subfloat[Science]{
        \hspace{-2em}
        \includegraphics{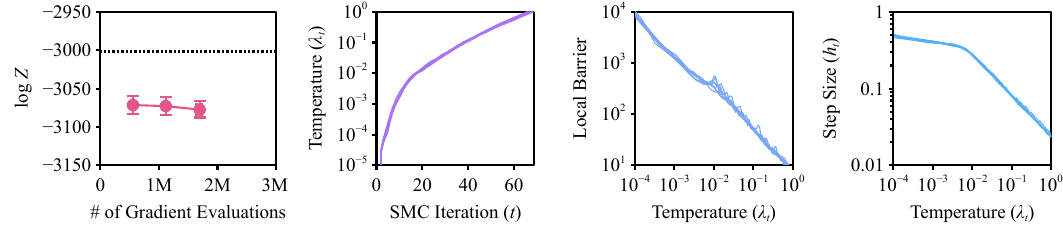}
    }
    \caption{
    \textbf{Normalizing constant estimate, temperature schedule, local communication barrier, and step size schedules obtained by running SMC-KLMC (continued).} 
    The dotted line is the ground truth value obtained from a large budget run.
    For the normalizing constant estimate, the confidence intervals in the vertical and horizontal directions are the \(80\%\) quantiles obtained from 32 replications.
    The temperature schedule, local communication barriers, and the step sizes from a subset of 8 runs are shown.
    }
\end{figure}

\newpage
\begin{figure}[H]
    \vspace{2ex}
    \subfloat[Three Men]{
        \hspace{-2em}
        \includegraphics{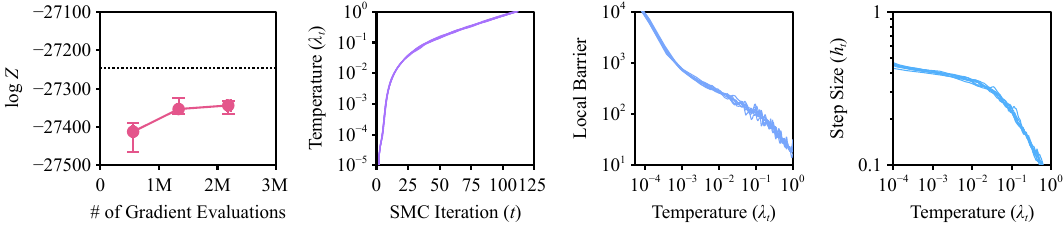}
    }
    \vspace{2ex}
    \\
    \vspace{2ex}
    \subfloat[TIMSS]{
        \hspace{-2em}
        \includegraphics{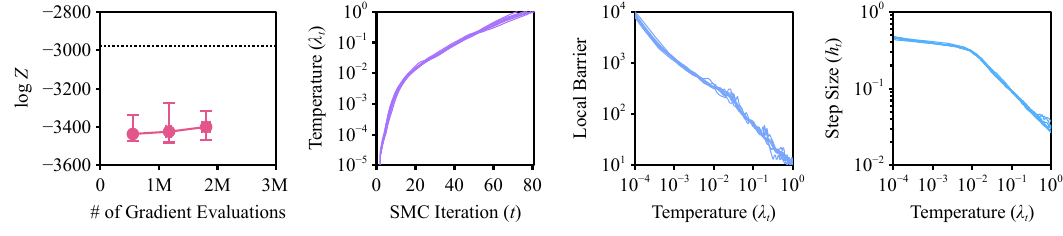}
    }
    \caption{
    \textbf{Normalizing constant estimate, temperature schedule, local communication barrier, and step size schedules obtained by running SMC-KLMC (continued).} 
    The dotted line is the ground truth value obtained from a large budget run.
    For the normalizing constant estimate, the confidence intervals in the vertical and horizontal directions are the \(80\%\) quantiles obtained from 32 replications.
    The temperature schedule, local communication barriers, and the step sizes from a subset of 8 runs are shown.
    }
\end{figure}